\documentclass[journal]{IEEEtran}

\usepackage{cite}
\usepackage{amsmath}
\usepackage{amsthm}
\usepackage[T1]{fontenc}
\usepackage{newtxtext,newtxmath}
\usepackage{bm}
\newtheorem{theorem}{Theorem}

\newtheorem{proposition}{Proposition}
\newtheorem{corollary}{Corollary}
\theoremstyle{remark}\newtheorem{remark}{Remark}
\theoremstyle{definition}\newtheorem{assumption}{Assumption}
\DeclareMathOperator*{\argmin}{arg\,min}
\DeclareMathOperator*{\argmax}{arg\,max}
\usepackage{tikz}
\usetikzlibrary{arrows.meta,positioning,fit,calc,patterns,backgrounds}
\usepackage{graphicx}
\usepackage{algorithm}
\usepackage{algpseudocode}
\usepackage{booktabs}
\usepackage{pdfpages}
\usepackage{multirow}
\usepackage{threeparttable}
\usepackage{xcolor}
\usepackage{url}

\usepackage{balance}
\newcommand{\figframe}[1]{{\setlength{\fboxsep}{2.5pt}%
  \setlength{\fboxrule}{0.4pt}\fcolorbox{black!22}{white}{%
  \parbox{\dimexpr\columnwidth-2\fboxsep-2\fboxrule\relax}{\centering #1}}}}
\newcommand{\figframewide}[1]{{\setlength{\fboxsep}{2.5pt}%
  \setlength{\fboxrule}{0.4pt}\fcolorbox{black!22}{white}{%
  \parbox{\dimexpr\textwidth-2\fboxsep-2\fboxrule\relax}{\centering #1}}}}
\makeatletter
\@ifpackageloaded{lineno}{%
\renewcommand{\@BAlancecol}{\if@twocolumn
  \setbox0=\vbox{\unvbox\@outputbox} \@tempdima=\ht0
  \advance\@tempdima by \topskip \advance\@tempdima
     by -\baselineskip \divide\@tempdima by 2
     \splittopskip=\topskip
  {\vbadness=\@M \loop \global\setbox3=\copy0
   \global\setbox1=\vsplit3 to \@tempdima
   \ifdim\ht3>\@tempdima \global\advance\@tempdima by 1pt \repeat}
   \setbox\@leftcolumn=\vbox to \@tempdima{%
     \protected@write\@auxout{}{\string\def\string\@LN@column{1}}%
     \unvbox1\vfil}
   \setbox\@outputbox=\vbox to \@tempdima
     {\protected@write\@auxout{}{\string\def\string\@LN@column{2}}%
      \dimen2=\dp3\unvbox3\kern-\dimen2
      \vfil}
  \fi}}{}
\makeatother

\newif\ifblind
\blindfalse

\newcommand{\Ninst}{100}  %
\newcommand{\Nseeds}{3}  %

\newcommand{\LkKsCombos}{24}

\newcommand{\TightN}{267}

\newcommand{\TightMeanPct}{99.7}

\newcommand{\TightWorstPct}{97.1}

\newcommand{\UtilGridLow}{8}  %
\newcommand{\UtilGridHigh}{82}  %
\newcommand{\UtilWinMin}{66}  %
\newcommand{\UtilLoseMax}{56}  %
\newcommand{\UtlVOneTSix}{74.3\%}  %
\newcommand{\UtlVTwoTSix}{53.1\%}  %

\newcommand{\UtilFiftyTSix}{76}  %
\newcommand{\UtilFiftyTTen}{80}  %
\newcommand{\UtilEighty}{74}  %
\newcommand{\ScPdrFiftyTSix}{1113.9$\pm$19.3}  %
\newcommand{\ScCpFiftyTSix}{1029.8$\pm$24.4}  %
\newcommand{\ScMixFiftyTSix}{1061.1$\pm$18.2}  %
\newcommand{\ScPdrFiftyTTen}{1789.9$\pm$39.5}  %
\newcommand{\ScCpFiftyTTen}{1704.1$\pm$45.6}  %
\newcommand{\ScMixFiftyTTen}{1668.0$\pm$31.9}  %
\newcommand{\ScPdrEighty}{2173.7$\pm$36.7}  %
\newcommand{\ScMixEighty}{1890.8$\pm$23.4}  %
\newcommand{\ScMixDPdrFiftyTSix}{4.7}  %
\newcommand{\ScMixDPdrFiftyTTen}{6.8}  %
\newcommand{\ScMixDPdrEighty}{13.0}  %
\newcommand{\ScMixWtlFiftyTSix}{20/0/0}  %
\newcommand{\ScMixWtlFiftyTTen}{20/0/0}  %
\newcommand{\ScMixCpPval}{5.6\times10^{-3}}  %
\newcommand{\ScMixPdrPvalMax}{1.9\times10^{-6}}  %

\newcommand{\SampDGaLow}{1.7}
\newcommand{\SampDGaHigh}{4.7}

\newcommand{\SampWallFifty}{180}  %
\newcommand{\SampN}{64}
\newcommand{\SampMixDCpEighty}{8.6}  %
\newcommand{\SampWallEighty}{440}  %
\newcommand{\ScCpEighty}{1974.2$\pm$83.2}  %
\newcommand{\ScMixDCpEighty}{4.2}  %
\newcommand{\ScMixWtlCpEighty}{16/0/4}  %
\newcommand{\ScMixPvalCpEighty}{2.0\times10^{-3}}  %
\newcommand{\ScCpOwnGapEighty}{34}  %
\newcommand{\ScCpHalfFiftyTSix}{1034.3$\pm$23.5}  %
\newcommand{\ScCpHalfFiftyTTen}{1709.4$\pm$46.7}  %
\newcommand{\ScCpHalfEighty}{2002.0$\pm$83.7}  %

\newcommand{\UtilLinkVThree}{66}  %
\newcommand{\UtilLinkVSix}{62}  %

\newcommand{\ProcHr}{4.4}
\newcommand{\MoveMinTOne}{30}
\newcommand{\MoveHrTSix}{2.8}
\newcommand{\MoveHrTTen}{4.7}

\newcommand{\ParamsDeploy}{34{,}892}  %
\newcommand{\TrainGpuHours}{3.4}      %
\newcommand{\TrainUpdates}{2{,}000}

\newcommand{\LkWedgeAbsVNine}{15.3}   %

\newcommand{\qGaSpreadLow}{4}     %
\newcommand{\qGaSpreadHigh}{14}   %
\newcommand{\MyoSigD}{0.86}

\newcommand{\GuidePriceZeroPct}{95}   %

\newcommand{\ScMixWtlEighty}{20/0/0}  %
\newcommand{\ScCpNoMatchDef}{2.2}  %
\newcommand{\ScMixWtlCpFifty}{14/0/6}  %

\newcommand{\GtwoTostPvalMax}{7.5\times10^{-14}}

\renewcommand{\Ninst}{100}
\renewcommand{\Nseeds}{3}  %
\newcommand{\seedset}{$\{301,302,303\}$}
\newcommand{\gridfleetset}{\{1,2,3\}}

\newcommand{\primarysize}{$10\times25$}
\newcommand{\siglevel}{0.05}

\newcommand{\cpbudget}{300}
\newcommand{\gawall}{60}
\newcommand{\gapopsize}{100}
\newcommand{\qPdrVOneTThree}{288.3$\pm$10.2}  %
\newcommand{\qCpVOneTThree}{257.4$\pm$8.1}  %
\newcommand{\qOursVOneTThree}{284.6$\pm$7.6}  %
\newcommand{\qGapVOneTThree}{+10.5\%}  %
\newcommand{\qWtlVOneTThree}{62/0/38}  %
\newcommand{\qPdrVOneTSix}{449.7$\pm$23.5}  %
\newcommand{\qCpVOneTSix}{394.2$\pm$9.6}  %
\newcommand{\qOursVOneTSix}{422.1$\pm$9.4}  %
\newcommand{\qGapVOneTSix}{+7.1\%}  %
\newcommand{\qWtlVOneTSix}{93/0/7}  %
\newcommand{\qWtlGaVOneTSix}{0/0/100}  %
\newcommand{\qPdrVOneTTen}{726.1$\pm$40.1}  %
\newcommand{\qCpVOneTTen}{639.7$\pm$16.9}  %
\newcommand{\qOursVOneTTen}{672.1$\pm$17.7}  %
\newcommand{\qGapVOneTTen}{+5.1\%}  %
\newcommand{\qWtlVOneTTen}{90/0/10}  %
\newcommand{\qPdrVTwoTThree}{252.8$\pm$8.7}  %
\newcommand{\qCpVTwoTThree}{225.5$\pm$8.7}  %
\newcommand{\qOursVTwoTThree}{248.1$\pm$8.5}  %
\newcommand{\qGapVTwoTThree}{+10.0\%}  %
\newcommand{\qWtlVTwoTThree}{83/0/17}  %
\newcommand{\qPdrVTwoTSix}{309.2$\pm$10.6}  %
\newcommand{\qCpVTwoTSix}{274.1$\pm$9.6}  %
\newcommand{\qOursVTwoTSix}{304.4$\pm$7.9}  %
\newcommand{\qGapVTwoTSix}{+11.1\%}  %
\newcommand{\qWtlVTwoTSix}{69/0/31}  %
\newcommand{\qPdrVTwoTTen}{409.3$\pm$14.8}  %
\newcommand{\qCpVTwoTTen}{369.3$\pm$12.0}  %
\newcommand{\qOursVTwoTTen}{401.4$\pm$8.5}  %
\newcommand{\qGapVTwoTTen}{+8.7\%}  %
\newcommand{\qWtlVTwoTTen}{75/0/25}  %
\newcommand{\qPdrVThreeTThree}{247.7$\pm$8.9}  %
\newcommand{\qCpVThreeTThree}{221.5$\pm$9.1}  %
\newcommand{\qOursVThreeTThree}{243.0$\pm$8.1}  %
\newcommand{\qGapVThreeTThree}{+9.7\%}  %
\newcommand{\qWtlVThreeTThree}{81/0/19}  %
\newcommand{\qPdrVThreeTSix}{286.8$\pm$8.9}  %
\newcommand{\qCpVThreeTSix}{245.8$\pm$8.9}  %
\newcommand{\qOursVThreeTSix}{282.7$\pm$8.3}  %
\newcommand{\qGapVThreeTSix}{+15.0\%}  %
\newcommand{\qWtlVThreeTSix}{69/0/31}  %
\newcommand{\qPdrVThreeTTen}{350.7$\pm$11.7}  %
\newcommand{\qCpVThreeTTen}{302.5$\pm$9.6}  %
\newcommand{\qOursVThreeTTen}{349.2$\pm$9.9}  %
\newcommand{\qGapVThreeTTen}{+15.4\%}  %
\newcommand{\qWtlVThreeTTen}{59/0/41}  %
\newcommand{\GaOneImpVOne}{+1.32}  %
\newcommand{\GbImpVOne}{+14.50}  %
\newcommand{\GbWtlVOne}{100/0/0}  %
\newcommand{\GbImpVTwo}{+1.96}  %
\newcommand{\GbWtlVTwo}{82/0/18}  %
\newcommand{\GuImpVOneTThree}{+1.63}  %
\newcommand{\GuImpVTwoTThree}{+0.62}  %
\newcommand{\GuImpVOneTSix}{+2.11}  %
\newcommand{\GuImpVTwoTSix}{+0.84}  %
\newcommand{\GuPvalMax}{1.6\times10^{-5}}  %
\newcommand{\TostCells}{16}
\newcommand{\TostEq}{15}

\newcommand{\TostBetterDiff}{-2.84\%}  %
\newcommand{\TostEqMaxAbs}{0.37\%}  %
\newcommand{\LkOursVThree}{4239$\pm$284}  %
\newcommand{\LkJointDVThree}{-14.7\%}  %
\newcommand{\LkModDVThree}{-16.3\%}  %
\newcommand{\LkHeurDVThree}{-8.4\%}  %
\newcommand{\LkOursVSix}{2949$\pm$166}  %
\newcommand{\LkJointDVSix}{+27.8\%}  %
\newcommand{\LkModDVSix}{+23.0\%}  %
\newcommand{\LkHeurDVSix}{+30.7\%}  %
\newcommand{\LkOursVNine}{2725$\pm$172}  %
\newcommand{\LkJointDVNine}{+49.2\%}  %
\newcommand{\LkModDVNine}{+44.1\%}  %
\newcommand{\LkHeurDVNine}{+45.5\%}  %
\newcommand{\LkOursVTwelve}{2725$\pm$173}  %
\newcommand{\LkJointDVTwelve}{+57.2\%}  %
\newcommand{\LkModDVTwelve}{+54.4\%}  %
\newcommand{\LkHeurDVTwelve}{+49.8\%}  %
\newcommand{\LkOursVFifteen}{2711$\pm$158}  %
\newcommand{\LkJointDVFifteen}{+56.8\%}  %
\newcommand{\LkModDVFifteen}{+56.8\%}  %
\newcommand{\LkHeurDVFifteen}{+49.5\%}  %
\newcommand{\LkOursVEighteen}{2720$\pm$162}  %
\newcommand{\LkJointDVEighteen}{+57.8\%}  %
\newcommand{\LkModDVEighteen}{+58.4\%}  %
\newcommand{\LkHeurDVEighteen}{+50.1\%}  %
\newcommand{\LkVThreePvalMax}{9.2\times10^{-14}}  %
\newcommand{\LkWedgeVSix}{-11.0\%}  %
\newcommand{\CertFiftyLow}{38}  %
\newcommand{\CertFiftyHigh}{47}  %
\newcommand{\CertNadm}{1{,}600}  %

\newcommand{\MixSaRelFiftyTSix}{+1.89}

\newcommand{\MixSaRelFiftyTTen}{+2.54}

\newcommand{\DisWallFifty}{7}
\newcommand{\DisWallEighty}{14}

\newcommand{\HetEnvMax}{100}            %
\newcommand{\HetRealizedHundredMean}{7.6}   %
\newcommand{\HetNTransformed}{1{,}500}  %
\newcommand{\HetAdvBaseVOne}{1.96}      %
\newcommand{\HetAdvPeakVOne}{3.54}      %
\newcommand{\HetAdvHundredVOne}{2.70}   %
\newcommand{\HetAdvBaseVTwo}{0.55}      %
\newcommand{\HetAdvHundredVTwo}{2.40}   %
\newcommand{\HetTrainPvalMax}{2.3\times10^{-3}} %

\newcommand{\AnchorModules}{3{,}034}   %
\newcommand{\AnchorTowers}{two}        %
\newcommand{\AnchorFloors}{56}         %
\newcommand{\ModulesPerFloor}{27}      %

\newcommand{\qPdrHolmMax}{0.025}

\newcommand{\AblVehDVOneTSix}{4.8}  %
\newcommand{\AblVehDVOneTTen}{6.2}  %
\newcommand{\AblVehWtlVOneTSix}{99/0/1}  %
\newcommand{\AblVehWtlVOneTTen}{97/0/3}  %
\newcommand{\AblVehDVOneTThree}{3.1}  %
\newcommand{\AblVehRestMax}{0.9}  %
\newcommand{\AblVehPHolmMax}{10^{-17}}  %

\newcommand{\DisrCpVsPolMin}{11.0}  %
\newcommand{\DisrCpVsPolMax}{19.0}  %
\newcommand{\DisrCpImpMax}{0.8}  %
\newcommand{\DisrCpWallShortMin}{9}  %
\newcommand{\DisrCpWallShortMax}{41}  %
\newcommand{\DisrGaVsPolMin}{4.8}  %
\newcommand{\DisrGaVsPolMax}{9.5}  %
\newcommand{\DisrGaWallMin}{33}  %
\newcommand{\DisrGaWallMax}{120}  %
\newcommand{\ColdFiftyNone}{3 of 3}  %

\newcommand{\LatEvtCpuTen}{2.3}  %
\newcommand{\LatEvtCpuEighty}{5.9}  %
\newcommand{\LatEvtGpuEighty}{4.3}  %
\newcommand{\LatSchedCpuTen}{0.8}  %
\newcommand{\LatSchedCpuEighty}{15.9}  %
\newcommand{\PeakRssCpuMax}{519}  %
\newcommand{\StMixDPdrFiftyLow}{3.6}  %
\newcommand{\MixSaRelFiftyLow}{+4.33}  %
\newcommand{\StMixWtlEightyLow}{5/0/15}  %
\newcommand{\StMixLongPdrEightyLow}{1.5}  %
\newcommand{\MixSaRelEightyLow}{+4.44}  %
\newcommand{\StSparseEightyLow}{4}  %
\newcommand{\MixSaLeadMax}{4.4}  %

\newcommand{\NoisePolReplanVsRsMax}{0.7}  %
\newcommand{\NoisePolVsPdrMin}{4.1}  %
\newcommand{\NoisePolVsPdrMax}{6.4}  %
\newcommand{\AblVehMixDFiftyTTen}{12.9}  %
\newcommand{\AblVehMixDEighty}{23.0}  %
\newcommand{\AblVehMixWtl}{20/0/0}  %
\newcommand{\AblVehMixPHolm}{2.9\times10^{-6}}  %
\newcommand{\AblVehMixSmallMax}{5.0}  %

\newcommand{\SeedRowsN}{23}  %
\newcommand{\BootExclZeroRows}{18}  %

\newcommand{\CpAddRootTSixMin}{3.6}  %
\newcommand{\CpAddRootTSixMax}{4.7}  %
\newcommand{\CpAddRootTTenMin}{2.2}  %
\newcommand{\CpAddRootTTenMax}{2.8}  %
\newcommand{\CpIncImpTSix}{4.8}  %
\newcommand{\CpIncWtlTSix}{12/0/0}  %
\newcommand{\CpIncWtlTTen}{6/0/6}  %
\newcommand{\MyoPOneSided}{7.5\times10^{-5}}  %
\newcommand{\ComaCells}{14}  %
\newcommand{\ComaDMin}{2.5}  %
\newcommand{\ComaDMax}{9.8}  %
\newcommand{\TightAllMean}{87.6}  %
\newcommand{\TightFleetMean}{76.0}  %
\newcommand{\TightFleetN}{304}  %
\newcommand{\RootNoVehCovMin}{18}  %
\newcommand{\RootNoVehCovMax}{26}  %
\newcommand{\RootVehCovMin}{74}  %
\newcommand{\RootVehCovMax}{80}  %
\newcommand{\RtImpMean}{3.2}  %
\newcommand{\RtGapShare}{33}  %
\newcommand{\RtResShare}{0.85}  %
\newcommand{\RtEmptyShare}{2.4}  %
\newcommand{\AnProdDMin}{8.0}  %
\newcommand{\AnProdDMax}{16.8}  %
\newcommand{\AnProdWtl}{20/0/0}  %
\newcommand{\AnProdP}{1.9\times10^{-6}}  %
\newcommand{\FbVsPdrMin}{3.8}  %
\newcommand{\FbVsPdrMax}{13.1}  %
\newcommand{\FbSttMin}{11}  %
\newcommand{\FbSttMax}{20}  %

\providecommand{\Cmax}{C_{\max}}
\providecommand{\Ccls}{\mathcal{C}}
\providecommand{\Fcls}{\mathcal{F}}
\providecommand{\elig}{\mathrm{elig}}
\providecommand{\cell}{\mathrm{cell}}
\providecommand{\Bchain}{B_{\mathrm{chain}}}
\providecommand{\Bmch}{B_{\mathrm{mch}}}
\providecommand{\Bveh}{B_{\mathrm{veh}}}

\providecommand{\jr}{\mathrm{jr}}

\begin{document}
\bstctlcite{BSTcontrol}

\title{BOLT: Multi-Bottleneck Bound-Guided Learning for Flexible
Job-Shop Scheduling with Time Lags and Resource-Coupled Transport}

\author{Ziheng~Zhang, Wei~Zhang, Yunyi~Zhao, Zuming~Liu, Wei~Zhou, and
Boon~Ping~Gan%
\thanks{This work has been submitted to the IEEE for possible publication. Copyright may be transferred without notice, after which this version may no longer be accessible.}%
\thanks{Z. Zhang, W. Zhang, and Y. Zhao are with the Infocomm Technology
Cluster, Singapore Institute of Technology, Singapore 828608 (e-mail:
ziheng.zhang@singaporetech.edu.sg; wei.zhang@singaporetech.edu.sg;
yunyi.zhao@singaporetech.edu.sg).}%
\thanks{Z. Liu is with the College of Smart Energy,
Shanghai Jiao Tong University, Shanghai, China (e-mail: zmliu@sjtu.edu.cn). W. Zhou is with the School of Computer Science and Informatics,
Cardiff University, Cardiff CF24 4AG, U.K. (e-mail: ZhouW26@cardiff.ac.uk). B. P. Gan is with D-SIMLAB Technologies, Singapore (e-mail:
boonping@d-simlab.com). Corresponding author: Wei Zhang.}}
\markboth{Preprint}%
{BOLT: Multi-Bottleneck Bound-Guided Learning for FJSP with Time Lags and
Resource-Coupled Transport}

\maketitle

\begin{abstract}
Flexible job-shop scheduling becomes resource-coupled when a scarce
vehicle fleet moves jobs between stations, because the schedule
sets the transport waits. BOLT (bound-oriented learning for
transport-coupled scheduling) builds a lower bound from chain,
machine-capacity, and fleet-capacity relaxations. The bound equals the
makespan at termination, so its per-step changes form a reward that
sums to a constant minus the makespan. Because it never exceeds the
attainable makespan, each candidate's bound increase is a
proven rise of the makespan floor, which the policy reads as a price.
The root bound certifies a completed schedule's suboptimality gap
without a solver. On industry-motivated 10-module
prefabricated-construction instances, explicit transport coupling
outperforms a transport-penalty policy in every tested cell where loaded
moves occupy at least $\UtilWinMin\%$ of fleet time and in none at or
below $\UtilLoseMax\%$. Across five zero-shot 50- and 80-module production
batches, BOLT outperforms the best dispatching-rule pair on four and
the penalty policy on the three with long moves. In the three-vehicle 80-module batch,
its mean makespan is $\ScMixDPdrEighty\%$ below the best pair and
$\ScMixDCpEighty\%$ below an hour-budget constraint-programming solver
whose proven gap averages $\ScCpOwnGapEighty\%$. BOLT decides in
\LatEvtCpuTen--\LatEvtCpuEighty~ms per event on one processor core and
re-plans a disruption in seconds.
\end{abstract}

{\renewcommand{\abstractname}{Note to Practitioners}%
\begin{abstract}
Many plants move heavy modules between stations on a few transporters
or automated guided vehicles, and the wait for a vehicle depends
on the schedule. This paper gives such a plant a dispatcher guided
by a guaranteed lower bound on completion time. The dispatcher
decides each machine and vehicle assignment within milliseconds on one
processor core and re-plans within seconds after a breakdown. Its station
and vehicle decisions have separate output layers, so a plant can retune
or roll back one class while the other stays fixed. Without a solver, the
dispatcher reports a guaranteed bound on each schedule's distance
from optimal, which is loose on large batches
(\CertFiftyLow--\CertFiftyHigh\% on average).

The cases studied use published process data for
prefabricated prefinished volumetric construction. Only tested cases with loaded moves of \MoveHrTSix{} to
\MoveHrTTen~h, as between separate production areas or with a slow fleet,
showed gains over a fixed transport allowance. In the
10-module cases, explicit transport modeling reduced makespan in every
tested case where loaded trips occupied at least $\UtilWinMin\%$ of
fleet time, and a fixed transport allowance sufficed at or below
$\UtilLoseMax\%$. A plant estimates this share from its expected
loaded travel, fleet size, and batch duration, and trains the policy
once on its instance mix. At high load, the dispatcher lost
to the best dispatching rules in one tested case, where one vehicle
served 80 modules.

The model assumes known processing, waiting, and travel times and does
not model stochastic arrivals, vehicle routing and congestion, or mixed
fleets. A schedule executed under duration
noise of up to $\pm30\%$ stayed within $\NoisePolReplanVsRsMax\%$ of
re-planning. When a plant can wait minutes for a plan, per-instance
search finds shorter schedules. Extending the method to stochastic
arrivals and online dispatch without a reserved destination machine is
the next step.

\end{abstract}}

\begin{IEEEkeywords}
Automated guided vehicles, bound-guided reinforcement learning,
flexible job-shop scheduling, prefabricated construction,
resource-coupled transport.
\end{IEEEkeywords}

\providecommand{\resultfig}[2]{%
  \IfFileExists{#1}%
    {\includegraphics[width=\linewidth]{#1}}%
    {\fbox{\parbox[c][3.4cm][c]{0.92\linewidth}{\centering\footnotesize
       Figure placeholder\\[2pt]#2\\[2pt](rendered when results land)}}}}

\section{Introduction}\label{sec:intro}

Automated production systems often move partly finished products between stations on a shared transporter or automated guided vehicle fleet. Their machine and vehicle schedules jointly set throughput. Prefabricated prefinished volumetric construction (PPVC) is one heavy-load case, in which factories build whole rooms before shipment to site. Singapore mandates PPVC on selected government land-sale sites, and one documented project moved 1{,}866 modules that each weighed 26 to 31 tons~\cite{hwang2018ppvc,bca2024ppvc}. Another documented development erected \AnchorTowers{} \AnchorFloors-story towers from \AnchorModules{} such modules, about \ModulesPerFloor{} modules per floor~\cite{unitedtec2022avenuesouth}. Each module rides a heavy transporter between stations, and every
move occupies one vehicle for its full duration, so the plant must
assign stations and vehicles together whenever a module becomes ready.

This joint decision is difficult because the schedule determines each transport wait. A job moves through a sequence of stations, with a vehicle carrying it between different locations. Some waits have durations that the instance fixes. A cast module must cure, which means harden, for a known time before
the next operation can begin, and a machine setup takes a known time
once the job order is chosen. Transport waits differ, because a finished module remains at its
station until a vehicle becomes free, and how long that takes depends
on which other jobs are using the fleet at that moment. Machine decisions set when transport requests arise, while fleet availability sets when the next operation can start. A fixed travel allowance therefore misjudges the delay when vehicles
are scarce, and the scheduler must represent the load that each
decision places on the fleet.

Existing scheduling methods do not provide that representation as a certified learning signal. Prefabrication scheduling studies omit transport or treat it as a
fixed time surcharge~\cite{wang2018offsite}, and learning studies of
dynamic flexible job shops model it as a fixed transfer
delay~\cite{zhang2024dynamic}. Exact and metaheuristic models schedule a finite fleet~\cite{bilge1995time,berterottiere2024flexible,yao2025fjspla}. Because they search each instance separately, they do not amortize
scheduling decisions into a learned millisecond-scale dispatcher. Learning-to-dispatch policies make online decisions in milliseconds~\cite{luo2022hmappo}. Several such methods derive each reward from changes in a lower bound on the makespan attainable from the partial schedule~\cite{zhang2020l2d,song2023fjspdrl,wang2024daniel,paper1anon}. If the bound equals the makespan at termination, these rewards sum to a constant minus the makespan, so maximizing their return minimizes makespan exactly. These bound-guided
methods~\cite{zhang2020l2d,song2023fjspdrl,wang2024daniel,paper1anon}
do not model an explicit vehicle resource. Most vehicle-aware
learners~\cite{li2025realtime,wang2025hierarchical,yang2026hgamppo}
use designed rewards or learned estimates. Moon et
al.~\cite{moon2024learning} combine vehicle decisions
with a chain-bound reward, but their bound omits aggregate fleet
capacity. A chain bound follows each job's own remaining operations with one
fixed travel time per move, so it treats every job as if a vehicle
were always free. It stays flat as the makespan grows with the number of jobs waiting for one vehicle. What remains missing is an admissible bound, tight at termination, that counts aggregate fleet workload, so that an objective-faithful reward also responds to fleet contention.

The principle adopted here is to represent every shared bottleneck in one admissible bound. For each partial schedule, chain, machine-capacity, and fleet-capacity relaxations bound job progression, machine workload, and mandatory loaded moves. Their maximum remains admissible because none can exceed the makespan
attainable from that state, and it selects the active bottleneck with
no tuned weight. The fleet-capacity term aggregates transport work across jobs, so it rises under contention where a chain bound stays flat. The same quantity can then connect the scheduling objective to learning decisions, certified gap bounds, and solver search.

BOLT (bound-oriented learning for transport-coupled scheduling) implements this principle with a certified multi-bottleneck lower bound. The bound is admissible at every reachable partial schedule (Section~\ref{sec:theory}). It equals the makespan at termination, so its one-step change supplies an objective-faithful reward. For each candidate, the successor bound gives a certified action price, which is a proven minimum increase in the makespan lower bound (Section~\ref{sec:method}). At the root, the same quantity gives a solver-free certified upper bound on the relative suboptimality gap of any completed schedule. As a redundant constraint, its fleet-capacity term also raises the lower bound that warm-started CP-SAT proves in the tested one-vehicle cells (Section~\ref{subsec:cert}). The merged-head and factorized realizations read the same bound signals, so controller ownership does not change the learning objective. Whether that objective's fleet term changes the schedules a plant runs depends on
how heavily the fleet is loaded.

On the prefabricated-construction instances tested, the benefit of
explicit coupling depends on fleet load. At the 10-module training size, explicit coupling shortens makespan in every cell at or above $\UtilWinMin\%$ loaded utilization, and a transport-penalty policy is adequate at or below $\UtilLoseMax\%$. Production batches of 50 and 80 modules match the output of two to three tower floors and lie outside the trained sizes. On the three batches with long moves, BOLT outperforms this penalty policy on every instance. In the 80-module batch with three vehicles, its mean makespan is $\ScMixDPdrEighty\%$ below the best
dispatching-rule pair and $\ScMixDCpEighty\%$ below warm-started CP-SAT
given an hour, whose proven gap averages $\ScCpOwnGapEighty\%$. Ablations on four pre-specified cells attribute the training
improvement to the action-price feature. Removing the fleet-capacity term lengthens the schedules by $\AblVehDVOneTSix\%$ and $\AblVehDVOneTTen\%$ in the two 10-module cells where that term sets the bound, and by $\AblVehMixDFiftyTTen$--$\AblVehMixDEighty\%$ on three production batches. BOLT takes $\LatEvtCpuTen$--$\LatEvtCpuEighty$~ms per event on one processor core, or \LatSchedCpuTen--\LatSchedCpuEighty~s for a whole schedule, from 10 to 80 modules. After a machine breakdown, it re-plans in \DisWallFifty{}--\DisWallEighty~s, while CP-SAT at budgets of 1 and 10~s returns the right-shifted plan unchanged. A genetic algorithm needs \DisrGaWallMin--\DisrGaWallMax~s to find a schedule shorter than BOLT's. Plants can use BOLT to learn and deploy resource-coupled transport scheduling, bound a completed schedule's relative suboptimality gap without a solver, and decide when explicit coupling is useful.

The scheduling model is deterministic, with processing, waiting, and travel
times known before scheduling begins, and machine choice lies within a
station cell, as in the PPVC layouts studied. Similar conditions arise when heavy partly finished products move on a small automated guided vehicle fleet~\cite{li2025realtime,wang2025hierarchical,berterottiere2024flexible}. The studied PPVC instances come from an industry-motivated generator whose routing classes and operation structure follow published PPVC process documentation~\cite{bca2024ppvc}. They contain no plant's production records, and the study reports no plant deployment. Within this scope, to our knowledge, no learned flexible job-shop scheduler combines this bound with certified action pricing, a solver-free certified gap bound, and solver strengthening. Splitting control between machine and vehicle controllers is not part of this claim, because it is established practice~\cite{li2025realtime,yang2026hgamppo,wang2025hierarchical,zhang2024dynamic}. None of these systems trains on a proven bound (Section~\ref{sec:related}). BOLT makes online dispatch decisions, while exact and metaheuristic methods with per-instance budgets provide solution-quality references. The paper makes four contributions.

\begin{itemize}
\item \textbf{Resource-coupled scheduling formulation.} We formalize flexible job-shop scheduling with machine-free time lags and a finite fleet whose vehicles are seized for each move and whose destination machines are reserved in transit. A delay-semantics taxonomy distinguishes resource-coupled transport from machine-free and machine-side delays (Section~\ref{sec:problem}).
\item \textbf{Certified multi-bottleneck bound.} We construct a certified multi-bottleneck lower bound from chain, machine-capacity, and fleet-capacity relaxations. It is admissible at every reachable partial schedule (the machine term
  under the type-pure eligibility of the studied layouts), and its
  tightness at termination makes its one-step change an
  objective-faithful reward (Section~\ref{sec:theory}).
\item \textbf{Bound-oriented learning and decision support.} BOLT reuses the bound for certified action prices and a solver-free bound on the relative suboptimality gap of a completed schedule. Its fleet term also acts as a redundant constraint that raises the lower bound CP-SAT proves in the tested one-vehicle cells (Sections~\ref{sec:method}, \ref{sec:experiments}).
\item \textbf{Operating regime and deployment.} On the training-size grid, explicit coupling shortens makespan in every cell at or above
  $\UtilWinMin\%$ loaded utilization, whereas a transport penalty is adequate at or below
  $\UtilLoseMax\%$, so the tested crossover lies between these values. Across five 50- and 80-module production batches outside the trained sizes, the policy outperforms the best dispatching-rule pair on four and loses by $\StMixLongPdrEightyLow\%$ where one vehicle serves 80 modules. Its factorized realization costs up to $\MixSaLeadMax\%$ in makespan there against one merged head, gives the station and vehicle
  controllers separate decision layers, and dispatches in milliseconds (Section~\ref{sec:experiments}).
\end{itemize}

\section{Related Work}\label{sec:related}

Table~S-I compares the closest works
by setting, bound structure, bound uses, and factorization. Several
learning studies below address dynamic or stochastic flexible job-shop scheduling
problems (FJSP), whereas ours is deterministic.

\emph{Bound-guided learning for job shops.} Learning-to-dispatch methods
for FJSP commonly derive a reward at each decision from an optimistic
completion-time estimate. Zhang et al.\ \cite{zhang2020l2d} use changes in a
chain-based lower bound as the reward, Song et al.\
\cite{song2023fjspdrl} extend this to FJSP, and DANIEL
\cite{wang2024daniel} pairs the same principle with a dual-attention
encoder. With undiscounted returns ($\gamma=1$), the cumulative reward
is a constant minus the realized makespan, and the time-lag
work~\cite{paper1anon} proves the same property for machine-free lags.
That work schedules machines alone, with no fleet and with every wait
fixed by the instance. This paper reuses its reward design, instance
generator, and lag features. It adds the resource-coupled transport
delay, the fleet-capacity term, the candidate prices, the root
certificate, the solver constraint, and the regime study. The bound
changes because a scarce fleet can become the makespan bottleneck that a
chain relaxation alone does not represent.
Moon et al.\ \cite{moon2024learning} extend this line to FJSP
with transport, but their bound omits the fleet. Deep reinforcement learning schedulers that select
dispatching rules~\cite{luo2022hmappo} react to disturbances without a bound. BOLT adds machine- and fleet-capacity
relaxations and proves the composite bound for the coupled
states defined here.

\emph{FJSP with transport and coupled resources.} Bilge and Ulusoy
\cite{bilge1995time} introduced simultaneous machine and fleet
scheduling, Berterotti\`{e}re et
al.\ \cite{berterottiere2024flexible} give a recent model and
instance family, and Yao et al.\ \cite{yao2025fjspla} give a
mixed-integer model for a limited automated guided vehicle (AGV) fleet. Hurink and Knust~\cite{hurink2005tabu} apply tabu search to a job shop with one transport robot, and Xin et al.~\cite{xin2024simultaneous} schedule machines and AGVs with a hybrid algorithm. All five schedule vehicles by per-instance search, and none provides an
analytic bound at each state that a dispatcher could reuse for reward,
action guidance, and solver propagation. Learned transport-aware
schedulers avoid this per-instance search. They select dispatching rules for an AGV fleet~\cite{li2023dynamic}, add transfer-time constraints to dynamic FJSP~\cite{zhang2024transport}, or pair learned sequencing with a rule for AGV assignment~\cite{cheng2025cooperative}. Table~\ref{tab:semantics}
compares the closest by transport, time-lag, and reservation semantics and the
signal that trains them.
\begin{table}[!t]
\centering
\caption{Learned transport-aware schedulers closest to BOLT. n.s.: not
specified in the published description.}
\label{tab:semantics}
\scriptsize
\setlength{\tabcolsep}{2.5pt}
\renewcommand{\arraystretch}{1.05}
\begin{tabular}{@{}p{0.145\columnwidth}p{0.225\columnwidth}p{0.125\columnwidth}p{0.22\columnwidth}p{0.20\columnwidth}@{}}
\toprule
\raggedright Work & \raggedright Transport & \raggedright Time lags &
\raggedright Destination machine at dispatch &
\raggedright Training signal\tabularnewline
\midrule
\raggedright Moon et al.~\cite{moon2024learning} &
\raggedright seized vehicle, empty and loaded legs & \raggedright none &
\raggedright chosen with the vehicle, blocked from dispatch &
\raggedright chain estimate, no fleet term, uncertified\tabularnewline
\raggedright Li et al.~\cite{li2025realtime} &
\raggedright AGV allocated per task & \raggedright n.s. &
\raggedright n.s. &
\raggedright hand-designed (machine idle time)\tabularnewline
\raggedright Zhang et al.~\cite{zhang2024dynamic} &
\raggedright transfer time, no vehicle agent & \raggedright n.s. &
\raggedright not applicable &
\raggedright hand-designed shaping\tabularnewline
\raggedright Wang et al.~\cite{wang2025hierarchical} &
\raggedright transbots serve a fixed job-machine sequence & \raggedright n.s. &
\raggedright assigned before the move, hold n.s. &
\raggedright hierarchical, n.s.\tabularnewline
\raggedright Yang et al.~\cite{yang2026hgamppo} &
\raggedright AGV selection & \raggedright n.s. &
\raggedright n.s. &
\raggedright hand-designed\tabularnewline
\raggedright Link et al.~\cite{link2026coordination} &
\raggedright seized AGV, empty leg, anticipatory pickup &
\raggedright none &
\raggedright fixed by the route, not held &
\raggedright sparse terminal, bound as normalizer\tabularnewline
\midrule
\raggedright BOLT (this work) &
\raggedright seized vehicle, empty and loaded legs &
\raggedright machine-free lags &
\raggedright committed at release, held idle until arrival &
\raggedright certified chain, machine and fleet bound\tabularnewline
\bottomrule
\end{tabular}
\end{table}
None of the listed works reports a machine-free lag or a training signal with a proven bound. Only Moon et al.~\cite{moon2024learning} also hold the
destination machine idle while the vehicle is in transit, as BOLT does.

\emph{Multi-agent and factorized industrial scheduling.} Resource-class
factorization is established in AGV-assisted and multi-agent
scheduling, and we compare this deployment choice with a
merged-head control (Section~\ref{subsec:parity}). Li et al.\ \cite{li2025realtime} run multi-agent proximal policy
optimization (MAPPO) over priority-rule actions with a learned central
critic. Zhang et al.\ \cite{zhang2024dynamic} train job and machine
agents whose transport is a transfer delay with no vehicle to seize.
Wang et al.\ \cite{wang2025hierarchical} add a job, machine, and robot
hierarchy, and Yang et al.\ \cite{yang2026hgamppo}, closest in
formulation, factor an AGV-assisted policy by resource class. None of
these four trains on a certified bound, which BOLT adds with its proof. Factorized control also raises credit assignment, for which counterfactual multi-agent policy gradients (COMA)~\cite{foerster2018coma} learn a per-agent critic that extends
difference rewards~\cite{wolpert2001difference}.

Link et al.\ \cite{link2026coordination} examine when \emph{joint}
training of job and AGV agents outperforms \emph{modular} training on
job shops without machine flexibility or waits. We compare explicit
coupling with a transport penalty on a shop that has both, and we
evaluate BOLT on their benchmark (Section~\ref{subsec:external}).

\emph{Positioning.} To our knowledge, no learned FJSP scheduler combines
the certified multi-bottleneck bound with certified action pricing, a
solver-free certified gap bound, and solver strengthening. Prior multi-agent rewards are hand-designed or use a bound only as a
static normalizer~\cite{zhang2024dynamic,link2026coordination}.

\ifdefined\definition\else\newtheorem{definition}{Definition}\fi

\providecommand{\Jset}{\mathcal{J}}
\providecommand{\Mset}{\mathcal{M}}
\providecommand{\Vset}{\mathcal{V}}
\providecommand{\Cmax}{C_{\max}}
\providecommand{\loc}{\operatorname{loc}}

\section{Resource-Coupled FJSP with Time Lags}\label{sec:problem}

We study flexible job-shop scheduling with time lags and transport
(FJSP-TL-T), a deterministic flexible job shop in which jobs
rely on a scarce vehicle fleet between processing steps. It extends
the machine-free time-lag setting of~\cite{paper1anon}.

\subsection{Formulation}\label{subsec:fjsptlt}

\emph{Base flexible job shop with time lags (FJSP-TL).} An instance has
jobs $\Jset=\{1,\dots,n\}$ and machines $\Mset=\{1,\dots,m\}$.
Job $i$ is a linear chain of
operations $O_{i1},\dots,O_{in_i}$, and assigning $O_{ij}$ to an
eligible machine $k\in\Mset_{ij}\subseteq\Mset$ takes time
$p_{ijk}>0$. Each machine runs one operation at a time, without
preemption. When $O_{ij}$
finishes, its machine is released, but the job cannot begin
$O_{i,j+1}$ until a machine-free lag $\ell_{ij}\ge0$ elapses. Such lags
model curing, ponding, and drying, which occupy the module but no
machine~\cite{paper1anon}. Let $\mu_{ij}$ be the machine chosen for
$O_{ij}$ and $C_{ij}$ its completion time.

\emph{Transport extension.} The shop layout places every
machine $k$ at a location, or cell, $\loc(k)$. Carrying a job
from cell $a$ to cell $b$ takes travel time $\tau(a,b)\ge 0$, with
$\tau(a,a)=0$, and an empty repositioning move takes
$\tau_e(a,b)$, equal here to $\tau(a,b)$. A homogeneous
fleet $\Vset$ serves all moves, one vehicle per move. Jobs sit at the
machine that ran their last completed operation and enter the shop at
the first operation's machine with no preceding transport.

\begin{definition}[Transport task]\label{def:transport}
A job transition $O_{ij}\rightarrow O_{i,j+1}$ between
different cells, $\loc(\mu_{ij})\neq\loc(\mu_{i,j+1})$, induces a
transport task $\rho=(i,a,b)$ from $a=\loc(\mu_{ij})$ to
$b=\loc(\mu_{i,j+1})$. The job becomes ready for the move at
$C_{ij}+\ell_{ij}$, and
the task is released when its destination machine is committed
(Section~\ref{subsec:decision}). A same-cell successor requires no
transport and receives the job in place.
\end{definition}

Serving $\rho$ occupies one vehicle for the entire move
(Fig.~\ref{fig:problem}). Vehicle $v$, idle at $z$ and free from $w$,
starts at $\max(w,\text{release}(\rho))$. It runs the empty move
$z\rightarrow a$ followed by the loaded move $a\rightarrow b$. The vehicle ends idle
and free at $b$, where the job becomes ready.

\emph{Objective.} A solution assigns and sequences each operation on
an eligible machine, and each transport task on the fleet. The
objective is to minimize the makespan
\begin{equation}\label{eq:makespan}
  \Cmax \;=\; \max_{i\in\Jset}\; C_{i n_i}.
\end{equation}

\subsection{A Taxonomy of Scheduling Delays}\label{subsec:taxonomy}

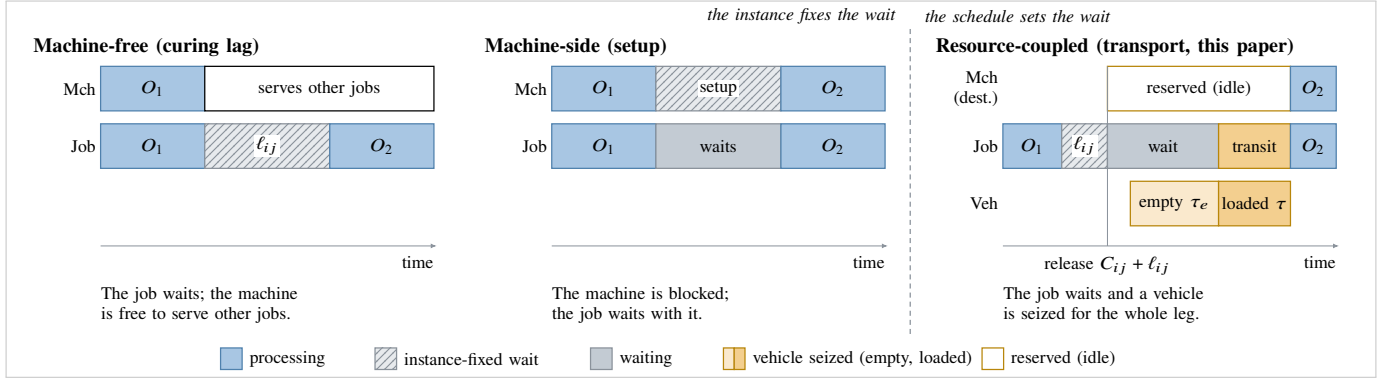
\begin{figure*}[!t]
\centering
\figframewide{\resizebox{0.97\linewidth}{!}{\definecolor{x2blf}{HTML}{A8C6E3}%
\definecolor{x2bln}{HTML}{4F81AD}%
\definecolor{x2amf}{HTML}{F3CF8F}%
\definecolor{x2amn}{HTML}{B8860B}%
\definecolor{x2gyf}{HTML}{C3CBD3}%
\definecolor{x2gyn}{HTML}{8A929B}%
\providecommand{\xFttl}{\fontsize{8.0}{9.0}\selectfont}%
\providecommand{\xFlane}{\fontsize{7.0}{8.0}\selectfont}%
\providecommand{\xFbar}{\fontsize{7.0}{8.0}\selectfont}%
\providecommand{\xFnote}{\fontsize{7.0}{8.0}\selectfont}%
\begin{tikzpicture}[x=1cm,y=1cm,font=\xFbar,
  every node/.style={text=black},
  proc/.style ={draw=x2bln,fill=x2blf,line width=0.5pt},
  vehl/.style ={draw=x2amn,fill=x2amf,line width=0.5pt},
  vehe/.style ={draw=x2amn,fill=x2amf!45,line width=0.5pt},
  resv/.style ={draw=x2amn,fill=none,line width=0.5pt},
  idle/.style ={draw=x2gyn,fill=x2gyf,line width=0.5pt},
  free/.style ={draw=black,fill=white,line width=0.5pt},
  ttl/.style  ={font=\xFttl\bfseries,inner sep=0pt,align=left},
  tg/.style   ={font=\xFlane,inner sep=0pt},
  nt/.style   ={font=\xFnote,inner sep=0pt,align=left},
  bl/.style   ={font=\xFbar,inner sep=0pt}]

\def\LH{0.62}          %
\def\YM{3.86}          %
\def\YJ{3.06}          %
\def\YV{2.26}          %
\def\YAX{1.98}         %
\def\YTIC{1.84}        %
\def\YNOTE{1.42}       %
\def\YTTL{4.62}        %

\newcommand{\tlbar}[5]{%
  \path[#4] ({\CX+\UU*(#1)},#3) rectangle ({\CX+\UU*(#2)},{#3+\LH});
  \node[bl] at ({\CX+\UU*((#1)+(#2))/2},{#3+0.5*\LH}) {#5};}
\newcommand{\tlhat}[4]{%
  \path[fill=x2gyf!40] ({\CX+\UU*(#1)},#3)
        rectangle ({\CX+\UU*(#2)},{#3+\LH});
  \path[pattern=north east lines,pattern color=x2gyn] ({\CX+\UU*(#1)},#3)
        rectangle ({\CX+\UU*(#2)},{#3+\LH});
  \draw[draw=x2gyn,line width=0.5pt] ({\CX+\UU*(#1)},#3)
        rectangle ({\CX+\UU*(#2)},{#3+\LH});
  \node[bl,fill=white,inner sep=0.6pt]
    at ({\CX+\UU*((#1)+(#2))/2},{#3+0.5*\LH}) {#4};}
\newcommand{\tlaxis}[1]{%
  \draw[draw=x2gyn,line width=0.4pt,-{Latex[length=2.4pt,width=2.0pt]}]
    ({#1+0.94},\YAX) -- ({#1+5.62},\YAX);
  \node[bl,anchor=north east] at ({#1+5.58},\YTIC) {time};}

\def\CX{0.94}\def\UU{1.45}
\node[ttl,anchor=south west] at (0.00,\YTTL) {Machine-free (curing lag)};
\node[tg,anchor=east] at (0.88,{\YM+0.5*\LH}) {Mch};
\node[tg,anchor=east] at (0.88,{\YJ+0.5*\LH}) {Job};
\tlbar{0}{1}{\YM}{proc}{$O_1$}
\tlbar{1}{3.2}{\YM}{free}{serves other jobs}
\tlbar{0}{1}{\YJ}{proc}{$O_1$}
\tlhat{1}{2.2}{\YJ}{$\ell_{ij}$}
\tlbar{2.2}{3.2}{\YJ}{proc}{$O_2$}
\tlaxis{0.00}
\node[nt,anchor=north west,text width=4.64cm] at (0.94,\YNOTE)
  {The job waits; the machine\\ is free to serve other jobs.};

\def\CX{7.22}\def\UU{1.45}
\node[ttl,anchor=south west] at (6.28,\YTTL) {Machine-side (setup)};
\node[tg,anchor=east] at (7.16,{\YM+0.5*\LH}) {Mch};
\node[tg,anchor=east] at (7.16,{\YJ+0.5*\LH}) {Job};
\tlbar{0}{1}{\YM}{proc}{$O_1$}
\tlhat{1}{2.2}{\YM}{setup}
\tlbar{2.2}{3.2}{\YM}{proc}{$O_2$}
\tlbar{0}{1}{\YJ}{proc}{$O_1$}
\tlbar{1}{2.2}{\YJ}{idle}{waits}
\tlbar{2.2}{3.2}{\YJ}{proc}{$O_2$}
\tlaxis{6.28}
\node[nt,anchor=north west,text width=4.64cm] at (7.22,\YNOTE)
  {The machine is blocked;\\ the job waits with it.};

\def\CX{13.50}\def\UU{0.909804}
\node[ttl,anchor=south west] at (12.56,\YTTL)
  {Resource-coupled (transport, this paper)};
\node[tg,anchor=east,align=right] at (13.44,{\YM+0.5*\LH}) {Mch\\(dest.)};
\node[tg,anchor=east] at (13.44,{\YJ+0.5*\LH}) {Job};
\node[tg,anchor=east] at (13.44,{\YV+0.5*\LH}) {Veh};
\tlbar{1.6}{4.4}{\YM}{resv}{reserved (idle)}
\tlbar{4.4}{5.1}{\YM}{proc}{$O_2$}
\tlbar{0}{0.9}{\YJ}{proc}{$O_1$}
\tlhat{0.9}{1.6}{\YJ}{$\ell_{ij}$}
\tlbar{1.6}{3.3}{\YJ}{idle}{wait}
\tlbar{3.3}{4.4}{\YJ}{vehl}{transit}
\tlbar{4.4}{5.1}{\YJ}{proc}{$O_2$}
\tlbar{1.95}{3.3}{\YV}{vehe}{empty $\tau_e$}
\tlbar{3.3}{4.4}{\YV}{vehl}{loaded $\tau$}
\tlaxis{12.56}
\node[nt,anchor=north west,text width=4.64cm] at (13.50,\YNOTE)
  {The job waits and a vehicle\\ is seized for the whole leg.};
\draw[draw=x2gyn,line width=0.4pt]
  ({\CX+\UU*1.6},{\YM+\LH}) -- ({\CX+\UU*1.6},\YAX);
\node[bl,anchor=north] at ({\CX+\UU*1.6},\YTIC) {release $C_{ij}+\ell_{ij}$};

\draw[draw=x2gyn,line width=0.5pt,dash pattern=on 2pt off 1.6pt]
  (12.21,5.00) -- (12.21,0.70);
\node[nt,anchor=east,font=\xFnote\itshape] at (12.01,5.18)
  {the instance fixes the wait};
\node[nt,anchor=west,font=\xFnote\itshape] at (12.41,5.18)
  {the schedule sets the wait};

\newcommand{\lgs}[3]{%
  \path[#2] (#1,0.26) rectangle (#1+0.30,0.56);
  \node[nt,anchor=west] at (#1+0.40,0.41) {#3};}
\lgs{2.61}{proc}{processing}
\path[fill=x2gyf!40] (4.76,0.26) rectangle (5.06,0.56);
\path[pattern=north east lines,pattern color=x2gyn]
  (4.76,0.26) rectangle (5.06,0.56);
\draw[draw=x2gyn,line width=0.5pt] (4.76,0.26) rectangle (5.06,0.56);
\node[nt,anchor=west] at (5.16,0.41) {instance-fixed wait};
\lgs{7.76}{idle}{waiting}
\path[vehe] (9.61,0.26) rectangle (9.76,0.56);
\path[vehl] (9.76,0.26) rectangle (9.91,0.56);
\node[nt,anchor=west] at (10.01,0.41) {vehicle seized (empty, loaded)};
\lgs{13.21}{resv}{reserved (idle)}

\end{tikzpicture}}}
\caption{The three delay-semantics classes
(Definition~\ref{def:taxonomy}), with
machine (Mch), job, and, where one is used, vehicle (Veh) lanes. Blue marks processing, hatching marks instance-fixed waits, gray marks the job waiting with no resource seized, amber marks the empty ($\tau_e$) and loaded ($\tau$) legs that seize the vehicle, and an outline marks the destination machine reserved idle until the job arrives. The vertical dashed rule marks the boundary of the class
studied here.}
\label{fig:problem}
\end{figure*}

\begin{definition}[Delay-semantics taxonomy]\label{def:taxonomy}
Let a delay be a mandatory wait attached to a scheduled entity. We
distinguish three semantics according to what the wait blocks and what
determines its duration. A \emph{machine-free} delay blocks the job
only, and its duration is fixed by the instance. The post-operation
lag $\ell_{ij}$ is of this type~\cite{paper1anon}. A
\emph{machine-side} delay blocks the machine and consumes no third
resource. Its duration is fixed by the instance, possibly as a
function of committed choices such as the preceding
job~\cite{allahverdi2015setup}. A \emph{resource-coupled} delay blocks
the job \emph{and} seizes a third resource whose availability the
schedule itself determines. Its effective duration, the time from the
job becoming ready to its arrival, is therefore set by the fleet's
assignment and sequencing across all jobs. Transport in FJSP-TL-T is
of this type.
\end{definition}

\emph{Why resource coupling requires a new bound.} In the first two
classes, every wait has a minimum duration fixed by the
instance~\cite{paper1anon}, and a chain recursion can use that
minimum for each uncommitted delay while remaining admissible and
tight. The same approach fails for a resource-coupled delay. A
per-transition minimum remains admissible but can be arbitrarily
loose, because it ignores
competition among jobs for the fleet. This
motivates the fleet-capacity component
(Remark~\ref{rem:fleet}).

\subsection{Scope and Modeling Assumptions}\label{subsec:scope}

Sequence-dependent setups and machine-availability constraints are
excluded to isolate the effect of the shared fleet. The model treats transport as occupancy
of a scarce resource with a tabulated travel time between cells, and it
models no road network, congestion, or route optimization. Vehicle charging, collisions,
heterogeneous fleets, station-buffer limits, and
factory-to-site delivery are outside the model. Station buffers are assumed unbounded. A finished module
therefore waits for its vehicle in the cell of its last machine, which
stays free for other operations. Blocking machines, which hold a
finished module until a vehicle collects it, and curing stages that
seize a capacity-limited resource create resource-coupled delays but
lie outside the model. Each lag $\ell_{ij}$ sets only a minimum wait, so maximum
time lags, which cap that wait, are outside both the taxonomy and the
model.

\subsection{Decision Model}\label{subsec:decision}

\emph{Machine-class event.} A job is decision-ready at
$C_{ij}+\ell_{ij}$ (time zero for an unstarted job). The machine-class
event occurs at
$t_M=\min_{(O_{ij},k)}\max(\text{ready}(O_{ij}),\text{free}(k))$
over eligible pairs. Its action space holds every eligible pair whose
own maximum is at most $t_M$, and committing a pair assigns the
operation to that machine. A same-cell successor
(Definition~\ref{def:transport}) then starts at once, because its
machine is free and its job present. By contrast, a different-cell successor
releases a transport task at $\max(\jr(i),b_m)$, the later of the
job-ready time and the machine's busy-until time. It also
\emph{reserves} its machine, which stays idle until the job arrives,
because a committed station cannot be reassigned in transit without
unloading. Its
processing begins on arrival.
The CP-SAT reference model searches a larger space with two different
timing rules. It drops the reservation,
and it lets a vehicle run its empty leg before the task is released and
wait at the pickup cell, while the environment starts the empty leg at
$\max(w,\text{release}(\rho))$. The gap to the solver therefore includes
both rules (quantified in Section~\ref{subsec:main}).
Relaxing the reservation under online dispatch remains future work.

\emph{Vehicle-class event.} A vehicle-class event occurs when a vehicle is
free and the pool of released, unserved transport tasks is non-empty. Its
action space is the pending tasks (at most one per job), and committing
dispatches that vehicle with the empty-plus-loaded timing above.

\emph{One decision per step and two event classes.} At each step,
pending events are ordered by a deterministic tie-break (time, class
with machine first, index), and the earliest event commits one
action. Each event class is served by one policy shared across its
resource instances~\cite{wang2024daniel}. Actions that would not
change the schedule are masked out, and a virtual ``pass'' in which
the event commits nothing is not available to the policy. It exists
only
as the counterfactual reference that defines the analytic credit
baseline of Section~\ref{sec:bcb}. Construction terminates when all completions
are realized, yielding the makespan of Eq.~\eqref{eq:makespan}. Both the construction and the bound of Section~\ref{sec:theory} depend only on the partial-schedule state. The same policy can therefore re-plan after a breakdown
by pricing the next event from the updated state
(Section~\ref{subsec:latency}).

\providecommand{\Cmax}{C_{\max}}
\providecommand{\Ccls}{\mathcal{C}}
\providecommand{\Fcls}{\mathcal{F}}
\providecommand{\elig}{\mathrm{elig}}
\providecommand{\cell}{\mathrm{cell}}
\providecommand{\Bchain}{B_{\mathrm{chain}}}
\providecommand{\Bmch}{B_{\mathrm{mch}}}
\providecommand{\Bveh}{B_{\mathrm{veh}}}
\providecommand{\jr}{\mathrm{jr}}

\section{The Certified Multi-Bottleneck Bound}\label{sec:theory}

BOLT reuses one lower bound on the makespan attainable from a partial
schedule. The bound is the maximum of three relaxations.

\subsection{Instance, State, and Completion Classes}

Fix one instance. $\elig(o)$ is the
eligibility set of operation $o$, $p(o,m)$ its processing time,
$p_{\min}(o)=\min_{m\in\elig(o)}p(o,m)$, and $\ell(o)$ its machine-free lag. Travel times need only be nonnegative with $\tau(a,a)=0$ and
$\tau_e\ge0$. No proof uses the triangle inequality, so every
guarantee holds on an arbitrary travel matrix, and $\tau_e=\tau$ here.
The static minimal travel envelope is
$\delta_{\min}(o)=\min_{a\in\cell(\elig(\mathrm{pred}\,o)),\,b\in\cell(\elig o)}\tau(a,b)$
($0$ for first operations), and a transition is \emph{mandatory} iff
$\delta_{\min}(o)>0$. For a move of $o$, $\tau_e$ is the empty leg to the pickup and $\tau_l$ the loaded leg.

\begin{assumption}[Type-pure eligibility]\label{as:pure}
For every $o$, all machines in $\elig(o)$ share one cell and can process exactly
the same operations, so $\elig(o)$ equals a single eligibility group
$g(o)$ of interchangeable parallel machines. This PPVC layout property holds by construction in the
generator.
\end{assumption}

\noindent Only the machine component $\Bmch$ needs
Assumption~\ref{as:pure}. Dropping
that component ($\chi_{\mathrm m}=0$ in Theorem~\ref{thm:main}) leaves
$\max(\Bchain,\Bveh)$, which carries all three guarantees below on
arbitrary eligibility sets.

At a partial-schedule state $s$ the operations partition into the
\emph{realized} (started, completion known), \emph{in-transit}
(committed to a different-cell machine, unserved), and
\emph{uncommitted} sets. Here $b_m$ is
machine $m$'s busy-until time, $f_v$ vehicle $v$'s free time, and
$\jr(i)$ job $i$'s ready time.

\begin{definition}[Constructive class $\Ccls(s)$]\label{def:C}
$\Ccls(s)$ is the set of completions the environment reaches from
$s$: every uncommitted operation gets an eligible machine and start, every
mandatory move a vehicle, respecting precedence, lags, eligibility,
machine no-overlap, and the reservation and service timing of
Section~\ref{subsec:decision} (task released at
$\mathrm{rel}(o)=\max(\jr(i),b_m)$, operation starting at
$\max(f_v,\mathrm{rel}(o))+\tau_e+\tau_l$). For $\pi\in\Ccls(s)$,
$\gamma_\pi(o)$ is the completion time,
$\Cmax(\pi)=\max_o\gamma_\pi(o)$.
\end{definition}

Committing shrinks continuations ($\Ccls(s')\subseteq\Ccls(s)$). The
constraint-programming solver CP-SAT searches the strictly larger space
$\Fcls(s)\supseteq\Ccls(s)$, which drops the reservation and lets a
vehicle start its empty leg before the task is released.

\subsection{Three Bottleneck Relaxations and Their Composition}

Each component isolates a resource that can determine the makespan.
The \emph{chain component} is
$\Bchain(s)=\max_o c(o)$, where the per-operation floor $c(o)$ is
seeded by the travel-aware chain recursion
$c(o)=c(o{-}1)+\ell(o{-}1)+\delta_{\min}(o)+p_{\min}(o)$. Each floor is raised at commitments together with the remaining
suffix of its chain, and an in-transit operation carries the optimistic
envelope $\hat c=\mathrm{rel}(o)+\tau_l+p(o,m)$. The \emph{machine-capacity
component} is
$\Bmch(s)=\max_g(\mathrm{busy}(g)+W_g)/|g|$ under
Assumption~\ref{as:pure}, with $\mathrm{busy}(g)=\sum_{m\in g}b_m$
and $W_g$ its group's uncommitted minimal
work. The \emph{fleet-capacity component}, the aggregate term that
resource coupling requires (Section~\ref{subsec:taxonomy}), is
$\Bveh(s)=\min_v f_v+W_{\mathrm{tr}}/|V|$, taken as $0$ when
$W_{\mathrm{tr}}\le0$. Here $W_{\mathrm{tr}}$ sums the exact loaded
durations of released-but-unserved tasks and the minimal envelopes
of uncommitted mandatory moves. Table~S-XXXII summarizes
what each component captures and relaxes.

\begin{remark}[Why a fleet component is necessary]\label{rem:fleet}
Take $K$ jobs whose first operations run in parallel, each on its own
machine in one cell, and whose successors all require the same
inter-cell move, served by one vehicle with loaded time $\tau$. A chain
relaxation charges at most one optimistic move along each job path, and
each machine group carries constant load, so
$\max(\Bchain,\Bmch)$ stays fixed as $K$ grows. Yet every feasible
schedule serializes the $K$ loaded moves on the one vehicle, so
$\Cmax\ge K\tau$. The fleet component closes this gap by aggregating
the mandatory transport workload over the vehicles:
$\Bveh=K\tau$ at the root (the exact instance is worked out as
Remark~S2 in the supplement). In general, every vehicle is free at the root, so
$\Bveh(s_0)=W_{\mathrm{tr}}/|V|$ divides the total minimal loaded travel
by the fleet size, which is the aggregate workload bound of energetic
reasoning for cumulative resources~\cite{baptiste1999satisfiability}. This paper adds its
admissibility at every reachable state with destination
reservations and its reuse as reward, action price, certificate, and
solver constraint.
\end{remark}

The maximum selects the active bottleneck. When machines or vehicles
are scarce, the corresponding term takes over, with no
tuned weight (confirmed in
Section~\ref{subsec:components}).

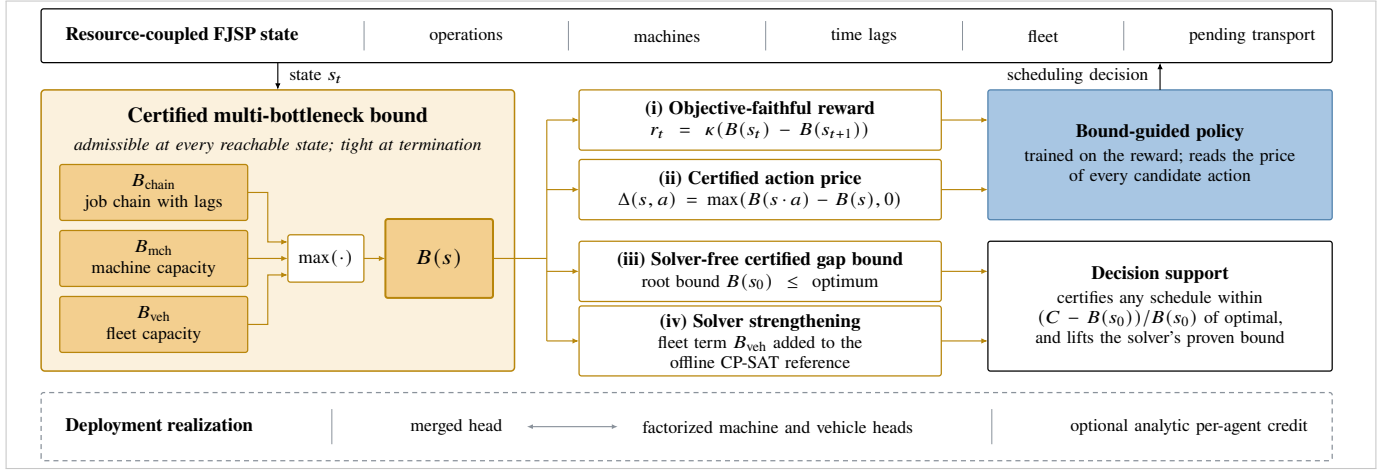
\begin{figure*}[t]
\centering
\figframewide{\IfFileExists{fig_bolt_body.tex}%
  {\resizebox{0.96\linewidth}{!}{\providecommand{\Bchain}{B_{\mathrm{chain}}}%
\providecommand{\Bmch}{B_{\mathrm{mch}}}%
\providecommand{\Bveh}{B_{\mathrm{veh}}}%
\providecommand{\Cmax}{C_{\max}}%
\definecolor{x2blf}{HTML}{A8C6E3}%
\definecolor{x2bln}{HTML}{4F81AD}%
\definecolor{x2amf}{HTML}{F3CF8F}%
\definecolor{x2amn}{HTML}{B8860B}%
\definecolor{x2gyf}{HTML}{C3CBD3}%
\definecolor{x2gyn}{HTML}{8A929B}%
\providecommand{\xFbig}{\fontsize{8.5}{9.5}\selectfont}%
\providecommand{\xFttl}{\fontsize{7.5}{8.5}\selectfont}%
\providecommand{\xFbdy}{\fontsize{7.0}{8.0}\selectfont}%
\begin{tikzpicture}[x=1cm,y=1cm,font=\xFbdy,
  every node/.style={text=black},
  bx/.style   ={draw,line width=0.5pt,rounded corners=1pt,align=center,
                inner sep=3pt,text=black},
  amb/.style  ={bx,fill=x2amf,draw=x2amn},
  ambo/.style ={bx,fill=white,draw=x2amn},
  blu/.style  ={bx,fill=x2blf,draw=x2bln},
  wht/.style  ={bx,fill=white,draw=black},
  arr/.style  ={-{Latex[length=2.6pt,width=2.2pt]},line width=0.45pt,
                draw=black},
  aarr/.style ={-{Latex[length=2.6pt,width=2.2pt]},line width=0.45pt,
                draw=x2amn},
  lin/.style  ={line width=0.45pt,draw=x2amn},
  sep/.style  ={line width=0.4pt,draw=x2gyn},
  lb/.style   ={inner sep=1pt,text=black}]

\draw[fill=white,draw=black,line width=0.5pt,rounded corners=1pt]
  (0,5.55) rectangle (18.00,6.30);
\coordinate (bt) at (0,6.13);
\coordinate (bb) at (0,5.72);
\node[lb,anchor=west] (sl) at (0.30,5.925)
  {\xFttl\bfseries Resource-coupled FJSP state};
\node[lb,anchor=west] (s1) at ([xshift=1.74cm]sl.east) {operations};
\node[lb,anchor=west] (s2) at ([xshift=1.74cm]s1.east) {machines};
\node[lb,anchor=west] (s3) at ([xshift=1.74cm]s2.east) {time lags};
\node[lb,anchor=west] (s4) at ([xshift=1.74cm]s3.east) {fleet};
\node[lb,anchor=west] (s5) at ([xshift=1.74cm]s4.east) {pending transport};
\coordinate (p0) at ([xshift=0.87cm]sl.east);
\coordinate (p1) at ([xshift=0.87cm]s1.east);
\coordinate (p2) at ([xshift=0.87cm]s2.east);
\coordinate (p3) at ([xshift=0.87cm]s3.east);
\coordinate (p4) at ([xshift=0.87cm]s4.east);
\draw[sep] (p0|-bt) -- (p0|-bb);
\draw[sep] (p1|-bt) -- (p1|-bb);
\draw[sep] (p2|-bt) -- (p2|-bb);
\draw[sep] (p3|-bt) -- (p3|-bb);
\draw[sep] (p4|-bt) -- (p4|-bb);

\draw[fill=x2amf!30,draw=x2amn,line width=0.6pt,rounded corners=1.5pt]
  (0,1.25) rectangle (6.60,5.17);
\node[anchor=north,font=\xFbig\bfseries] at (3.30,5.05)
  {Certified multi-bottleneck bound};
\node[anchor=north,font=\xFbdy\itshape] at (3.30,4.63)
  {admissible at every reachable state; tight at termination};

\node[amb,minimum width=2.62cm,minimum height=0.78cm,text width=2.36cm]
  (bc) at (1.57,3.74) {$\Bchain$\\[0.3pt]job chain with lags};
\node[amb,minimum width=2.62cm,minimum height=0.78cm,text width=2.36cm]
  (bm) at (1.57,2.82) {$\Bmch$\\[0.3pt]machine capacity};
\node[amb,minimum width=2.62cm,minimum height=0.78cm,text width=2.36cm]
  (bv) at (1.57,1.90) {$\Bveh$\\[0.3pt]fleet capacity};

\node[ambo,minimum width=1.06cm,minimum height=0.66cm] (mx)
  at (3.97,2.82) {$\max(\cdot)$};
\node[amb,minimum width=1.50cm,minimum height=1.10cm,line width=0.8pt]
  (bs) at (5.55,2.82) {\xFbig$B(s)$};

\draw[aarr] (2.88,3.74) -- (3.16,3.74) -- (3.16,3.05) -- (3.44,3.05);
\draw[aarr] (2.88,2.82) -- (3.44,2.82);
\draw[aarr] (2.88,1.90) -- (3.16,1.90) -- (3.16,2.59) -- (3.44,2.59);
\draw[aarr] (4.50,2.82) -- (4.80,2.82);

\draw[lin] (6.30,2.82) -- (7.05,2.82);
\draw[lin] (7.05,1.67) -- (7.05,4.75);
\draw[aarr] (7.05,4.75) -- (7.50,4.75);
\draw[aarr] (7.05,3.78) -- (7.50,3.78);
\draw[aarr] (7.05,2.64) -- (7.50,2.64);
\draw[aarr] (7.05,1.67) -- (7.50,1.67);

\node[ambo,minimum width=5.05cm,minimum height=0.84cm,text width=4.72cm]
  (u1) at (10.025,4.75)
  {{\xFttl\bfseries (i) Objective-faithful reward\\[0.5pt]}%
   $r_t=\kappa(B(s_t)-B(s_{t+1}))$};
\node[ambo,minimum width=5.05cm,minimum height=0.84cm,text width=4.72cm]
  (u2) at (10.025,3.78)
  {{\xFttl\bfseries (ii) Certified action price\\[0.5pt]}%
   $\Delta(s,a)=\max(B(s\!\cdot\!a)-B(s),0)$};
\node[ambo,minimum width=5.05cm,minimum height=0.84cm,text width=4.72cm]
  (u3) at (10.025,2.64)
  {{\xFttl\bfseries (iii) Solver-free certified gap bound\\[0.5pt]}%
   root bound $B(s_0)\le\mathrm{optimum}$};
\node[ambo,minimum width=5.05cm,minimum height=0.84cm,text width=4.72cm]
  (u4) at (10.025,1.67)
  {{\xFttl\bfseries (iv) Solver strengthening\\[0.5pt]}%
   fleet term $\Bveh$ added to the offline CP-SAT reference};

\draw[aarr] (12.55,4.75) -- (13.20,4.75);
\draw[aarr] (12.55,3.78) -- (13.20,3.78);
\draw[aarr] (12.55,2.64) -- (13.20,2.64);
\draw[aarr] (12.55,1.67) -- (13.20,1.67);

\node[blu,minimum width=4.80cm,minimum height=1.81cm,text width=4.46cm]
  (pol) at (15.60,4.265)
  {{\xFttl\bfseries Bound-guided policy\\[1.5pt]}%
   trained on the reward; reads the price\\ of every candidate action};
\node[wht,minimum width=4.80cm,minimum height=1.81cm,text width=4.46cm]
  (dsu) at (15.60,2.155)
  {{\xFttl\bfseries Decision support\\[1.5pt]}%
   certifies any schedule within\\ $(C-B(s_0))/B(s_0)$ of optimal,\\
   and lifts the solver's proven bound};

\draw[arr] (3.30,5.55) -- (3.30,5.17);
\node[lb,anchor=west] at (3.42,5.36) {state $s_t$};
\draw[arr] (15.60,5.17) -- (15.60,5.55);
\node[lb,anchor=east] at (15.48,5.36) {scheduling decision};

\draw[draw=x2gyn,line width=0.5pt,rounded corners=1pt,
      dash pattern=on 2pt off 1.6pt] (0,0) rectangle (18.00,0.95);
\coordinate (dt) at (0,0.72);
\coordinate (db) at (0,0.23);
\node[lb,anchor=west] (dl) at (0.30,0.475)
  {\xFttl\bfseries Deployment realization};
\node[lb,anchor=west] (d1) at ([xshift=2.11cm]dl.east) {merged head};
\node[lb,anchor=west] (d2) at ([xshift=1.87cm]d1.east)
  {factorized machine and vehicle heads};
\node[lb,anchor=west] (d3) at ([xshift=2.11cm]d2.east)
  {optional analytic per-agent credit};
\draw[{Latex[length=2.4pt,width=2.0pt]}-{Latex[length=2.4pt,width=2.0pt]},
      line width=0.45pt,draw=x2gyn]
  ([xshift=0.30cm]d1.east) -- ([xshift=-0.30cm]d2.west);
\coordinate (q0) at ([xshift=1.055cm]dl.east);
\coordinate (q1) at ([xshift=1.055cm]d2.east);
\draw[sep] (q0|-dt) -- (q0|-db);
\draw[sep] (q1|-dt) -- (q1|-db);

\end{tikzpicture}}}%
  {\parbox[c][3.2cm][c]{0.9\textwidth}{\centering\footnotesize Figure
   placeholder: BOLT framework schematic (fig\_bolt\_body, rendering)}}}
\caption{BOLT framework with one certified bound
$B=\max(\Bchain,\Bmch,\Bveh)$ and its four uses (i)--(iv). Amber marks the
bound and quantities derived from it, blue the learned policy, and the
dashed band optional deployment extensions.}
\label{fig:architecture}
\end{figure*}

\subsection{Guarantees}

\begin{proposition}[Component admissibility]\label{prop:components}
At every reachable $s$ and every $\pi\in\Ccls(s)$, each of $\Bchain(s)$,
$\Bmch(s)$, and $\Bveh(s)$ is at most $\Cmax(\pi)$. For every realized operation $o$, its chain floor satisfies
$c(o)=\gamma_\pi(o)$.
\end{proposition}

BOLT's reward and action prices are
differences of $B$, so their meaning rests on the properties below.

\begin{theorem}[Admissibility, tightness, faithfulness]\label{thm:main}
Let $\chi_{\mathrm m},\chi_{\mathrm v}\in\{0,1\}$ select the machine and vehicle
components and set
$B(s)=\max(\Bchain(s),\chi_{\mathrm m}\Bmch(s),
\chi_{\mathrm v}\mathbf 1_{|V|>0}\Bveh(s))$. The main result takes
$\chi_{\mathrm m}=\chi_{\mathrm v}=1$.
\begin{enumerate}
\item[(a)] \emph{(Admissibility.)} For every reachable $s$ and every
$\pi\in\Ccls(s)$, $B(s)\le\Cmax(\pi)$.
\item[(b)] \emph{(Tightness.)} At termination $s_T$, $B(s_T)=\Cmax$, the
realized makespan.
\item[(c)] \emph{(Faithfulness.)} With per-step reward
$r_t=(B(s_t)-B(s_{t+1}))\,\kappa$ for an instance-fixed constant
$\kappa=1/(\overline p-\underline p+\varepsilon)>0$ ($\overline
p,\underline p$ the instance's extreme processing times,
$\varepsilon>0$ a guard) and undiscounted return
($\gamma=1$), $\sum_{t=0}^{T-1}r_t=(B(s_0)-\Cmax)\,\kappa$. Since $B(s_0)$ is
instance-fixed, $\argmax$ of the return equals $\argmin$ of $\Cmax$.
\end{enumerate}
\end{theorem}
\noindent\emph{Proof:} in the supplement. Remark~S4 there relates the
reward to potential-based shaping.

Faithfulness follows from (b) and from $B(s_0)$ being fixed by the
instance, so it would hold for any state function that equals the
makespan at termination. Admissibility (a) gives each intermediate
value its meaning, so each negative reward and each positive action
price is a proven rise of the makespan floor, and the root value bounds
the optimum (Corollary~\ref{cor:root}).

\subsection{Certified Action Prices and the Root Certificate}

Admissibility makes the bound a certified one-step decision
signal and a solver-free quality certificate.

\begin{corollary}[Admissible action prices]\label{cor:guide}
Let $s$ be reachable, let $a$ be a feasible commitment at $s$ (an
operation-machine assignment or a task-vehicle service), and let
$s\!\cdot\!a$ be its successor. Every completion of $s$ that begins
with $a$ lies in $\Ccls(s\!\cdot\!a)$, so Theorem~\ref{thm:main}(a) at
$s\!\cdot\!a$ gives
$\Cmax(\pi)\ge\max(B(s\!\cdot\!a),B(s))=B(s)+\Delta(s,a)$ for every
such $\pi$, with $\Delta(s,a)=\max(B(s\!\cdot\!a)-B(s),0)$ the action
price of Section~\ref{sec:guide}. Thus $B(s)+\Delta(s,a)$ is an admissible lower bound on every
completion that begins with $a$: a zero price certifies that the
commitment leaves the certified makespan floor intact, and a positive price is a proven floor increase.
\end{corollary}
\begin{proof}
$\Ccls(s\!\cdot\!a)\subseteq\Ccls(s)$, so Theorem~\ref{thm:main}(a)
holds at both states; take the maximum.
\end{proof}

\begin{corollary}[Root lower bound on the true optimum]\label{cor:root}
At the root $s_0$ (no in-transit operations, $b_m=f_v=0$) all three relaxations are
valid for the reservation-free space $\Fcls(s_0)$ as well, so
$B(s_0)\le\mathrm{opt}_{\Fcls}(s_0)\le\mathrm{opt}_{\Ccls}(s_0)$.
\end{corollary}

This holds at the root only, because a reservation can push $B$ above
$\mathrm{opt}_{\Fcls}$ once an operation is in transit (Remark~S3 in the
supplement). The solver model is built only at $s_0$, so no reported
quantity depends on the difference. For any completed
schedule of makespan $C$, the root bound certifies $C$ as within
$(C-B(s_0))/B(s_0)$ of optimal without running a solver.

\section{BOLT: Bound-Oriented Learning}\label{sec:method}

\subsection{Framework Overview}
\label{sec:overview}

BOLT uses
the certified bound of Section~\ref{sec:theory} throughout training and
deployment (Fig.~\ref{fig:architecture}). At each event, the environment
exposes the partial schedule and the feasible commitments, and the bound
supplies the reward and one price per candidate. These signals require
no particular policy architecture, so the merged-head and factorized
realizations of Section~\ref{sec:realization} consume the same signals
through one dual-attention encoder.

\subsection{Objective-Faithful Bound-Difference Reward}
\label{sec:reward}

Following bound-difference rewards for job
shops~\cite{zhang2020l2d,song2023fjspdrl,paper1anon}, BOLT emits
at each decision the difference
$r_t=\kappa\,(B(s_t)-B(s_{t+1}))$, with $\kappa$ the positive
instance-fixed scale of Theorem~\ref{thm:main}. This construction is standard, and the contribution is the certified
composite bound under resource coupling. The undiscounted return is a
scaled and shifted $-\Cmax$, so maximizing it minimizes
makespan exactly
(Theorem~\ref{thm:main}(c)).

\subsection{Certified Action Prices}
\label{sec:guide}

The central learning mechanism in BOLT is candidate-specific bound
guidance. For a feasible commitment $a$ at state $s$ with successor
$s\!\cdot\!a$, BOLT defines the
action price $\Delta(s,a)=\max(B(s\!\cdot\!a)-B(s),0)$. The price is
a one-step lookahead of the bound, equal to the positive part of the
reward cost $-r_t/\kappa$ that the commitment would incur. It differs
from a heuristic penalty because the successor bound is admissible.
By Corollary~\ref{cor:guide}, $B(s)+\Delta(s,a)$ is a valid lower
bound on any schedule beginning with that commitment. All candidate
prices are computed in one pass, and the policy reads them as one input
channel, never as a mask or a reward term. The policy may still
select an apparently expensive action when other learned features
justify it. Admissibility does not itself certify that the ranking improves
schedules, which Section~\ref{subsec:credit} tests separately.

\subsection{Optional Analytic Credit}
\label{sec:bcb}

The bound also yields an optional analytic credit for the head that
acts at each step. It subtracts the reward of the virtual pass state of
Section~\ref{subsec:decision},
$r_t^{\mathrm{pass}}=(B(s_t)-B(s_t^{\mathrm{pass}}))\,\kappa$, from the
generalized-advantage-estimation advantage~\cite{schulman2016gae}.
Because $r_t^{\mathrm{pass}}$ does not depend on the sampled action, the
subtraction leaves the gradient estimator unbiased (lemma in the
supplement), and Section~\ref{subsec:credit} finds that it gives no
scheduling gain.

\subsection{Policy Realization and Deployment}
\label{sec:realization}

BOLT has two actor heads over one shared encoder, $\pi_M$ for
machine-class events and $\pi_V$ for vehicle-class events. The encoder
follows DANIEL's dual-attention design~\cite{wang2024daniel} over
operation and machine nodes. Vehicles enter the machine-side attention
stream as a third node set, projected to machine-feature width with a
learned type embedding. Operation nodes retain the lag channels
of~\cite{paper1anon}, and both node types gain anticipatory transport
channels. Exactly one head acts at each event, so training is proximal
policy optimization (PPO) over two event-class heads with one critic
that reads the pooled embedding. We use the MAPPO
implementation~\cite{yu2022mappo}, which reduces to this case when one
head acts per step.

The two heads mirror how plants assign these decisions. A
manufacturing execution system typically dispatches stations, and a
vehicle or crane controller dispatches the fleet. Separate output heads
give each decision class its own decision-layer parameters, so one
class can be fine-tuned, validated, or rolled back while the other
class's head stays fixed. The encoder is shared, so an encoder change
affects both classes. Each controller runs a replica of the network,
receives a broadcast of the shop state, and acts only on its own event
class. A controller whose learned decisions are withdrawn can therefore
fall back to a dispatching rule while the other keeps its learned
decisions (Section~\ref{subsec:parity}). The merged head, drawn on one controller in
Fig.~\ref{fig:deploy}, can be replicated in the same way, and the factorized
form differs only in that the two classes share no output layer. No experiment here measures the value of
this structural property.

\providecommand{\Cmax}{C_{\max}}
\providecommand{\Vset}{\mathcal{V}}
\providecommand{\Bchain}{B_{\mathrm{chain}}}
\providecommand{\Bmch}{B_{\mathrm{mch}}}
\providecommand{\Bveh}{B_{\mathrm{veh}}}

\providecommand{\resultfig}[2]{%
  \IfFileExists{#1}%
    {\includegraphics[width=\linewidth]{#1}}%
    {\fbox{\parbox[c][3.4cm][c]{0.92\linewidth}{\centering\footnotesize
       Figure placeholder\\[2pt]#2\\[2pt](rendered when results land)}}}}

\section{Experiments}\label{sec:experiments}

\begin{figure}[!t]
\centering
\figframe{%
\includegraphics[width=\linewidth]{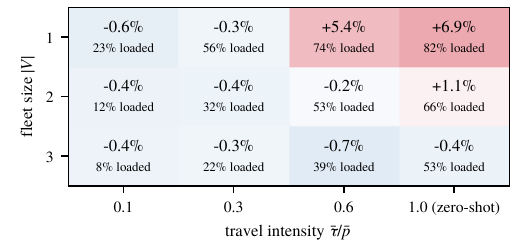}\\[1pt]
{\footnotesize (a)}\\[2pt]
\includegraphics[width=\linewidth]{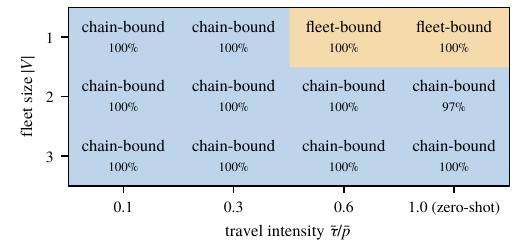}\\[1pt]
{\footnotesize (b)}}
\caption{(a) Coupling-regret map: relative makespan penalty of the
penalty anchor, a transport-penalty policy (one seed, half budget), against BOLT
(\Nseeds{}-seed mean; $n=\Ninst{}$ per cell, \primarysize).
The second line of each cell is the minimum loaded utilization. (b) The bound component attaining the root maximum, per cell (share of instances; blue chain, amber fleet); the
machine component never attains the maximum on this grid.}
\label{fig:coupling}
\end{figure}

\subsection{Experimental Setup}\label{subsec:setup}

\emph{Benchmark and regime grid.} The primary benchmark is the
PPVC-T generator.\footnote{Generator, instance data, baseline
references, trained models, and evaluation code: \url{https://github.com/NTUZZH/FJSP-TLT-MARL}.} It is based on the FJSP-TL
generator of~\cite{paper1anon}, whose routing classes and operation
structure follow the Singapore Building and Construction Authority's
PPVC documentation~\cite{bca2024ppvc}. The instances are therefore
industry-motivated, and none is taken from a plant's production
records. Most station types have two to four parallel stations, and
four routing classes visit them in different orders, which makes the
plant a flexible job shop. We add
a $3\times3$ station-cell layout with Manhattan travel matrix $\tau$
and a homogeneous fleet, keeping lag ranges unchanged. Policies are trained
once, on a mixture of fleet sizes $|V|\in\gridfleetset$ and travel
intensities $\bar\tau/\bar p\in\{0.1,0.3,0.6\}$ (mean shortest travel of a mandatory move
over mean processing time). The training size is \primarysize{} (jobs
$\times$ machines; 22 operations per job).
All test routes have 22 operations because the generator models each
documented routing class as 22 station visits.
The 20- and 30-module transfer cells bracket the per-floor batch of a
documented PPVC tower~\cite{unitedtec2022avenuesouth}, and the 50- and
80-module production batches cover two to three of its floors. Evaluation adds held-out fleet sizes, instance sizes, and travel intensities (Section~\ref{subsec:scale}). Each training-size cell holds \Ninst{} test instances and each production-batch cell holds 20, and the repository lists the fixed generator seeds of every test set.

\emph{What the travel setting means, and what binds.} A processing
step averages \ProcHr~h. Actual moves between different cells run
slightly longer than the shortest travel, about
\MoveMinTOne~min at $\bar\tau/\bar p=0.1$, \MoveHrTSix~h at $0.6$, and
\MoveHrTTen~h at $1.0$. The two longer settings exceed moves inside one
hall and represent moves between separate production areas or a slow
fleet, and none is calibrated to a plant's travel times. In the tested cells, BOLT outperforms
the penalty anchor of Section~\ref{subsec:regret} only with these
longer moves. With short moves ($0.1$), both the anchor and the GA are
at least as good as BOLT (Table~S-II, Fig.~\ref{fig:coupling}(a)). We describe fleet load by the minimum
loaded utilization, the share of fleet time spent carrying loads by any
schedule as short as the best one found. It runs from \UtilGridLow\% to
\UtilGridHigh\% over the grid and from \UtilEighty\% to
\UtilFiftyTTen\% on the five production batches
(Tables~\ref{tab:scale} and~S-V).

\emph{Baselines.} The first family contains
nine priority dispatching rule (PDR) pairs, three machine rules
crossed with three vehicle rules, named in the supplement. The second is a genetic algorithm (GA) whose random-key
chromosome holds operation priorities, machine preferences, and the
weights of a state-dependent vehicle rule. It starts from the
nine PDR pairs and uses population \gapopsize{}, elitism 5, tournaments
of two, uniform crossover, and Gaussian mutation (rate and scale 0.1).
Its search stops after \gawall{}~CPU-s beyond seeding unless stated
otherwise. Every chromosome is decoded through the same simulator, so
the GA searches the same constructive class as BOLT (specification in
the supplement, measured cost in Fig.~S-1). The third is a merged-head control. It
makes the same decisions through one decision head, with
the same encoder, signals, and
budget~\cite{song2023fjspdrl,moon2024learning}. The fourth is a
warm-started CP-SAT reference (OR-Tools 9.14, default parameters,
\cpbudget{}~s per instance, four workers, one run per instance). It is seeded from the best-of-nine PDR solution
and has the fleet term of Section~\ref{subsec:cert} as a redundant
constraint (model in the supplement). This time-limited reference
weakens at production scale. Without the warm start, it found no
feasible schedule within an hour on \ColdFiftyNone{} probed 50-module
instances, and at 80 modules its own proven gap after an hour
averages \ScCpOwnGapEighty\%. Existing vehicle-aware
learned schedulers do not support machine-free lags with these transport
and reservation semantics, so reproducing them on PPVC-T would change
their problem. BOLT is compared with published learned policies on the
independently released benchmark of~\cite{link2026coordination}
(Section~\ref{subsec:external}).

\emph{Protocol.} Every arm reported here is trained on \Nseeds{} seeds
\seedset{} and its makespans averaged per instance. Single-run arms (seed
301) are marked where they appear.
Comparisons are paired at the instance level, with each instance's
outcome the mean over the training seeds, so the tests quantify
instance variation. A two-level bootstrap over training seeds and instances excludes zero in \BootExclZeroRows{} of the \SeedRowsN{} headline comparison cells, and the other cells are not distinguishable at the seed level (Table~S-XI).
The best dispatching-rule pair is selected per cell on the same instances,
which favors the rule and makes that comparison conservative for BOLT.
We use Wilcoxon signed-rank tests at $\alpha=\siglevel{}$, Holm-corrected within each pre-specified cell family, and report per-instance win/tie/loss counts (W/T/L).
Parity claims use two one-sided tests (TOST) with a pre-specified
$\pm2\%$ margin, and all mechanism and equivalence claims follow
pre-specified decision rules (rules and dates in the supplement). Training is undiscounted
($\gamma=1$, as Theorem~\ref{thm:main}(c) requires) and runs
\TrainUpdates{} PPO updates
(\TrainGpuHours{} GPU-hours) per arm. A new station or operation type
has no learned type embedding, so it needs retraining at this cost
(Table~S-XIII). Within the tested ranges of
Sections~\ref{subsec:scale} and~\ref{subsec:components}, fleet size,
travel intensity, station count, and, after size-mixture training,
instance size can change without retraining. The best dispatching-rule pair
still outperforms the policy in three tested settings (the 10-module
policy at 15 modules with moderate travel, one vehicle for 80 modules,
and wide processing-time dispersion at moderate travel). An independent validator checks
every schedule (protocol and configuration in the supplement).

\subsection{Comparison Study}\label{subsec:comparison}

\subsubsection{Scheduling Quality on the Regime Grid}\label{subsec:main}

BOLT yields shorter makespans than the best of nine dispatching-rule pairs
in every cell of the primary grid
(Table~\ref{tab:main}). Margins are largest with one vehicle and heavy
travel (W/T/L $\qWtlVOneTSix$ at $|V|=1$,
$\bar\tau/\bar p=0.6$). Per-instance
search provides a different comparison. Given 60~CPU-s, the PDR-seeded
GA outperforms the one-shot policy at this training
size by $\qGaSpreadLow$ to $\qGaSpreadHigh\%$ on heavy cells
($\qWtlGaVOneTSix$; Fig.~S-1). Best-of-\SampN{} sampling narrows
that margin to $\SampDGaLow$--$\SampDGaHigh\%$ (Table~S-IV), so where a
minute of per-instance search is affordable, search remains the better
choice. The CP-SAT incumbents of Table~\ref{tab:main} are shorter still.
Re-timing BOLT's own assignments and sequences under the solver's two
timing rules (Section~\ref{subsec:decision}) shortens its schedules by
\RtImpMean\% on average and closes about \RtGapShare\% of that gap
(Table~S-XII). The reservation accounts for
\RtResShare\% of makespan and the early empty legs for \RtEmptyShare\%.
What remains of the gap mixes policy quality with sequences that the
constructive class cannot reach. Policy quality is part of it, because
the GA searches the same class and still outperforms BOLT.

\begin{table*}[!t]
\centering
\caption{Makespan on the FJSP-TL-T regime grid at \primarysize{} (jobs $\times$ machines): hours, mean$\pm$std ($n=\Ninst{}$ per cell; learned arms: per-instance \Nseeds{}-seed means). Bold: shortest makespan in the row. Light-travel column and merged-head control: supplement.}
\label{tab:main}
\renewcommand{\arraystretch}{0.86}
\setlength{\tabcolsep}{6pt}
\begin{threeparttable}
\scriptsize
\begin{tabular*}{\textwidth}{@{\extracolsep{\fill}}cc ccc cc@{}}
\toprule
\multicolumn{2}{c}{Regime cell} &
\multicolumn{3}{c}{Online dispatch} &
\multicolumn{2}{c}{Solver reference} \\
\cmidrule(lr){1-2}\cmidrule(lr){3-5}\cmidrule(lr){6-7}
$|V|$ & $\bar\tau/\bar p$ &
Best PDR\tnote{a} & BOLT & W/T/L\tnote{b} &
CP-SAT\tnote{c} & $\Delta$\%\tnote{d} \\
\midrule
$1$ & $0.3$ & \qPdrVOneTThree & \qOursVOneTThree & \qWtlVOneTThree & \textbf{\qCpVOneTThree} & \qGapVOneTThree \\
$1$ & $0.6$ & \qPdrVOneTSix & \qOursVOneTSix & \qWtlVOneTSix & \textbf{\qCpVOneTSix} & \qGapVOneTSix \\
$1$ & $1.0^{\ast}$ & \qPdrVOneTTen & \qOursVOneTTen & \qWtlVOneTTen & \textbf{\qCpVOneTTen} & \qGapVOneTTen \\
\midrule
$2$ & $0.3$ & \qPdrVTwoTThree & \qOursVTwoTThree & \qWtlVTwoTThree & \textbf{\qCpVTwoTThree} & \qGapVTwoTThree \\
$2$ & $0.6$ & \qPdrVTwoTSix & \qOursVTwoTSix & \qWtlVTwoTSix & \textbf{\qCpVTwoTSix} & \qGapVTwoTSix \\
$2$ & $1.0^{\ast}$ & \qPdrVTwoTTen & \qOursVTwoTTen & \qWtlVTwoTTen & \textbf{\qCpVTwoTTen} & \qGapVTwoTTen \\
\midrule
$3$ & $0.3$ & \qPdrVThreeTThree & \qOursVThreeTThree & \qWtlVThreeTThree & \textbf{\qCpVThreeTThree} & \qGapVThreeTThree \\
$3$ & $0.6$ & \qPdrVThreeTSix & \qOursVThreeTSix & \qWtlVThreeTSix & \textbf{\qCpVThreeTSix} & \qGapVThreeTSix \\
$3$ & $1.0^{\ast}$ & \qPdrVThreeTTen & \qOursVThreeTTen & \qWtlVThreeTTen & \textbf{\qCpVThreeTTen} & \qGapVThreeTTen \\
\bottomrule
\end{tabular*}
\begin{tablenotes}\footnotesize
\item[a] Best of the nine priority-dispatching-rule (PDR) pairs per
  cell. \textsuperscript{b}~W/T/L: instances where BOLT is shorter, tied, or longer than the best PDR pair; every cell is significant (paired one-sided Wilcoxon, Holm over the cells, largest adjusted $p=\qPdrHolmMax$).
\item[c] Warm-started CP-SAT (OR-Tools), \cpbudget{}~s/instance;
  incumbent shown. \textsuperscript{d}~$\Delta$: BOLT vs.\ the solver's incumbent, an
  upper bound (UB) on the optimum, $(\text{BOLT}-\text{UB})/\text{UB}$;
  not an optimality gap.
  \textsuperscript{$\ast$}~Zero-shot: outside the training mixture.
\end{tablenotes}
\end{threeparttable}
\end{table*}

\subsubsection{BOLT Remains Competitive under Zero-Shot Transfer to
Production Batches}\label{subsec:scale}

Three held-out directions test generalization
(Table~\ref{tab:main} zero-shot rows and Table~S-II). With a fourth
vehicle, a fleet size outside training, the 10-module policy's
makespan stays at or slightly below the best PDR pair's. At $15\times25$, the best PDR pair
overtakes both realizations of the 10-module policy at moderate but
not heavy travel, and that policy's margin shrinks as instances grow
(Fig.~S-3). Retraining on a $10/15/20$-module mixture
with the same architecture, reward, and budget puts the policy's mean
makespan below the best pair's in every cell of Fig.~S-3, from 15 to 80
modules.

\begin{table}[!t]
\centering
\caption{Makespan in hours on the 50- and 80-module production batches ($n=20$ per cell;
policy: per-instance \Nseeds{}-seed means). Bold: shortest makespan
in the row. The GA comparison is in Table~S-IV.}
\label{tab:scale}
\renewcommand{\arraystretch}{0.90}
\setlength{\tabcolsep}{2.6pt}
\begin{threeparttable}
\scriptsize
\begin{tabular}{@{}lc cc cc@{}}
\toprule
& & \multicolumn{2}{c}{Online} & \multicolumn{2}{c}{Solver reference} \\
\cmidrule(lr){3-4}\cmidrule(lr){5-6}
Cell & $u$ &
Best PDR & BOLT & CP-SAT 1800~s & CP-SAT 3600~s \\
\midrule
$50$, $2$, $0.6$ & \UtilFiftyTSix\% & \ScPdrFiftyTSix & \ScMixFiftyTSix & \ScCpHalfFiftyTSix & \textbf{\ScCpFiftyTSix} \\
$50$, $2$, $1.0$ & \UtilFiftyTTen\% & \ScPdrFiftyTTen & \textbf{\ScMixFiftyTTen} & \ScCpHalfFiftyTTen & \ScCpFiftyTTen \\
$80$, $3$, $1.0$ & \UtilEighty\% & \ScPdrEighty & \textbf{\ScMixEighty} & \ScCpHalfEighty & \ScCpEighty \\
\bottomrule
\end{tabular}
\begin{tablenotes}\footnotesize
\item Cell: modules, $|V|$, $\bar\tau/\bar p$ (25 machines); every
  cell is zero-shot (training: 10--20 modules, $\bar\tau/\bar
  p\le0.6$). BOLT vs.\ best PDR: W/T/L $\ScMixWtlFiftyTSix$,
  $\ScMixWtlFiftyTTen$, $\ScMixWtlEighty$, all
  $p\le\ScMixPdrPvalMax$, and each training seed alone is shorter than
  the best pair on average in every cell; vs.\ CP-SAT 3600~s at $50$, $2$, $1.0$:
  $\ScMixWtlCpFifty$, $p=\ScMixCpPval$. CP-SAT: warm-started, four workers, one run per
  instance; the 1800-s column holds the same runs' half-budget
  incumbents, and every verdict matches. $u$: minimum
  loaded utilization $W_{\mathrm{tr}}/(|V|\,C)$ (supplement).
\end{tablenotes}
\end{threeparttable}
\end{table}

\emph{On the production batches of Table~\ref{tab:scale}, the
size-mixture policy is shorter than the best dispatching-rule pair on every
instance}, by $\ScMixDPdrFiftyTSix\%$ and $\ScMixDPdrFiftyTTen\%$ at
50 modules and $\ScMixDPdrEighty\%$ at 80 modules. Two one-vehicle cells with short moves load the fleet as heavily
(column $u$ and Table~S-V). There the policy is $\StMixDPdrFiftyLow\%$
shorter than the best pair at 50 modules and $\StMixLongPdrEightyLow\%$
longer at 80 modules (W/T/L $\StMixWtlEightyLow$), where one vehicle
serves \StSparseEightyLow{} times as many modules as in training.
Against the hour-budget CP-SAT solver in Table~\ref{tab:scale}, the
policy's makespan is longer in the moderate 50-module cell and shorter
in the two heavier cells. At the heavier 50-module cell, the solver's 3600-s incumbents average
$\ScCpNoMatchDef\%$ above the policy, but one of the three training
seeds is longer than the solver there (Table~S-XI). In the 80-module cell,
the policy's mean makespan is $\ScMixDCpEighty\%$ below the incumbents'
(W/T/L $\ScMixWtlCpEighty$, $p=\ScMixPvalCpEighty$). This time-limited
reference weakens at production scale (Section~\ref{subsec:setup}), so
the comparison shows that a one-shot policy produces stronger
incumbents than this reference, and it does not show that the policy
approaches the optimum.

More compute per instance shortens the schedules of both the policy and
the search. Best-of-\SampN{} sampling widens the policy's margin over
the hour-budget solver to $\SampMixDCpEighty\%$ at 80 modules, at
\SampWallFifty--\SampWallEighty~s per instance. With 60~CPU-s of search
after seeding, the GA lies between greedy and best-of-\SampN{} decoding
at these sizes. It leads at 50 modules when its budget is raised to
ten minutes (Tables~S-IV and~S-XXIV; incumbent against compute in
Figs.~S-1 and~S-2). Within its online budget
the trained policy produces the stronger incumbents at these
sizes, whereas longer per-instance search can recover the advantage.

\subsubsection{External Benchmark as a Consistency Check}
\label{subsec:external}

An external benchmark checks whether the scarce-fleet result carries
over to an independently released simulator. We use the benchmark of
Link et al.~\cite{link2026coordination}, whose generator, simulator,
heuristics, and learned models are public. It compares BOLT with their
learned schedulers and heuristics, not with the penalty anchor, so it does
not test the utilization boundary. The comparison does not rank BOLT
either, because
the schedulers were trained under different budgets.
Kolmogorov--Smirnov tests do not reject our port of their generator
against their own code in any of \LkKsCombos{} combinations ($n=100$ per
side, non-rejection only; supplement). Their job shop has no machine flexibility, so
Assumption~\ref{as:pure} holds trivially and the bound's guarantees
apply. We trained our policy (\Nseeds{} runs, per-instance seed means)
on their distribution ($15\times10$, fleet sizes mixed over
$|V|\in\{3,6,\dots,18\}$) and scored it by replaying its decision
sequences in \emph{their released simulator}. Their executor times our
decision sequences and does not hold machines, so our reservation does
not apply there. Comparisons are unpaired, against the three strongest released rows
per fleet size.

\begin{table}[t]
\centering
\caption{External benchmark (Link et al.~\cite{link2026coordination}), $15\times10$, makespan in simulator units (mean~$\pm$~std, $n=\Ninst{}$; per-instance means over \Nseeds{} runs, every seed agreeing in sign). $\Delta$: relative difference to their joint model, best modular pair, and best heuristic (negative favors BOLT; bold: the fleet size where BOLT outperforms all). All differences significant (Mann--Whitney; at $|V|=3$ all $p\le\LkVThreePvalMax$).}
\label{tab:link}
\begin{threeparttable}
\footnotesize
\setlength{\tabcolsep}{4.5pt}
\scriptsize
\begin{tabular}{@{}rcccc@{}}
\toprule
$|V|$ & BOLT & $\Delta$ joint & $\Delta$ modular & $\Delta$ heuristic \\
\midrule
$3$ & \textbf{\LkOursVThree} & \textbf{\LkJointDVThree} &
  \textbf{\LkModDVThree} & \textbf{\LkHeurDVThree} \\
$6$ & \LkOursVSix & \LkJointDVSix & \LkModDVSix & \LkHeurDVSix \\
$9$ & \LkOursVNine & \LkJointDVNine & \LkModDVNine & \LkHeurDVNine \\
$12$ & \LkOursVTwelve & \LkJointDVTwelve & \LkModDVTwelve & \LkHeurDVTwelve \\
$15$ & \LkOursVFifteen & \LkJointDVFifteen & \LkModDVFifteen & \LkHeurDVFifteen \\
$18$ & \LkOursVEighteen & \LkJointDVEighteen & \LkModDVEighteen & \LkHeurDVEighteen \\
\bottomrule
\end{tabular}
\end{threeparttable}
\end{table}

At $|V|=3$, where loaded moves occupy \UtilLinkVThree\% of fleet time
and their own learned models lose to dispatching rules, BOLT outperforms
all three references (bold row of Table~\ref{tab:link}). From $|V|=6$
upward, all three references are shorter than BOLT. The $|V|=6$ setting runs at
\UtilLinkVSix\%, inside the untested interval between \UtilLoseMax\%
and \UtilWinMin\%, and larger fleets run lower. Re-timing our
schedules under their anticipatory vehicle rule changes our makespan
by $\LkWedgeVSix$ at $|V|=6$ and by up to $-\LkWedgeAbsVNine\%$ at
$|V|\ge9$, which explains part of the gap. The remainder is a quality
loss on the pure sequencing left once vehicles stop being scarce
(variant robustness and default travel intensity in the
supplement).

\subsection{Mechanism and Ablation Analysis}\label{subsec:mechanism}

\subsubsection{Fleet Load Determines When Explicit Transport Coupling
Matters}\label{subsec:regret}

Whether explicit transport modeling improves makespan depends on fleet
load. Following the regime mapping
of~\cite{link2026coordination}, we re-evaluate a transport-penalty policy,
the penalty anchor, against BOLT inside the explicit environment
(Fig.~\ref{fig:coupling}(a)). The anchor is trained with full travel-time knowledge but no vehicle
contention, and it was fixed before any premise test (specification
in the supplement). BOLT outperforms the anchor in the
three cells whose loaded utilization is at least $\UtilWinMin\%$,
and the anchor matches or outperforms BOLT in the nine cells
at or below $\UtilLoseMax\%$. The three winning cells are one
in-distribution cell ($|V|=1$, $\bar\tau/\bar p=0.6$, \UtlVOneTSix{}
loaded) and two zero-shot cells at $\bar\tau/\bar p=1.0$. No grid cell
lies between \UtilLoseMax\% and \UtilWinMin\%, so the crossover is
located only within this interval. The anchor in
Fig.~\ref{fig:coupling}(a) is one seed at half budget, and its vehicle
decisions follow the nearest-vehicle-first rule while BOLT's vehicle
head is learned. Fig.~\ref{fig:coupling}(a) therefore compares the two
complete policies and is descriptive.

A pre-specified premise test probes the same claim. It required the explicit-coupling policy to outperform the anchor on both scarce-fleet cells ($|V|\in\{1,2\}$; paired one-sided Wilcoxon, Holm-corrected). The first attempt, at moderate travel ($\bar\tau/\bar p=0.6$) with single runs, passed on the one-vehicle cell ($\UtlVOneTSix$ loaded, $\GaOneImpVOne\%$) but not on the two-vehicle cell ($\UtlVTwoTSix$ loaded), so the test failed. Its one permitted retry, at heavy travel ($\bar\tau/\bar p=1.0$), was fixed after that result and before any retry data, so it is an exploratory second test and not confirmatory evidence. With both arms retrained on travel intensities 0.6 and 1.0 over \Nseeds{} seeds, it passed on both cells, with improvements of $\GbImpVOne\%$ (W/T/L $\GbWtlVOne$) and $\GbImpVTwo\%$ ($\GbWtlVTwo$). Both attempts agree with the utilization boundary, and the failed attempt restricts every claim that follows to high fleet load (protocol and $p$ values in the supplement).

At production scale, BOLT is shorter than the anchor on every instance of the three batches of Table~\ref{tab:scale}, by \AnProdDMin--\AnProdDMax\% (W/T/L \AnProdWtl{} in each cell, $p=\AnProdP$, all arms on one processor thread). Every training seed alone is also shorter (Table~S-XIX). In the two one-vehicle batches with short moves, the factorized policy does not shorten the anchor's makespan, while the merged-head policy does (Table~S-V).

\subsubsection{Different Operating Regimes Activate Different Bound
Components}
\label{subsec:components}

Which relaxation attains the root maximum is readable from the
instances without solving (Fig.~\ref{fig:coupling}(b); Table~S-XXV).
The fleet component attains the maximum in the two most
transport-constrained cells (Remark~\ref{rem:fleet}). At
$|V|=2$ and $\bar\tau/\bar p=1.0$, explicit coupling already improves
the makespan although the chain term still sets the root bound on
nearly every instance. Machine scarcity is the complementary regime.
At 50 modules with light travel, factories of 25, 15, and 9 stations
are machine-bound on every instance (Table~S-XXVI). There the policy
outperforms the 60-CPU-s GA at 25 and 15 stations, which one-shot
decoding does not achieve in the fleet-bound cells. At nine stations the
two are not distinguishable, consistent with every method approaching
the same machine-capacity floor.

\subsubsection{Candidate Action Pricing Drives the Learning Gain}
\label{subsec:credit}

With the same architecture, data, and budget, we add the
bound-difference reward, then the analytic credit in the advantage,
then the action prices as features (Fig.~\ref{fig:curves},
Figs.~S-4 and~S-5, Table~S-VI).

\begin{figure}[!t]
\centering
\figframe{%
\includegraphics[width=\linewidth]{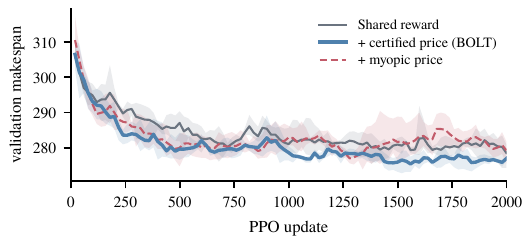}}
\caption{Validation makespan (hours) against PPO updates (line: seed
mean; band: seed range; mean over the nine training cells).}
\label{fig:curves}
\end{figure}

\begin{figure*}[!t]
\centering
\figframewide{\resizebox{0.96\linewidth}{!}{\definecolor{x2blf}{HTML}{A8C6E3}%
\definecolor{x2bln}{HTML}{4F81AD}%
\definecolor{x2amf}{HTML}{F3CF8F}%
\definecolor{x2amn}{HTML}{B8860B}%
\definecolor{x2gyf}{HTML}{C3CBD3}%
\definecolor{x2gyn}{HTML}{8A929B}%
\providecommand{\xFbig}{\fontsize{8.5}{9.5}\selectfont}%
\providecommand{\xFttl}{\fontsize{7.5}{8.5}\selectfont}%
\providecommand{\xFbdy}{\fontsize{7.0}{8.0}\selectfont}%
\begin{tikzpicture}[x=1cm,y=1cm,font=\xFbdy,
  every node/.style={text=black},
  bx/.style   ={draw,line width=0.5pt,rounded corners=1pt,align=center,
                inner sep=3pt,text=black},
  blu/.style  ={bx,fill=x2blf,draw=x2bln,minimum width=3.30cm,
                minimum height=0.44cm},
  wht/.style  ={bx,fill=white,draw=black,minimum height=0.52cm},
  gry/.style  ={bx,fill=x2gyf,draw=x2gyn,minimum height=0.40cm},
  hous/.style ={draw=black,line width=0.6pt,rounded corners=1.5pt,
                fill=white},
  hlb/.style  ={inner sep=0pt,anchor=north,text width=4.10cm,align=center,
                text=black},
  arr/.style  ={-{Latex[length=2.6pt,width=2.2pt]},line width=0.45pt,
                draw=black},
  aarr/.style ={{Latex[length=2.6pt,width=2.2pt]}-%
                {Latex[length=2.6pt,width=2.2pt]},line width=0.5pt,
                draw=x2amn},
  narr/.style ={{Latex[length=2.4pt,width=2.0pt]}-%
                {Latex[length=2.4pt,width=2.0pt]},line width=0.4pt,
                draw=x2gyn,dash pattern=on 1.6pt off 1.4pt},
  sep/.style  ={line width=0.4pt,draw=x2gyn},
  lb/.style   ={inner sep=1pt,text=black}]

\draw[draw=x2gyn,line width=0.5pt,rounded corners=1.5pt,fill=white]
  (0,1.86) rectangle (6.10,5.90);
\draw[draw=x2gyn,line width=0.5pt,rounded corners=1.5pt,fill=white]
  (7.50,1.86) rectangle (18.40,5.90);

\node[lb,font=\xFttl\bfseries] at (3.05,5.66) {Merged-head BOLT};
\node[lb,font=\xFttl\bfseries] at (12.95,5.66)
  {Factorized BOLT (deployed)};

\node[gry,minimum width=5.00cm] (lbc) at (3.05,5.14)
  {Shop-state broadcast};

\draw[hous] (0.80,2.88) rectangle (5.30,4.70);
\node[hlb] at (3.05,4.62)
  {\textbf{Single controller}\\(machine and vehicle events)};
\node[blu] (lenc) at (3.05,3.79) {Encoder};
\node[blu] (lhd)  at (3.05,3.21) {Merged decision head};
\draw[arr] (lenc.south) -- (lhd.north);

\node[wht,minimum width=4.50cm] (lact) at (3.05,2.36)
  {Machine or vehicle event action};

\draw[arr] (3.05,4.94) -- (3.05,4.70);
\draw[arr] (3.05,2.88) -- (lact.north);

\node[gry,minimum width=9.80cm] (rbc) at (12.95,5.14)
  {Shop-state broadcast};

\draw[hous] (8.05,2.88) rectangle (12.55,4.70);
\node[hlb] at (10.30,4.62)
  {\textbf{Station controller}\\(manufacturing execution system)};
\node[blu] (senc) at (10.30,3.79) {Encoder};
\node[blu] (shd)  at (10.30,3.21) {Machine head};
\draw[arr] (senc.south) -- (shd.north);

\draw[hous] (13.35,2.88) rectangle (17.85,4.70);
\node[hlb] at (15.60,4.62)
  {\textbf{Fleet controller}\\(automated guided vehicle fleet)};
\node[blu] (fenc) at (15.60,3.79) {Encoder};
\node[blu] (fhd)  at (15.60,3.21) {Vehicle head};
\draw[arr] (fenc.south) -- (fhd.north);

\draw[arr] (10.30,4.94) -- (10.30,4.70);
\draw[arr] (15.60,4.94) -- (15.60,4.70);

\draw[narr] (12.55,3.50) -- (13.35,3.50);
\fill[white] (12.95,3.50) circle (0.115);
\draw[draw=x2gyn,line width=0.7pt]
  (12.865,3.415) -- (13.035,3.585) (12.865,3.585) -- (13.035,3.415);

\node[wht,minimum width=4.00cm] (sact) at (10.30,2.36)
  {Machine event action};
\node[wht,minimum width=4.00cm] (fact) at (15.60,2.36)
  {Vehicle event action};
\draw[arr] (10.30,2.88) -- (sact.north);
\draw[arr] (15.60,2.88) -- (fact.north);

\end{tikzpicture}}}
\caption{Merged-head (left) and factorized (right) realizations of
BOLT.}
\label{fig:deploy}
\end{figure*}
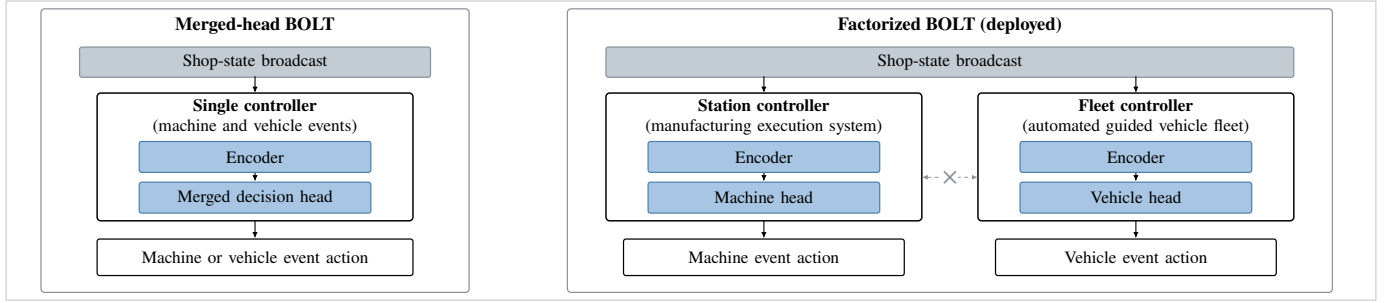

\emph{The action-price input produces the training improvement.}
Under the pre-specified rule, the bound-guided arm
had to outperform the credit-only arm significantly on all four
cells. It succeeded on every cell and seed, with pooled
gains of $\GuImpVOneTThree\%$,
$\GuImpVTwoTThree\%$, $\GuImpVOneTSix\%$, and $\GuImpVTwoTSix\%$ (all
Holm-corrected $p\le\GuPvalMax$; seed-level tests in the
supplement).

\emph{Removing the fleet-capacity term lengthens the learned schedules
where the fleet binds.} An ablation arm trains BOLT without $\Bveh$ in
the bound and changes nothing else, so its reward, action prices,
and credit read only the chain and machine terms. The fleet term sets
the root bound in the two one-vehicle cells at $\bar\tau/\bar p=0.6$
and at the held-out 1.0 (Fig.~\ref{fig:coupling}(b)). In these cells
the ablation's makespan is $\AblVehDVOneTSix\%$ and
$\AblVehDVOneTTen\%$ longer, and both differences pass the
pre-specified one-sided Wilcoxon test (Holm-adjusted
$p<\AblVehPHolmMax$; W/T/L \AblVehWtlVOneTSix{} and
\AblVehWtlVOneTTen{}). All three seeds agree at 0.6, and two of three
agree at 1.0. The difference falls to $\AblVehDVOneTThree\%$ at
$\bar\tau/\bar p=0.3$ and stays within $\AblVehRestMax\%$ in the nine
cells with more vehicles or lighter travel (Table~S-VIII). On the three production batches of
Table~\ref{tab:scale}, it sets the root bound on every instance. Retraining the size-mixture policy without it lengthens
the schedules there by $\AblVehMixDFiftyTTen$--$\AblVehMixDEighty\%$
(pre-specified test, W/T/L \AblVehMixWtl{} in each cell,
Holm-adjusted $p=\AblVehMixPHolm$), against at most
$\AblVehMixSmallMax\%$ at the training sizes (Table~S-IX). Without
$\Bveh$, the root bound of these batches covers only
\RootNoVehCovMin--\RootNoVehCovMax\% of the best-known makespan, against
\RootVehCovMin--\RootVehCovMax\% with it. The ablation therefore shows
that a dense signal with fleet workload outperforms one without it, and
it does not compare the certified bound with a tighter uncertified
estimate.

\emph{Analytic credit alone gives no gain.} The pre-specified credit
test finds all four scarce-fleet cells equivalent to shared-reward
training within $\pm2\%$ (Holm-adjusted $p\le\GtwoTostPvalMax$), as the
control-variate lemma predicts (details in the supplement).

\emph{A per-candidate price carries most of the gain}
(Table~S-VII). A dense myopic duration price of the same mean size is
longer than the certified price in one of the four pre-specified cells,
by $\MyoSigD\%$ (one-sided $p=\MyoPOneSided$, Holm-corrected), and not
distinguishable in the other three. The evidence therefore supports
candidate-specific pricing and does not show that admissibility itself
improves schedules. A COMA-style learned counterfactual
critic~\cite{foerster2018coma} is longer than the certified arm in all
\ComaCells{} evaluated cells, by $\ComaDMin$--$\ComaDMax\%$, and this
result is descriptive because no test was fixed for that arm. The certified
price is zero on $\GuidePriceZeroPct\%$ of candidates, and the policy
ranks these by its learned features.

\subsubsection{Price Guidance under Processing-Time
Heterogeneity}\label{subsec:hetero}

The chain recursion substitutes the eligible minimum for every
uncommitted operation, so strongly machine-dependent processing times
could loosen the pre-commitment bound and weaken price guidance. We test
this effect zero-shot by redrawing processing times around each
operation's mean while the allowed spread across machines (the
dispersion level $R$) rises to $\HetEnvMax$. This level gives a realized
within-operation ratio of \HetRealizedHundredMean{} on average
(protocol and tables in the supplement). The certified floor stays below
the best-known makespan on all \HetNTransformed{} transformed instances.
On the two training cells, the price channel's advantage over
shared-reward training rises from $\HetAdvBaseVOne\%$ to
$\HetAdvPeakVOne\%$ on the scarce-fleet cell ($\HetAdvHundredVOne\%$ at
the stress level) and from $\HetAdvBaseVTwo\%$ to
$\HetAdvHundredVTwo\%$ on the moderate cell. Every level is significant
without multiplicity adjustment (largest $p=\HetTrainPvalMax$). On the
held-out heavy-travel cell, the advantage is negative at the baseline dispersion,
not distinguishable from zero at $R=2$, and positive from $R=5$ upward.
Commitment replaces the optimistic minimum with the machine-specific
time in the successor bound, so the share of candidates with a positive
price rises with dispersion (Corollary~\ref{cor:guide}). Against fixed
rules, the policy's zero-shot margin over the best pair persists at
heavy travel through $R=20$ and is overtaken at moderate travel from
realized ratios near $2.5$.

\subsection{Deployment and Industrial Decision Support}\label{subsec:deploy}

\subsubsection{Factorized Control and Its Makespan Cost}\label{subsec:parity}

Factorization costs makespan against the merged-head control on the
production batches but not at the training sizes. Across the
\TostCells{} training-size and transfer cells, the factorized policy is
cell-wise TOST-equivalent to the merged-head control within $\pm2\%$ in
\TostEq{} cells, with largest absolute difference $\TostEqMaxAbs$
(Table~S-XXVII). It is better in the
$15\times25$ zero-shot cell at $\bar\tau/\bar p=1.0$
($\TostBetterDiff$). On the production batches, both arms are retrained
on the same size mixture, and the merged head is shorter in all five
greedy-decoded cells (Table~S-XVIII). Its lead lies above the margin in
the two one-vehicle cells ($\MixSaRelFiftyLow\%$ and
$\MixSaRelEightyLow\%$). Equivalence is not established at 50 modules
in Table~\ref{tab:scale} ($\MixSaRelFiftyTSix\%$ and
$\MixSaRelFiftyTTen\%$ at $n=20$), and the arms are equivalent within
the margin at 80 modules under both decodings.
Because both realizations consume the same bound signals and backbone
capacity, the comparison isolates the cost of factorization, which is
therefore not the source of the scheduling gain. When the nearest-vehicle-first rule replaces every learned vehicle
decision, both realizations still outperform the best fixed dispatching-rule
pair on all five production batches, by \FbVsPdrMin--\FbVsPdrMax\%
(Table~S-XX). With the shortest-travel-time rule instead, both lose
\FbSttMin--\FbSttMax\% against their learned decisions, so the fallback
rule matters. BOLT keeps the factorized form for the
separate decision layers of Section~\ref{sec:realization}.

\subsubsection{Millisecond Dispatch and Second-Scale Disruption
Recovery}\label{subsec:latency}

The deployed factorized policy (Fig.~\ref{fig:deploy}) runs as a
\ParamsDeploy-parameter replica on each controller, with no central solver. On one processor core, the deployed
size-mixture policy takes \LatEvtCpuTen~ms per event at 10 modules and
\LatEvtCpuEighty~ms at 80 modules, including the shop-state update, and
it builds a whole schedule in \LatSchedCpuTen~s at 10 modules and \LatSchedCpuEighty~s at 80 modules.
A workstation GPU takes \LatEvtGpuEighty~ms per event at 80 modules,
and the process uses at most \PeakRssCpuMax~MiB (Table~S-XXXIII).
Deciding per event also lets the policy re-plan after a breakdown. When the machine with the most remaining work fails at 30\% of the plan on the two heavy production batches, the policy re-plans the rest in \DisWallFifty{}--\DisWallEighty~s and outperforms right-shift repair and re-dispatching (Table~S-XXI). Warm-started CP-SAT returns the right-shifted plan unchanged at search budgets of 1 and 10~s, which take \DisrCpWallShortMin--\DisrCpWallShortMax~s of wall time. At 60~s it shortens that plan by at most \DisrCpImpMax\%, so the policy's schedule stays \DisrCpVsPolMin--\DisrCpVsPolMax\% shorter (Table~S-XXII). The GA overtakes the policy after \DisrGaWallMin--\DisrGaWallMax~s of wall time, and the policy's schedule is then \DisrGaVsPolMin--\DisrGaVsPolMax\% longer. A plant that can wait about two minutes therefore gets a shorter repaired schedule from the GA, while the policy provides a repaired plan within seconds and a decision at every event. Under duration noise of up to $\pm30\%$ (one training seed, two cells), re-planning did not significantly shorten right-shift execution (largest difference $\NoisePolReplanVsRsMax\%$), and the re-plan stays $\NoisePolVsPdrMin$--$\NoisePolVsPdrMax\%$ shorter than rule re-dispatch (Table~S-XXIII).

\subsubsection{The Same Bound Certifies Learned Schedules and
Tightens the CP-SAT Bound}\label{subsec:cert}

The root
bound $B(s_0)$ of Corollary~\ref{cor:root}, computed from static
instance data before any dispatch, certifies any schedule of makespan
$C$ as within $(C-B(s_0))/B(s_0)$ of optimal. This certificate is valid
on every instance, and its tightness depends on which term sets the
bound. The solver closes \TightN{} of the \CertNadm{} grid and transfer
instances, all of them chain-bound, and on these the root bound averages
\TightMeanPct\% of the proven optimum (worst \TightWorstPct\%). Over all \CertNadm{} instances it
averages \TightAllMean\% of the best-known makespan, and
\TightFleetMean\% on the \TightFleetN{} instances where the fleet term
sets it (both lower limits on tightness; Table~S-XVI, Fig.~S-6). The measured
certificate is therefore loosest where the fleet binds. On the production batches of Table~\ref{tab:scale}, the
certificate places the policy's schedules within
\CertFiftyLow--\CertFiftyHigh\% of optimal on average per cell. This
guarantee is loose, but it stays valid where the time-limited solver no
longer closes the problem.

Added to CP-SAT as a
redundant cumulative constraint, the fleet term raises the solver's
proven lower bound above the better of its own bound and $B(s_0)$.
The excess is \CpAddRootTSixMin--\CpAddRootTSixMax\% at
$\bar\tau/\bar p=0.6$ and \CpAddRootTTenMin--\CpAddRootTTenMax\% at
$1.0$ ($|V|=1$, 12 instances per cell, Table~S-X). With the constraint, the incumbent is
\CpIncImpTSix\% shorter at $0.6$ (W/T/L \CpIncWtlTSix) and unchanged at
$1.0$ (\CpIncWtlTTen), and no variant closes an instance in \cpbudget~s.

\section{Conclusion}\label{sec:conclusion}

When a scarce fleet moves jobs between stations, the
transfer delay depends on the schedule. BOLT uses one certified
multi-bottleneck bound whose chain, machine-capacity, and
fleet-capacity components never exceed the makespan still attainable.
The bound equals the makespan at termination, so its temporal
difference is an objective-faithful reward, and its candidate-action
increases are proven rises of the makespan floor that the policy reads
as prices. Ablations on four pre-specified cells attribute the training
gain to these prices, and removing the fleet-capacity term lengthens the
learned schedules by $\AblVehMixDFiftyTTen$--$\AblVehMixDEighty\%$ on
three production batches. On the industry-motivated PPVC benchmark,
explicit transport coupling reduces makespan in every tested 10-module
cell where loaded moves occupy at least $\UtilWinMin\%$ of fleet time,
and a transport-penalty policy is adequate at or below $\UtilLoseMax\%$. Across five
production batches, BOLT outperforms the best dispatching-rule pair in
four and the penalty policy in the three with long moves. With three vehicles
at 80 modules, it also produces shorter schedules than an hour-budget solver that ends
far from its own bound. In the two one-vehicle batches, the factorized
policy does not shorten the penalty policy's makespan, and a merged-head
policy does. The root bound gives a valid but loose certificate of a
completed schedule's suboptimality gap without a solver, and as a
redundant constraint it raises the lower bound that CP-SAT proves.
Stochastic arrivals, vehicle routing and congestion, heterogeneous
fleets, finite buffers, and online dispatch without destination-machine
reservations remain for future work, and other objectives would need
their own bound. Within the tested scope, BOLT dispatches in
milliseconds per event, re-plans a disruption in seconds, and gives
station and vehicle controllers separate decision layers.

\bibliographystyle{IEEEtran}
\bibliography{refs}
\clearpage


\section*{Supplementary material}

# Supplementary Material for "BOLT: Multi-Bottleneck Bound-Guided Learning for Flexible Job-Shop Scheduling with Time Lags and Resource-Coupled Transport"

Ziheng Zhang, Wei Zhang, Yunyi Zhao, Zuming Liu, Wei Zhou, and Boon Ping Gan

This supplement provides the complete proofs and technical remarks for Sections IV and V of the main manuscript. It also provides the baseline specifications, additional results, and the training and generator configurations cited in the main text. Key statements retain their main-text numbering and are restated in equivalent form. The notation follows Sections III–V of the main text.

## I. Positioning Table

Table S-I compares the closest works on the main text's contribution dimensions. These include the scheduling setting, whether a scarce transport resource is modeled explicitly, what lower-bound structure (if any) the method carries, and which of the bound's uses each work supports (reward construction, certified candidate-action guidance, a schedule-quality certificate, solver strengthening), and the controller structure. The reference numbers refer to this document's own reference list.

## II. Light-Travel Column of the Regime Grid

Table S-II completes the regime grid of the main text with the light-travel column.

## III. Fleet Utilization by Regime Cell

In absolute terms, a processing step averages 4.4 h. Because $\bar{\tau}$ uses the shortest travel of each mandatory move, actual moves between different cells run slightly longer, about 30 min at $\bar{\tau}/\bar{p} = 0.1$ and 2.8 and 4.7 h at 0.6 and 1.0.

The travel intensity $\bar{\tau}/\bar{p}$ is convenient for generating instances but hard to compare across plants because it describes one move, not the load on the fleet. A more useful measure is the share of the fleet's time claimed by loaded moves. At the root state every vehicle is free, so the vehicle component of the bound reduces to $B_{\text{veh}}(s_0) = W_{\text{tr}}/|V|$, with $W_{\text{tr}}$ the total minimal loaded work the instance still forces, so any schedule of makespan $C$ must keep the fleet loaded for at least

$$u = \frac{W_{\text{tr}}}{|V|\,C} = \frac{B_{\text{veh}}(s_0)}{C} \tag{1}$$

of the horizon. In Table S-III, $C$ is the best makespan found on each instance by any reported method. Empty legs are left out of $W_{\text{tr}}$, so every entry understates the true occupancy. Because $C$ is an upper bound on the optimum, $u$ also understates what the optimal schedule must reach. Both approximations run the same way, which is why $u$ is reported as a floor.

The three cells in which explicit transport modeling reduces makespan (Fig. 3(a) of the main text) are exactly the three at $u \geq 66\%$, and the nine in which it does not are all at $u \leq 56\%$. No cell lies between these two values, so the crossover is located only within this interval. Carrying more modules per vehicle can reach the same load. At $|V| = 2$ and $\bar{\tau}/\bar{p} = 0.6$, $u$ rises from 53.1% at 10 modules to 62.4% at 15 and 76% at 50, and the production-batch cells of Table III of the main text all lie above 66% (76/80/74%).

On the ported external benchmark of Link et al., with $C$ the best released or replayed mean per fleet size, the estimate gives $u = 66\%$ at $|V| = 3$ and 62/51/40/32/27% at $|V| = 6/9/12/15/18$. That benchmark compares BOLT with learned models and heuristics, not with a penalty policy, so it does not test the crossover, and its $|V| = 6$ setting lies inside the untested interval between 56% and 66%. Their generator's own travel intensity, measured on its 100 test instances, is $\bar{\tau}/\bar{p} = 1.09$ (range 0.93–1.30).

## IV. Scale-Up Compute Sweep

Fig. S-1 plots the genetic algorithm's quality against its *actual* per-instance compute, which includes dispatching-rule seeding as well as search, relative to the 10-module-trained policy. No budget leaves the GA worse than the 10-module policy once it can emit a schedule (the mixture level shown is its 50-module value). Larger instances require more computation to produce the first schedule. Nominal search budgets of 0.25 to 15 CPU-s collapse onto a single point at 30 and 50 modules, because the seeding dominates and one generation already exceeds the budget. On the production batches the first schedule costs 31–84 CPU-s before the search budget starts. Against the solver reference the ranking reverses with size: at the training size the 300-s CP-SAT incumbent is shorter than the 60-CPU-s GA in seven of the nine reported cells (the GA leads only in the two most transport-bound, $|V| = 1$ at $\bar{\tau}/\bar{p} \geq 0.6$), while at 50 and 80 modules the one-minute GA is shorter than the 3600-s incumbent in every cell. Both are seeded with the same dispatching schedules, so the inversion is search progress alone.

At production scale, CP-SAT returns a schedule within its budget only with the dispatching warm start. A cold-start

TABLE S-I
POSITIONING OF THE CLOSEST LEARNING-BASED WORKS ON JOB-SHOP SCHEDULING WITH AND WITHOUT TRANSPORT, BY BOUND STRUCTURE AND BOUND USE. "SCARCE TRANSPORT": JOBS MOVE VIA AN EXPLICITLY MODELED, CAPACITY-LIMITED VEHICLE RESOURCE. "CERTIFIED GUIDANCE": PER-CANDIDATE ACTION SIGNALS WITH A PROVEN GUARANTEE. ABBREVIATIONS: AGV, AUTOMATED GUIDED VEHICLE; LB, LOWER BOUND.

| Work | Scheduling setting | Scarce transport | Bound structure | Bound as reward | Certified guidance | Schedule certificate | Solver strengthening | Controller factorization |
|---|---|---|---|---|---|---|---|---|
| Song et al. [1] | FJSP; makespan | No | Chain (optimistic completion estimate) | Yes | No | No | No | No (single policy) |
| DANIEL [2] | FJSP; makespan | No | Chain (minimum eligible-machine processing) | Yes | No | No | No | No (single policy) |
| Moon et al. [3] | FJSP with transport; makespan | Yes | Chain-style recursive completion estimate (no fleet-capacity term; not formally certified) | Yes | No | No | No | No (single policy) |
| Li et al. [4] | Dynamic FJSP with AGVs; total tardiness | Yes | None (hand-designed dense reward) | No | No | No | No | Yes (task, machine, AGV) |
| Zhang et al. [5] | Dynamic FJSP; transport as job-transfer delay; weighted tardiness | No (delay, not a seizable vehicle) | None (hand-designed shaping) | No | No | No | No | Yes (job, machine) |
| Wang et al. [6] | Dynamic FJSP with transport robots; makespan | Yes | None (heuristic imitation) | No | No | No | No | Yes (hierarchical) |
| Yang et al. [7] | AGV-assisted FJSP; makespan | Yes | None (hand-designed reward) | No | No | No | No | Yes (resource-class heads) |
| Link et al. [8] | Deterministic job shop with transport; makespan | Yes | LB as static normalizer only | No (sparse terminal) | No | No | No | Yes (job, AGV) |
| Lv et al. [9] | FJSP rescheduling under breakdowns; no transport | No | None (deviation-based difference rewards) | No | No | No | No | Yes (machine agents) |
| **BOLT** | FJSP with time lags and resource-coupled transport; makespan | **Yes** | **Chain + machine capacity + fleet capacity, certified** | **Yes** | **Yes** | **Yes** | **Yes** | **Optional (factorized or merged head)** |

TABLE S-II
LIGHT-TRAVEL COLUMN OF THE REGIME GRID ($\bar{\tau}/\bar{p} = 0.1$) AND THE FOUR ZERO-SHOT TRANSFER CELLS (MAKESPAN IN HOURS, BEST FIXED PRIORITY-DISPATCHING-RULE (PDR) PAIR, GENETIC ALGORITHM (GA) AT 60 CPU-S, MERGED-HEAD CONTROL, WARM-STARTED CP-SAT AT 300 S, FACTORIZED BOLT, GAP% VS THE SOLVER'S INCUMBENT, WIN/TIE/LOSS (W/T/L) VS BEST PDR; TRANSFER CELLS HAVE NO GA RUN; THE TWO 15×25 ROWS ARE AT $|V| = 2$). IN THE TRANSFER ROWS, THE LEARNED COLUMNS ARE THE 10-MODULE-TRAINED POLICIES. THE SIZE-MIXTURE POLICY APPEARS IN FIG. S-3. BOLD: SHORTEST VALUE IN THE ROW. WITH LIGHT TRAVEL, THE GA OUTPERFORMS BOTH LEARNED POLICIES ON EVERY INSTANCE.

| $\|V\|$ | $\bar{\tau}/\bar{p}$ | Best PDR | GA | Merged | CP-SAT | BOLT | Gap% | W/T/L |
|---|---|---|---|---|---|---|---|---|
| 1 | 0.1 | 229.9±9.1 | 218.1±8.1 | 225.5±8.2 | **211.2±9.0** | 225.9±8.2 | +7.0% | 79/0/21 |
| 2 | 0.1 | 227.0±9.0 | 215.4±8.4 | 222.5±8.5 | **209.7±9.2** | 222.2±8.6 | +5.9% | 82/0/18 |
| 3 | 0.1 | 226.9±9.1 | 215.3±8.4 | 222.1±8.6 | **209.7±9.1** | 221.8±8.4 | +5.8% | 87/0/13 |
| 4 | 0.6 | 281.6±9.6 | – | 277.8±8.1 | **238.8±8.9** | 277.6±8.4 | +16.2% | 70/0/30 |
| 4 | 1.0 | 334.8±10.2 | – | 334.3±10.5 | **277.3±9.5** | 334.4±9.6 | +20.6% | 53/0/47 |
| 15×25, 0.6 | | 376.7±12.1 | – | 394.2±10.4 | **352.1±10.9** | 395.4±11.3 | +12.3% | 14/0/86 |
| 15×25, 1.0 | | 562.5±23.0 | – | 556.0±18.9 | **518.3±14.4** | 539.8±14.7 | +4.1% | 80/0/20 |

TABLE S-III
MINIMUM LOADED FLEET UTILIZATION $u$ OF (1) PER CELL, $C$ THE BEST MAKESPAN FOUND ON THAT INSTANCE BY ANY REPORTED METHOD. PRIMARY-GRID AND TRANSFER ROWS: MEAN OVER 100 INSTANCES; THE 50×25 AND 80×25 ROWS: $n = 20$. THE TWO 15×25 AND 50×25 ROWS ARE AT $|V| = 2$, THE 80×25 ROW AT $|V| = 3$.

| | $\bar{\tau}/\bar{p}$ | | | |
|---|---|---|---|---|
| $\|V\|$ | 0.1 | 0.3 | 0.6 | 1.0 |
| 1 | 23.0% | 56.3% | 74.3% | 81.7% |
| 2 | 11.5% | 32.2% | 53.1% | 65.7% |
| 3 | 7.7% | 21.9% | 39.5% | 53.3% |
| 4 | – | – | 30.4% | 43.9% |
| 15×25 | – | – | 62.4% | 72.5% |
| 50×25 | – | – | 76% | 80% |
| 80×25 | – | – | – | 74% |

probe solved the same strengthened model without any initial schedule, on the same four workers, cores, and settings, for one cell per size. With no warm-start makespan, the model's horizon stays at its generic bound. At the heavier 50-module cell ($|V| = 2$, $\bar{\tau}/\bar{p} = 1.0$), the cold solver found no feasible schedule within 3600 s on 3 of 3 instances. The hour-budget comparison with the policy in the main text therefore gives the solver its best-known start. At the training size ($|V| = 1$, $\bar{\tau}/\bar{p} = 0.6$, $n = 5$), the start changes little. There the cold solver found its first schedule after 12.5–16.8 s, and its incumbents at 300 s average 389.4 h against 391.0 h for the warm-started runs on the same instances. Its proven lower bounds equal the warm-started ones on 5 of 5 instances, so the lower bound does not depend on the start.

## V. SAMPLED DECODING (BEST-OF-$N$)

Table S-IV reports the matched-budget control of Sections VI-B1 and VI-B2 of the main text. For each instance, candidate 0 is the greedy rollout and $N-1 = 63$ further schedules are sampled from the same policy, with the framework's random seed fixed per pass. The result is the best schedule that passes the independent feasibility validator. Values are per-instance 3-seed means. The quoted wall times come from a seed-301 rerun on an otherwise idle GPU with four pinned cores, which reproduces every sampled schedule. They are not hardware-equalized against the GA's CPU-seconds.

## VI. ONE-VEHICLE PRODUCTION BATCHES WITH SHORT MOVES

With one vehicle for a whole batch and short moves ($\bar{\tau}/\bar{p} = 0.3$), the deployed policy's lead over the dispatching rules ends at 80 modules (Table S-V). The deployed policy is

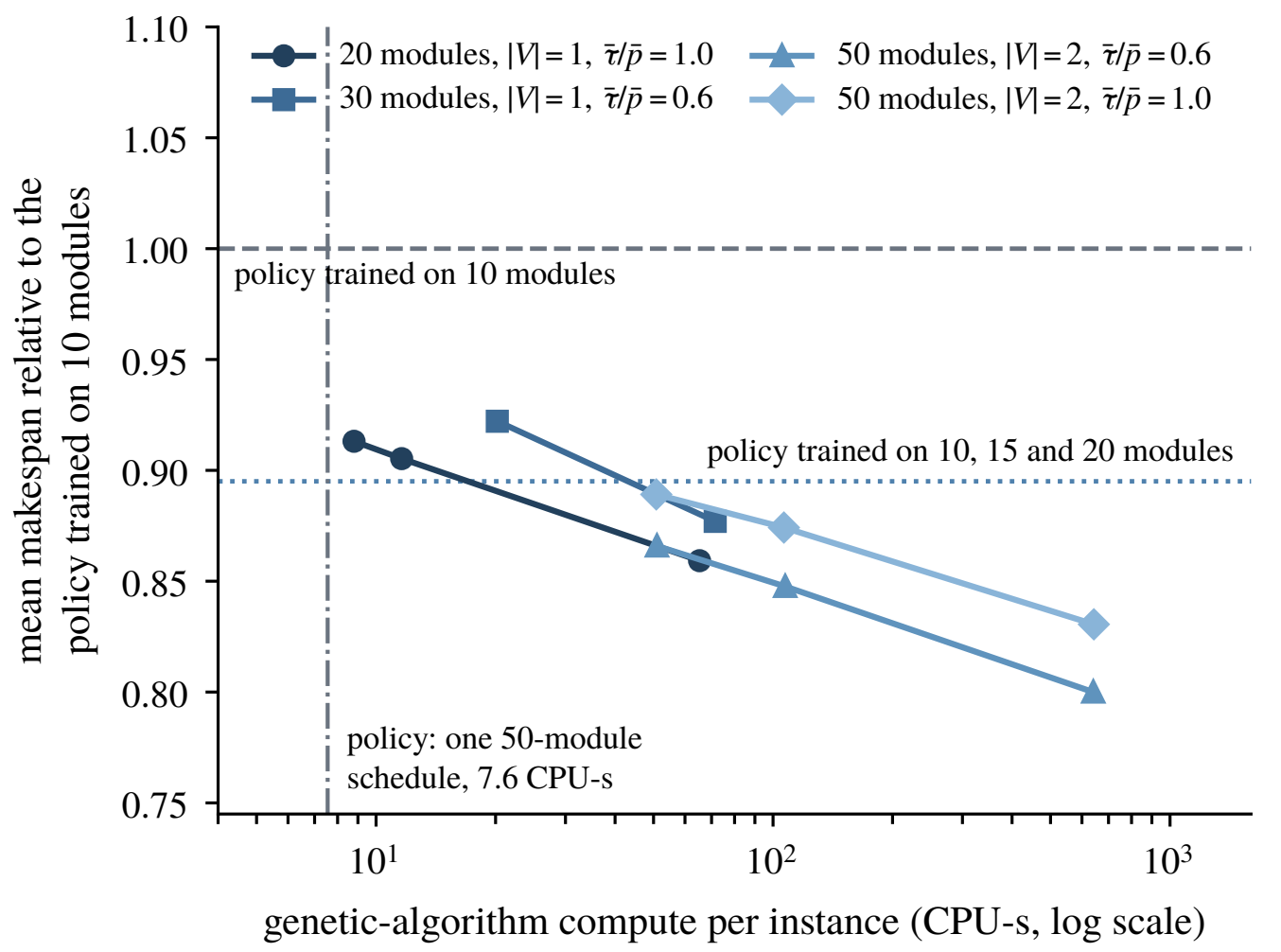


Fig. S-1. GA mean makespan relative to the 10-module-trained policy on the same cell (gray dashed: parity; blue dotted: the size-mixture policy's mean level over the two 50-module cells) against total GA compute per instance (CPU-s, seeding plus search, log scale), $n$ = 30 per cell at 20 and 30 modules and $n$ = 20 at 50. The vertical line marks the compute of one 50-module schedule by the policy itself (7.6 CPU-s, one thread on one dedicated core, $n$ = 3).

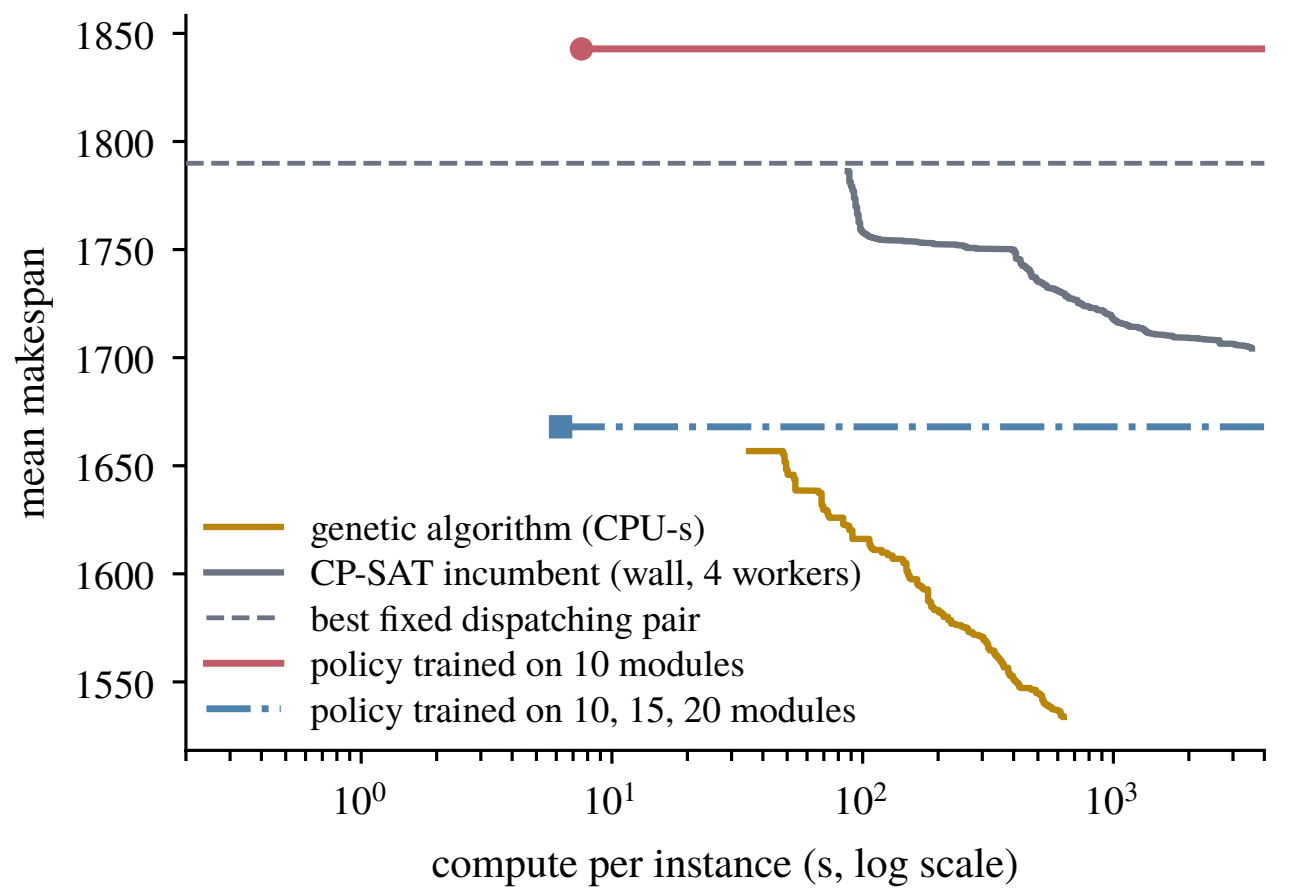


Fig. S-2. Mean incumbent makespan (hours) against compute per instance at the hardest 50-module cell ($|V|$=2, $\bar{\tau}/\bar{p}$=1.0; $n$ = 20 shared instances): per-instance search (genetic algorithm in CPU-s including its seeding; CP-SAT in wall-s on four workers, from its dispatching warm start) against the fixed outputs of the two policies of Fig. S-3, drawn from their end-to-end schedule times (single-core CPU for the 10-module policy, batched GPU for the size mixture; both measured on a shared machine, not certified latencies). Search produces better schedules once it runs; the policies respond first, and only the size-mixture policy outperforms the best fixed pair at this size.

3.6% shorter than the best fixed pair at 50 modules (W/T/L 16/0/4, $p$ = 2.3×10$^{-3}$) and 1.5% longer at 80 modules (5/0/15, $p$ = 2.7×10$^{-3}$). Both cells load the vehicle above 66%, with a minimum loaded utilization of 76%, but their fleets are sparser per module than any trained setting. One vehicle per 50 and per 80 modules gives 0.2 and 0.125 vehicles per 10 modules, while the sparsest trained setting has 0.5 (one vehicle for 20 modules). These fleets are therefore 2.5 and 4 times sparser than in training.

Against the 10-module penalty anchor, the deployed policy

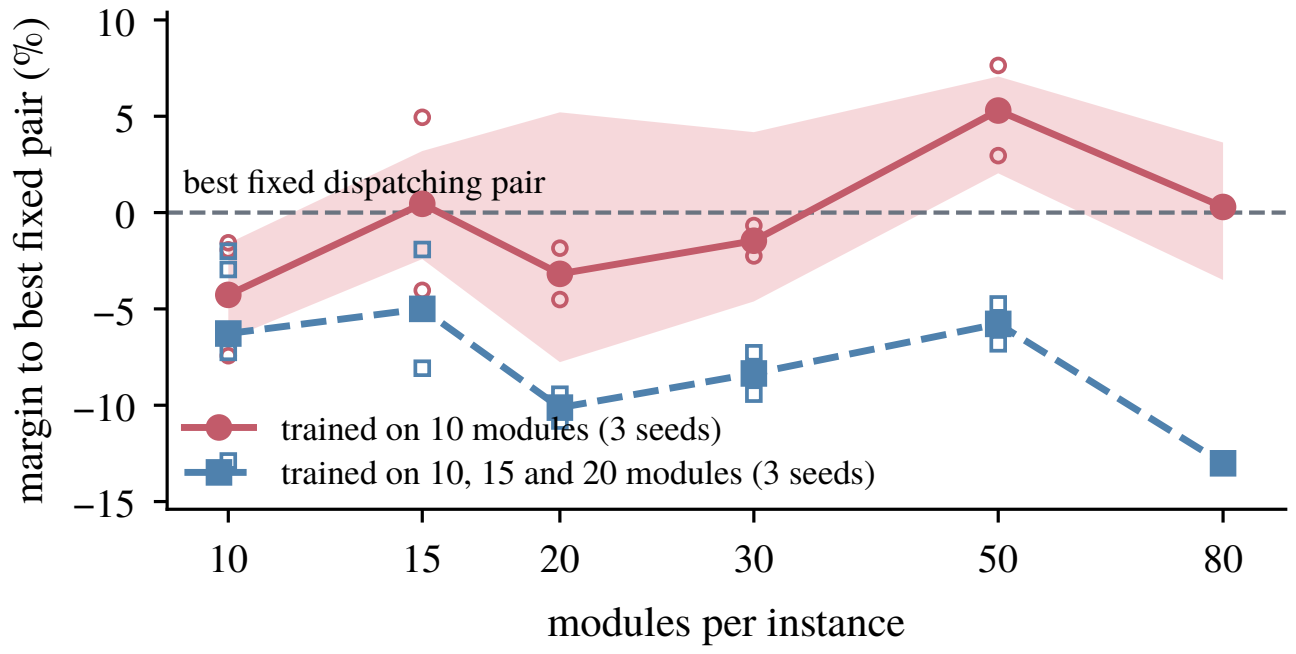


Fig. S-3. Zero-shot size boundary: makespan margin to the best fixed dispatching pair (negative: policy better) from the training size of 10 modules to the 80-module production batch. Lines: per-size means; markers: individual cells ($n$ = 100 at 10–15 modules, $n$ = 30 at 20 and 30 modules, $n$ = 20 at 50 and 80); bands: seed range (3 seeds). Fleet and travel settings vary with size, so compare within a series, not across sizes. Every size above 10 modules is zero-shot for the 10-module policy; the size-mixture policy is in-support through 20 modules and zero-shot beyond.

TABLE S-IV
BEST-OF-64 SAMPLED DECODING VS. ONE-SHOT GREEDY AND THE 60-CPU-S GA (MAKESPAN IN HOURS; $n$ = 100 PER 10 × 25 CELL, $n$ = 20 AT 50 AND 80 MODULES; THE PRODUCTION-BATCH ROWS USE THE SIZE-MIXTURE POLICY). W/T/L: SAMPLED VS. GA, PAIRED WILCOXON. BOLD: SHORTEST IN THE ROW.

| Cell | Greedy | Best-of-64 | GA 60 s | Δ vs GA | W/T/L |
|---|---|---|---|---|---|
| 10×25, 1, 0.6 | 422.1 | 408.1 | **392.3** | +4.0% | 0/0/100 |
| 10×25, 1, 1.0 | 672.1 | 619.6 | **591.9** | +4.7% | 0/0/100 |
| 10×25, 2, 0.6 | 304.4 | 294.4 | **289.5** | +1.7% | 3/0/97 |
| 10×25, 2, 1.0 | 401.4 | 384.4 | **374.8** | +2.6% | 0/0/100 |
| 50×25, 2, 0.6 | 1061.1 | **995.1** | 1016.3 | −2.1% | 18/0/2 |
| 50×25, 2, 1.0 | 1668.0 | **1568.4** | 1611.0 | −2.6% | 20/0/0 |
| 80×25, 3, 1.0 | 1890.8 | **1803.9** | 1852.2 | −2.6% | 20/0/0 |

does not shorten makespan in either cell. It is 0.7% longer than the anchor at 50 modules (W/T/L 6/0/14, $p$ = 0.14) and 1.3% longer at 80 modules (5/0/15, $p$ = 0.014). The merged-head mixture reads the same bound signals through one decision head, and it is the shortest online arm in both cells. It is 3.5% and 3.0% shorter than the anchor (19/0/1 in each cell) and 7.6% and 2.8% shorter than the best pair (20/0/0 in each cell). Its lead over the deployed policy is +4.33% and +4.44%, beyond the ±2% equivalence margin (Table S-XVIII). The 10-module price-channel policy, trained at the anchor's size, is 11.4% and 9.8% longer than the anchor (0/0/20 in each cell). At these batch sizes, explicit coupling therefore shortens makespan against the anchor only in its merged-head form. The 60-CPU-s GA searches each instance and is shorter than every online arm in both cells.

## VII. GATE STATISTICS AND THE CREDIT-AND-GUIDANCE LADDER

Fig. S-5 shows the per-cell improvement of the bound-guided arm over the credit-only arm. The gate tests quoted in the main text use ratios of cell means, which differ from the per-instance mean differences plotted there only in the second decimal. Table S-VI reports the full per-cell ladder, and Fig. S-4 the training trajectories of its six arms (the main text's Fig. 4 keeps the three arms that carry the mechanism claim).

TABLE S-V

MAKESPAN IN HOURS ON THE ONE-VEHICLE PRODUCTION BATCHES WITH SHORT MOVES (MEAN±STD, $n = 20$ PER CELL). BOLD: SHORTEST ONLINE ARM; THE GA IS A PER-INSTANCE SEARCH ARM AND IS EXCLUDED FROM THE BOLD RULE.

| Cell | $u$ | Best PDR | GA 60 s | Penalty anchor | Price channel, 10-module | Factorized mixture | Merged mixture |
|---|---|---|---|---|---|---|---|
| 50, 1, 0.3 | 76% | 1075.5±46.8 | 974.1±12.4 | 1029.5±24.9 | 1146.5±23.9 | 1036.4±14.0 | **993.4±11.3** |
| 80, 1, 0.3 | 76% | 1647.9±34.0 | 1555.6±21.5 | 1652.3±24.8 | 1814.1±29.0 | 1673.3±27.1 | **1602.2±17.3** |

Cell: modules, $|V|$, $\bar{\tau}/\bar{p}$ (25 machines); both cells lie outside the trained sizes. Learned arms: per-instance means over seeds $\{301, 302, 303\}$, except the penalty anchor. The anchor is the single seed-301 run of Fig. 3(a) of the main text, and as there its vehicle decisions follow the nearest-vehicle-first rule. Factorized mixture: the deployed policy of Table III of the main text. $u$: minimum loaded utilization $W_{\mathrm{tr}}/(|V|\,C)$ with $C$ the best makespan found on each instance by any reported method, so $u$ is a lower bound. The anchor ran on one processor thread and the other learned arms on a GPU. Re-evaluated on one processor thread, the deployed policy's cell means change by at most 0.09%, and it stays 0.6% and 1.2% longer than the anchor (W/T/L 4/0/16 in each cell).

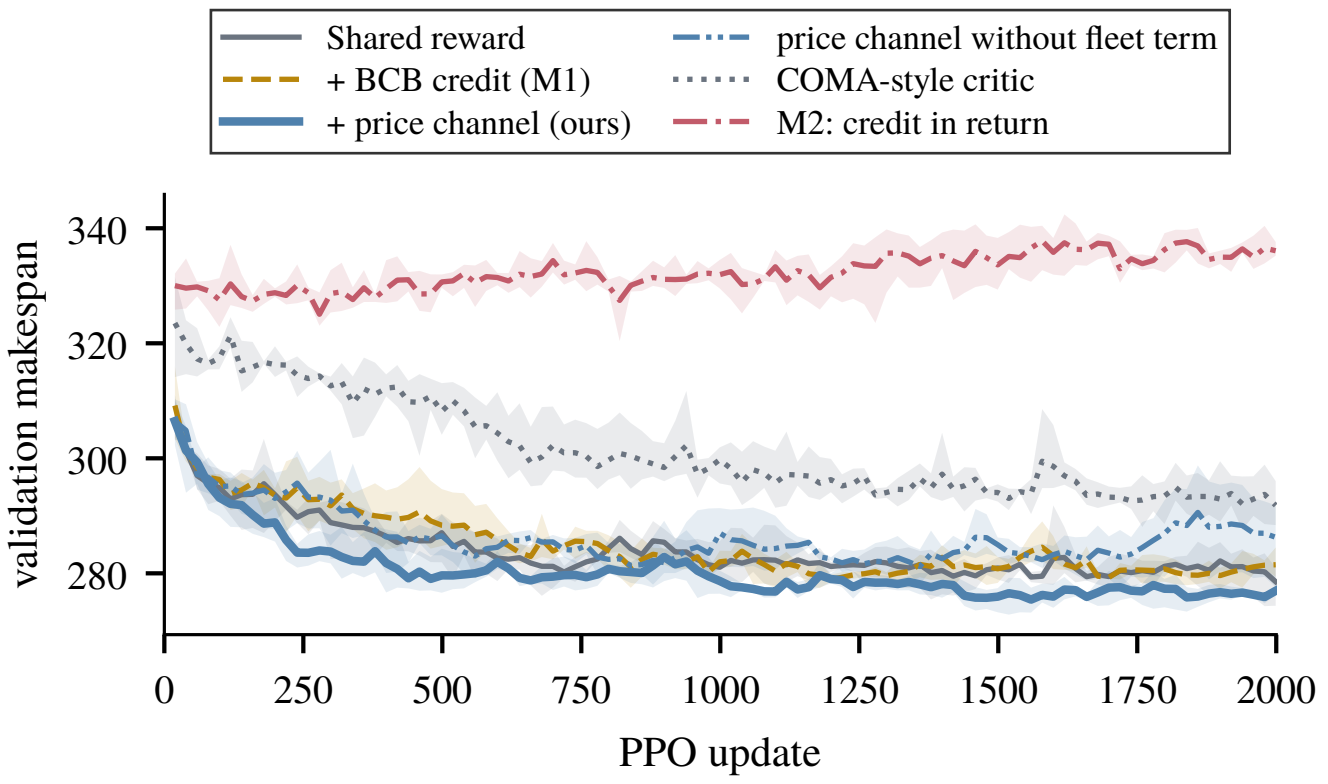


Fig. S-4. Validation makespan (hours) against PPO updates for the six ladder arms (line: seed mean; band: seed range; score: mean over the nine training cells). The curve "price channel without fleet term" is the fleet-term ablation of Table S-VIII, trained with $B_{\mathrm{veh}}$ removed from the bound. The price channel converges below every other arm; writing the credit into the return (M2) alters the objective and never improves.

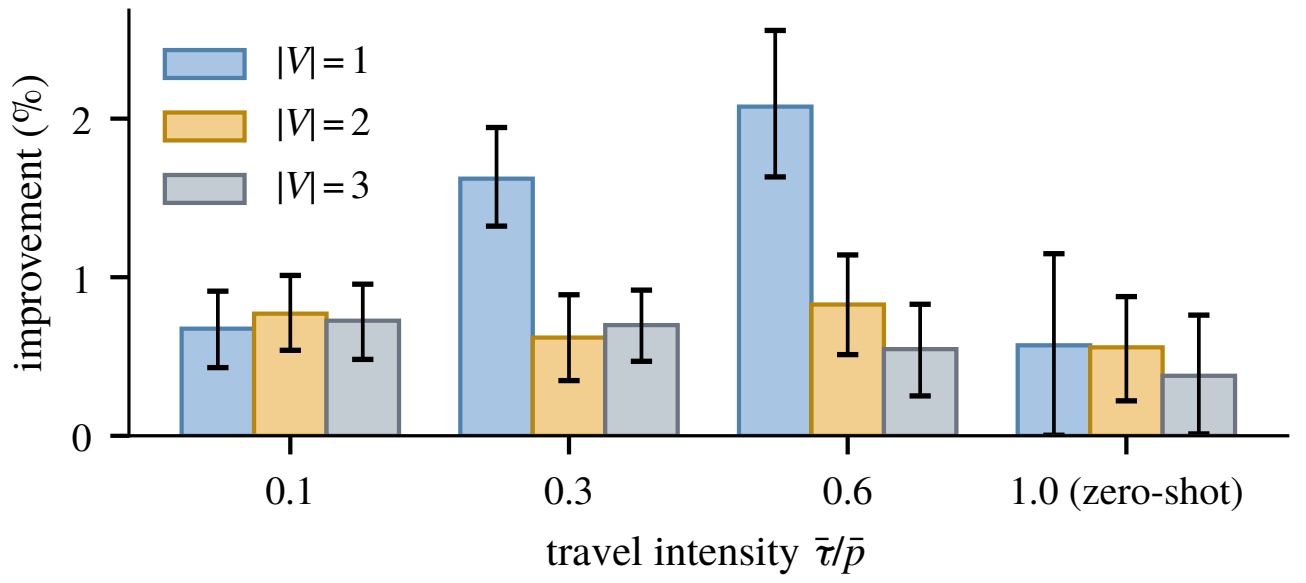


Fig. S-5. Per-cell improvement of the certified-price arm over the credit-only arm (shared reward plus BCB credit, no price channel): per-instance 3-seed-mean difference with 95% bootstrap confidence interval, $n = 100$ per cell. The two arms converge at the zero-shot column.

The first premise-gate attempt (Section VI-C1 of the main text) ran at $\bar{\tau}/\bar{p} = 0.6$ on single training runs, seed 301. It measured +1.32% at $|V| = 1$ (W/T/L 58/0/42, Holm-corrected $p = 4.96{\times}10^{-2}$) and −0.23% at $|V| = 2$ ($p = 0.76$), so the gate did not pass. The retry at $\bar{\tau}/\bar{p} = 1.0$ used 3 seeds, each passing on its own, and the main text reports the pooled numbers.

Table S-VII compares two control arms with the certified-price arm cell by cell. The myopic arm replaces the certified price with a non-admissible duration price of matched mean

TABLE S-VI

CREDIT-AND-GUIDANCE LADDER ON THE FOUR PRE-SPECIFIED GATE CELLS AT $10 \times 25$ (MAKESPAN, PER-INSTANCE 3-SEED MEAN ± STD, $n = 100$). THE UPPER ROWS ADD ONE BOUND ROLE AT A TIME. THE LOWER ROWS ALTER WHERE THE SAME SIGNAL IS APPLIED OR REMOVE THE FLEET TERM FROM THE BOUND. BOLD: SHORTEST IN THE COLUMN.

| | $\bar{\tau}/\bar{p} = 0.3$ | | $\bar{\tau}/\bar{p} = 0.6$ | |
|---|---|---|---|---|
| Arm | $\lvert V\rvert = 1$ | $\lvert V\rvert = 2$ | $\lvert V\rvert = 1$ | $\lvert V\rvert = 2$ |
| Shared reward | 287.2±7.2 | 249.6±7.8 | 430.9±8.7 | 306.1±8.3 |
| + BCB credit (M1) | 289.3±7.7 | 249.7±8.1 | 431.2±10.1 | 307.0±8.4 |
| + price channel (ours) | **284.6±7.6** | **248.1±8.5** | **422.1±9.4** | **304.4±7.9** |
| M2: credit in return | 365.5±12.8 | 262.4±8.5 | 645.8±17.9 | 365.6±12.4 |
| COMA-style critic | 312.4±9.0 | 258.0±9.0 | 462.7±12.9 | 319.7±8.5 |
| Ours without $B_{\mathrm{veh}}$ | 293.4±8.2 | 248.5±8.2 | 442.5±9.6 | 307.2±8.1 |

size, and its one-sided test, Holm-corrected over the four gate cells, was fixed before the runs. In one cell the certified price is shorter (one-sided Holm-corrected $p = 7.5 \times 10^{-5}$). The arms are not distinguishable in the other three, where the spread across training seeds exceeds their difference. For the COMA-style arm, which learns a counterfactual critic, no test was fixed, so its rows are descriptive. It is longer than the certified arm in all 14 cells, and on the four gate cells it is longer than shared-reward training by +3.3% to +8.8%.

The fleet-term ablation of Section VI-C3 retrains the headline configuration with $B_{\mathrm{veh}}$ removed from the bound. Its reward, action prices, and credit therefore read only the chain and machine terms at training and at test time. Seeds, update budget, parameter count (40,205), and test sets equal those of the headline arm. The test was fixed before the runs. It is a paired one-sided Wilcoxon test on per-instance 3-seed means, BOLT shorter, in the two cells where the fleet term sets the root bound ($|V| = 1$ with $\bar{\tau}/\bar{p} \in \{0.6, 1.0\}$), Holm-corrected over the two cells at $\alpha = 0.05$. Both cells pass, and the other ten grid cells are descriptive (Table S-VIII).

On three production batches, removing the fleet term lengthens the schedules by 12.9–23.0%. In the three production-batch cells of Table III in the main text, $B_{\mathrm{veh}}$ sets the root bound on all 20 instances of each cell. The size-mixture policy was therefore retrained with $B_{\mathrm{veh}}$ removed from the bound, with the recipe, seeds, and parameter count of the deployed policy. Its test was also fixed before the runs. It is the same one-sided Wilcoxon test, deployed policy shorter, Holm-corrected over the three cells at $\alpha = 0.05$. All three

TABLE S-VII
Myopic-price and COMA-style control arms against the certified-price arm at $10 \times 25$ (makespan in hours, per-instance 3-seed means, $n = 100$ per cell).

| Cell | Certified | Arm | Relative (%) | W/T/L | $p$ |
|---|---|---|---|---|---|
| *Myopic duration price (test fixed before the runs)* | | | | | |
| 1, 0.3 | 284.6 | 287.0 | +0.86 | 63/0/37 | $7.5 \times 10^{-5}$ |
| 2, 0.3 | 248.1 | 248.5 | +0.14 | 50/0/50 | 0.61 |
| 1, 0.6 | 422.1 | 422.7 | +0.14 | 52/0/48 | 0.61 |
| 2, 0.6 | 304.4 | 304.3 | −0.03 | 49/0/51 | 0.61 |
| *COMA-style critic (descriptive)* | | | | | |
| 1, 0.1 | 225.9 | 234.2 | +3.6 | 98/0/2 | – |
| 1, 0.3 | 284.6 | 312.4 | +9.8 | 100/0/0 | – |
| 1, 0.6 | 422.1 | 462.7 | +9.6 | 100/0/0 | – |
| 1, 1.0 | 672.1 | 713.4 | +6.1 | 94/0/6 | – |
| 2, 0.1 | 222.2 | 229.3 | +3.2 | 99/0/1 | – |
| 2, 0.3 | 248.1 | 258.0 | +4.0 | 99/0/1 | – |
| 2, 0.6 | 304.4 | 319.7 | +5.0 | 100/0/0 | – |
| 2, 1.0 | 401.4 | 429.5 | +7.0 | 100/0/0 | – |
| 3, 0.1 | 221.8 | 230.0 | +3.7 | 99/0/1 | – |
| 3, 0.3 | 243.0 | 251.2 | +3.4 | 97/0/3 | – |
| 3, 0.6 | 282.7 | 292.3 | +3.4 | 95/0/5 | – |
| 3, 1.0 | 349.2 | 363.1 | +4.0 | 96/0/4 | – |
| 4, 0.6 | 277.6 | 284.9 | +2.6 | 90/0/10 | – |
| 4, 1.0 | 334.4 | 343.0 | +2.5 | 86/0/14 | – |

Cell: $|V|$, $\bar{\tau}/\bar{p}$. Relative: (arm − certified)/certified, positive when the arm is longer. W/T/L: instances on which the certified arm is shorter, equal or longer. $p$: paired one-sided Wilcoxon test, certified arm shorter, Holm-corrected over the four gate cells. No test was fixed for the COMA-style arm.

cells pass (Table S-IX). The ablation's makespan is 13.0% and 12.9% longer at 50 modules and 23.0% longer at 80 modules, with W/T/L 20/0/0 and Holm-adjusted $p = 2.9 \times 10^{-6}$ in each cell. Every seed agrees, with per-seed increases from 8.8% to 34.9%. At the 10- and 15-module training sizes of the mixture, the increase is 0.7–5.0% ($n = 100$ per cell, descriptive).

TABLE S-VIII
Fleet-term ablation on the $10 \times 25$ grid: relative makespan increase of the arm without $B_{\text{veh}}$ over BOLT ($n = 100$ per cell). W/T/L counts the instances on which BOLT is shorter, equal, or longer.

| $\|V\|$ | $\bar{\tau}/\bar{p}$ | Increase (%) | W/T/L | Per seed 301, 302, 303 (%) |
|---|---|---|---|---|
| 1 | 0.1 | +0.1 | 48/0/52 | +0.1, −0.4, +0.6 |
| 1 | 0.3 | +3.1 | 98/0/2 | +3.3, +1.5, +4.5 |
| 1 | 0.6[a] | +4.8 | 99/0/1 | +4.5, +2.9, +7.2 |
| 1 | 1.0[a] | +6.2 | 97/0/3 | +11.2, −3.0, +11.8 |
| 2 | 0.1 | +0.2 | 58/0/42 | +0.4, −0.1, +0.4 |
| 2 | 0.3 | +0.1 | 48/0/52 | +0.2, −0.3, +0.5 |
| 2 | 0.6 | +0.9 | 73/0/27 | +0.6, +0.4, +1.8 |
| 2 | 1.0 | +0.5 | 65/0/35 | +0.3, −0.2, +1.5 |
| 3 | 0.1 | +0.2 | 60/2/38 | +0.4, 0.0, +0.4 |
| 3 | 0.3 | +0.3 | 57/0/43 | +0.2, −0.1, +0.6 |
| 3 | 0.6 | −0.1 | 47/0/53 | −0.4, −0.2, +0.3 |
| 3 | 1.0 | +0.3 | 61/0/39 | 0.0, +0.1, +0.8 |

[a] Pre-specified test cells, both with Holm-adjusted $p < 10^{-17}$. Travel intensity 1.0 lies outside the training grid.

## VIII. Baseline Specifications: CP-SAT Model and Dispatching Rules

The warm-started CP-SAT reference model (OR-Tools) contains one optional interval per operation-machine pair with an exactly-one selection constraint and a machine no-overlap constraint. Circuit constraints link per-vehicle transport intervals

TABLE S-IX
Makespan increase when the size-mixture policy is retrained without $B_{\text{veh}}$, relative to the deployed size-mixture policy. W/T/L counts the instances on which the deployed policy is shorter, equal, or longer.

| Modules | $\|V\|$ | $\bar{\tau}/\bar{p}$ | $n$ | Increase (%) | W/T/L | Per seed 301, 302, 303 (%) |
|---|---|---|---|---|---|---|
| 50[a] | 2 | 0.6 | 20 | +13.0 | 20/0/0 | +14.5, +15.3, +8.8 |
| 50[a] | 2 | 1.0 | 20 | +12.9 | 20/0/0 | +10.6, +18.9, +9.2 |
| 80[a] | 3 | 1.0 | 20 | +23.0 | 20/0/0 | +20.5, +34.9, +13.5 |
| 10 | 1 | 0.6 | 100 | +1.8 | 84/0/16 | +2.0, +2.5, +0.8 |
| 10 | 1 | 1.0 | 100 | +2.6 | 88/0/12 | −1.2, +5.6, +3.6 |
| 10 | 2 | 0.6 | 100 | +1.2 | 72/1/27 | +1.0, +0.6, +2.1 |
| 10 | 2 | 1.0 | 100 | +0.7 | 63/0/37 | +1.0, +0.2, +1.0 |
| 15 | 2 | 0.6 | 100 | +2.5 | 90/0/10 | +2.4, +2.6, +2.5 |
| 15 | 2 | 1.0 | 100 | +5.0 | 96/0/4 | +2.1, +4.4, +8.4 |

Machines: 25. The 10- and 15-module rows are training sizes of the mixture and are descriptive.
[a] Pre-specified test cells, each with Holm-adjusted $p = 2.9 \times 10^{-6}$.

covering the empty and loaded legs. Generalized precedence imposes lags between consecutive operations of a job. A redundant fleet-capacity cumulative constraint and an auxiliary linear energy cut carry the $B_{\text{veh}}$ relaxation of the main text. The model also tightens the horizon to the warm-start makespan and breaks vehicle symmetry. The solver is seeded with the best-of-nine priority-dispatching-rule (PDR) solution. The nine PDR pairs cross the machine rules shortest processing time (SPT), most work remaining (MWKR), and first in first out (FIFO) with the vehicle rules shortest travel time (STT), nearest vehicle first (NVF), and FIFO.

The genetic algorithm (GA) uses a random-key chromosome with one priority key, one machine-preference key, and one vehicle priority offset per operation, plus four weights of a state-dependent vehicle rule. At a vehicle event, the rule scores each pending task by a weighted sum of its min-max normalized STT, NVF, and release-order features and its priority offset. The weights $(1, 0, 0, 0)$, $(0, 1, 0, 0)$, and $(0, 0, 1, 0)$ therefore reproduce the STT, NVF, and FIFO rules exactly. Every chromosome is decoded through the same event-driven simulator as BOLT, so the GA searches the same constructive class and needs no repair, and the independent validator rechecks its best schedule. The population of 100 holds the nine PDR seeds, each decoding to its own PDR makespan, and random chromosomes. Each generation keeps the five best chromosomes, selects parents by tournaments of two, and applies uniform crossover per vector and Gaussian mutation per gene with rate 0.1 and standard deviation 0.1. The search stops when the budget of process CPU time expires (60 CPU-s on one thread unless stated otherwise), and the seeding before the search is not charged to this budget (Fig. S-1).

The CP-SAT fleet-term ablation of Section VI-D3 matches effort, with four search workers on both sides and 12 instances in each of the two one-vehicle cells at $\bar{\tau}/\bar{p} = 0.6$ and 1.0 (Table S-X). Without the fleet constraint, the solver proves a bound below $B(s_0)$ on every instance, and $B(s_0)$ is available without search. The reference is therefore the larger of the two. Against it, the constraint raises the proven bound on all 24 instances, by 11–13 h (3.6–4.7% at 0.6 and 2.2–2.8% at 1.0). Measured against the model without the constraint and

without $B(s_0)$, the same bounds rise by 27% and 87%, which mostly reflects the missing root bound. The constraint shortens the incumbent by 4.8% at 0.6 (W/T/L 12/0/0, $p = 4.9 \times 10^{-4}$) and leaves it unchanged at 1.0 (6/0/6, $p = 1.0$), and no variant proves an instance optimal in 300 s. Removing only the auxiliary linear energy cut, a weaker form of the same relaxation, leaves the proven bound unchanged on all 12 instances of the 0.6 cell, the only cell where it was run. The excess of 11–13 h matches the lead-in and lead-out times of the root refinement in the certificate section ($h$, the earliest start of a loaded move, and $t$, the shortest processing after one). On every instance, the solver's bound lies between $B(s_0) + h$ and $B(s_0) + h + t$.

TABLE S-X
CP-SAT with and without the redundant fleet constraint at 10 modules and one vehicle: mean proven lower bounds and incumbents in hours, 12 instances per cell.

| | $\bar{\tau}/\bar{p} = 0.6$ | $\bar{\tau}/\bar{p} = 1.0$ |
|---|---|---|
| *Lower bound* | | |
| Root bound $B(s_0)$ | 289.9 | 483.3 |
| CP-SAT without the constraint | 237.4 | 264.7 |
| Larger of the two above | 289.9 | 483.3 |
| CP-SAT with the constraint | 301.9 | 495.2 |
| Excess over the larger (h) | 11–13 | 11–13 |
| Excess over the larger (%) | 3.6–4.7 | 2.2–2.8 |
| *Incumbent* | | |
| CP-SAT without the constraint | 413.0 | 648.3 |
| CP-SAT with the constraint | 394.4 | 648.4 |
| W/T/L, with against without | 12/0/0 | 6/0/6 |
| $p$ | $4.9 \times 10^{-4}$ | 1.0 |

Without the constraint: the model without the cumulative fleet constraint, its makespan links and the energy cut. Both models run 300 s with four workers and the same warm start. Excess: per-instance range of the bound with the constraint minus the larger of $B(s_0)$ and the bound without it. W/T/L: instances on which the incumbent with the constraint is shorter, equal or longer; $p$: exact two-sided Wilcoxon test. No run proves an instance optimal (0 of 24 per variant); wall times 300.0–300.7 s.

Statistical protocol: equivalence tests (Section VI-D1) are reported as cell-wise TOST assessments against the pre-specified ±2% margin, with 90% bootstrap confidence intervals. The external comparisons of Section VI-B3 are also reported uncorrected; their $p$-values remain below conventional familywise corrections.

Each training seed alone favors the same arm as the seed mean in 20 of the 23 headline comparison cells (Table S-XI). The tests of the main text average the makespans of the three seeds per instance, so their $p$-values reflect instance variation only. A two-level bootstrap adds the seed variation. Each of its 10,000 replicates draws three seeds with replacement and then the instances. The resulting 95% interval excludes zero in 18 cells. On the production batches, every seed of the size-mixture policy is shorter on average than the best dispatching pair, with interval bounds from +1.5% to +15.6%. The fleet-term ablation on production batches and the price channel on the four gate cells also hold for every seed, and their intervals exclude zero. In three cells, one seed favors the other arm. At one vehicle and $\bar{\tau}/\bar{p} = 1.0$, a travel intensity outside training, seed 302 of BOLT is 1.5% longer than the best pair. The ablation without $B_{\text{veh}}$ trained with seed 302 is 3.0% shorter than that BOLT seed. At 50 modules, two vehicles and $\bar{\tau}/\bar{p} = 1.0$, seed 301 of the size-mixture policy is 1.8% longer than the CP-SAT 3600-s incumbent. The intervals of these three cells include zero. In two further grid cells, every seed favors BOLT, but the intervals include zero. Their mean margins are 1.3% at one vehicle with $\bar{\tau}/\bar{p} = 0.3$ and 0.4% at three vehicles with $\bar{\tau}/\bar{p} = 1.0$. With three seeds, the seed level of the bootstrap has only ten distinct draws, so the intervals are approximate.

TABLE S-XI
Relative makespan difference of each training seed alone, number of seeds with the sign of the seed mean, and two-level bootstrap 95% interval for the headline comparisons. The interval resamples training seeds and then instances. Positive values favor the first-named arm.

| Cell | Per seed (%) 301 | 302 | 303 | Agree | 95% interval (%) | Same sign (%) |
|---|---|---|---|---|---|---|
| *BOLT vs. best dispatching pair, 10 modules (Table II)* | | | | | | |
| 1, 0.3 | +0.1 | +0.7 | +3.1 | 3/3 | [−0.1, +2.9] | 95.8 |
| 1, 0.6 | +5.1 | +5.0 | +8.4 | 3/3 | [+4.3, +8.3] | 100.0 |
| 1, 1.0 | +10.4 | −1.5 | +13.4 | 2/3 | [−1.3, +13.3] | 96.4 |
| 2, 0.3 | +1.9 | +1.7 | +2.0 | 3/3 | [+1.4, +2.3] | 100.0 |
| 2, 0.6 | +1.3 | +1.3 | +2.0 | 3/3 | [+0.9, +2.3] | 100.0 |
| 2, 1.0 | +1.7 | +1.6 | +2.5 | 3/3 | [+1.2, +2.8] | 100.0 |
| 3, 0.3 | +1.9 | +1.8 | +1.9 | 3/3 | [+1.4, +2.4] | 100.0 |
| 3, 0.6 | +1.5 | +1.5 | +1.3 | 3/3 | [+0.9, +1.9] | 100.0 |
| 3, 1.0 | +0.3 | +0.8 | +0.2 | 3/3 | [−0.3, +1.2] | 87.5 |
| *Size-mixture policy vs. best dispatching pair (Table III)* | | | | | | |
| 50, 2, 0.6 | +1.6 | +3.2 | +9.4 | 3/3 | [+1.5, +9.2] | 100.0 |
| 50, 2, 1.0 | +3.1 | +6.0 | +11.3 | 3/3 | [+3.0, +11.1] | 100.0 |
| 80, 3, 1.0 | +9.2 | +14.4 | +15.4 | 3/3 | [+9.4, +15.6] | 100.0 |
| *Size-mixture policy vs. CP-SAT 3600 s (Table III)* | | | | | | |
| 50, 2, 1.0 | −1.8 | +1.2 | +6.9 | 2/3 | [−1.9, +6.5] | 82.6 |
| 80, 3, 1.0 | +0.1 | +5.7 | +6.9 | 3/3 | [+0.1, +7.8] | 97.8 |
| *BOLT vs. no fleet term, 10 modules (Table S-VIII)* | | | | | | |
| 1, 0.6 | +4.5 | +2.9 | +7.2 | 3/3 | [+3.0, +7.3] | 100.0 |
| 1, 1.0 | +11.2 | −3.0 | +11.8 | 2/3 | [−2.8, +13.6] | 94.2 |
| *Size-mixture policy vs. no fleet term (Table S-IX)* | | | | | | |
| 50, 2, 0.6 | +14.5 | +15.3 | +8.8 | 3/3 | [+3.9, +21.5] | 99.8 |
| 50, 2, 1.0 | +10.6 | +18.9 | +9.2 | 3/3 | [+3.9, +22.0] | 99.8 |
| 80, 3, 1.0 | +20.5 | +34.9 | +13.5 | 3/3 | [+10.7, +33.8] | 100.0 |
| *Price channel vs. credit-only arm (Table S-VI)* | | | | | | |
| 1, 0.3 | +0.9 | +0.8 | +3.2 | 3/3 | [+0.3, +3.3] | 99.7 |
| 2, 0.3 | +0.5 | +0.6 | +0.7 | 3/3 | [+0.2, +1.0] | 99.9 |
| 1, 0.6 | +1.8 | +1.2 | +3.3 | 3/3 | [+0.1, +4.3] | 98.4 |
| 2, 0.6 | +0.8 | +0.8 | +1.0 | 3/3 | [+0.2, +1.5] | 99.5 |

Cell: $|V|$ and $\bar{\tau}/\bar{p}$ at 10 modules ($n$ = 100), or modules, $|V|$ and $\bar{\tau}/\bar{p}$ for production batches ($n$ = 20). The difference is taken in percent of the dispatching pair, the CP-SAT incumbent or the credit-only arm, and of the arm with the fleet term in the ablation rows. Same sign: share of the 10,000 replicates with the sign of the seed-mean difference.

Reduction check: with an ample fleet and zero travel times the implementation reduces to the FJSP-TL pipeline of [10], matching its published greedy mean makespan of 210.9, so the transport extensions leave the base solver unchanged.

## A. Timing Rules of the Solver Space

Table S-XII keeps every machine assignment and every machine and vehicle sequence of BOLT's schedules and re-computes only the start times under the timing rules of the CP-SAT model. That model differs from the environment in two rules. It holds no destination machine during a move, and

it lets a vehicle run its empty leg before the task is released and wait at the pickup cell. The start times are recomputed in topological order of the fixed precedence graph, and an independent checker of the CP-SAT semantics accepts every re-timed schedule. Pooled over the 10-module cells, re-timing shortens the schedules by 3.2% and closes 33% of the gap to the CP-SAT incumbent. A Shapley split assigns 0.85% of makespan to the reservation and 2.4% to the early empty legs, so the reservation accounts for 8.9% of the gap. At production scale, re-timing shortens the schedules by 1.9–3.5%. The decomposition keeps BOLT's own sequences, and a policy acting in the larger space could choose other sequences. Its timing share is therefore a lower bound on the part of the gap that the restrictions of the constructive class cause.

TABLE S-XII
MAKESPAN OF BOLT'S SCHEDULES BEFORE AND AFTER RE-TIMING UNDER THE TIMING RULES OF THE CP-SAT MODEL (HOURS, CPU ROLLOUTS, PER-INSTANCE 3-SEED MEANS).

| Cell | BOLT | Re-timed | Reduction (%) | CP-SAT | Gap closed (%) |
|---|---|---|---|---|---|
| *10 modules, $n$ = 100 per cell* | | | | | |
| 1, 0.3 | 284.7 | 277.2 | 2.61 | 257.4 | 27.3 |
| 1, 0.6 | 422.2 | 416.8 | 1.29 | 394.2 | 19.5 |
| 1, 1.0 | 671.9 | 668.8 | 0.46 | 639.7 | 9.6 |
| 2, 0.3 | 248.1 | 239.7 | 3.42 | 225.5 | 37.5 |
| 2, 0.6 | 304.3 | 292.3 | 3.97 | 274.1 | 39.9 |
| 2, 1.0 | 401.7 | 388.8 | 3.20 | 369.3 | 39.7 |
| 3, 0.3 | 243.1 | 233.8 | 3.83 | 221.5 | 43.0 |
| 3, 0.6 | 282.9 | 267.9 | 5.29 | 245.8 | 40.3 |
| 3, 1.0 | 349.2 | 331.6 | 5.02 | 302.5 | 37.6 |
| All cells | 356.5 | 346.3 | 3.23 | 325.5 | 32.8 |
| *Production batches, size-mixture policy, $n$ = 20 per cell* | | | | | |
| 50, 2, 0.6 | 1052.2 | 1015.4 | 3.49 | 1029.8 | – |
| 50, 2, 1.0 | 1669.0 | 1637.3 | 1.89 | 1704.1 | – |
| 80, 3, 1.0 | 1883.3 | 1840.8 | 2.26 | 1974.2 | – |

Cell: $|V|$, $\bar{\tau}/\bar{p}$ at 10 modules, or modules, $|V|$, $\bar{\tau}/\bar{p}$ (machines: 25). BOLT: cell means within 0.1% (10 modules) and 0.8% (production batches) of the GPU means of Tables II and III. Reduction: mean per-instance (BOLT − re-timed)/BOLT. CP-SAT: warm-started incumbent, 300 s at 10 modules and 3600 s on production batches. Gap closed: (BOLT − re-timed)/(BOLT − CP-SAT) on the cell means; it is not given for production batches, where BOLT is already shorter than the incumbent in 2 of 3 cells. A Shapley split of the pooled 10-module reduction assigns 0.85% of makespan to removing the destination-machine reservation and 2.38% to allowing empty legs before task release.

## IX. Training Configuration and Generator Parameters

Table S-XIII records the complete training configuration. Every arm of the credit-and-guidance ladder shares it and differs only in the credit and guidance settings named in the main text. All 23 per-run configuration snapshots are included in the code release, so each value can be audited without retraining. The learned-critic (COMA-style) arm follows its published update convention and does not normalize advantages. All other arms share the masked per-environment normalization listed here. Table S-XIV records the instance generator. Its travel time $\tau$ covers the move between zones, and loading and unloading are not modeled separately.

## X. Notation

Table S-XV collects the notation of the main text.

TABLE S-XIII
TRAINING CONFIGURATION AND HYPERPARAMETERS, SHARED BY EVERY LEARNED ARM; THE SEED ROW APPLIES TO THE THREE-SEED ARMS OF THE MAIN TEXT'S PROTOCOL PARAGRAPH, AND SINGLE-RUN ARMS USE SEED 301.

| Setting | Value |
|---|---|
| *Reinforcement learning* | |
| Algorithm | PPO, one critic, two heads[b] |
| Discount factor $\gamma$ | 1 (required by Theorem 1(c)) |
| GAE parameter $\lambda$ | 0.98 |
| Clip parameter $\epsilon$ | 0.2 (ratio only; value loss unclipped) |
| Entropy coefficient | 0.01 |
| Value-loss coefficient | 0.5 |
| Advantage normalization | per-environment, masked |
| *Optimization* | |
| Optimizer | Adam ($\beta_1$=0.9, $\beta_2$=0.999) |
| Learning rate | $3 \times 10^{-4}$, constant (no schedule) |
| Gradient clipping | none |
| Epochs per update | 4 |
| Minibatch size | 1024 transitions |
| Parallel environments | 20 |
| Transitions per update | ≈ 6,700 (28 gradient steps) |
| PPO updates | 2,000[a] |
| Training seeds | 301, 302, 303[a] |
| *Network* | |
| Encoder | dual attention, 2 layers, 4 heads per stream |
| Per-layer width | 32 (concatenated), then 8 (head mean) |
| Embedding width $d$ | 8 |
| Type embeddings | $5 \times 8$ operation, $10 \times 8$ station/vehicle |
| Actor heads | 2× MLP, 3 layers, hidden 64 |
| Critic | MLP, 3 layers, hidden 64 |
| Parameters | 40,205 trained; 34,892 deployed (no critic) |
| Checkpoint size | 185 KiB |
| *Compute* | |
| Training device | NVIDIA RTX PRO 5000 Blackwell, 48 GB |
| Evaluation CPU | Intel Core Ultra 9 285K, 24 cores |
| Training cost per arm | 3.4 GPU-hours (exclusive device) |
| Decision latency and memory | Table S-XXXIII |

[a] The main and control arms use these values. Two families do not: the transport-as-penalty anchor and the premise-gate arms train for 1,000 updates, and the anchor of Fig. 3(a) of the main text is a single run at seed 301. The pre-specified premise test compares arms that match on both, three seeds and 1,000 updates per side.
[b] Event-class heads, one acting per step, implemented with MAPPO [11].

## XI. Proofs for the Admissible Bound

The three component bounds are treated in Propositions S1–S3 below; together they prove Proposition 1 of the main text.

**Lemma S1** (Append-only releases)**.** *In every $\pi \in C(s)$: (i) no unstarted op on machine $m$ starts before a release $r_m \geq b_m$; (ii) no future loaded move served by vehicle $v$ begins before $f_v$.*

*Proof.* Commits and serves occur at nondecreasing event times, and $b_m, f_v$ only increase: a commit sets the release $\mathrm{rel}(o) = \max(\mathrm{jr}(i), b_m) \geq b_m$, and a serve sets depart $= \max(f_v, \mathrm{rel}(o)) \geq f_v$. A reserved machine has $b_m = \hat{c}$ and its true release equals its committed op's completion $\geq \hat{c} = b_m$ (Lemma S2). Per-resource timelines never back-fill an earlier idle gap. □

**Lemma S2** (Optimistic in-transit envelope)**.** *Let $o$ (job $i$) be committed in-transit at $s$ to a fixed machine $m$, its move running from cell $a$ to cell $b$ with release* $\mathrm{rel}(o) = \max(\mathrm{jr}(i), b_m)$ *and loaded leg $\tau_l = \tau(a, b)$, and set the optimistic in-transit*

TABLE S-XIV
GENERATOR PARAMETERS FOR THE PPVC-T INSTANCE FAMILY. DURATIONS IN HOURS. SOURCE MARKS VALUES TAKEN FROM THE PPVC GUIDEBOOK OF THE BUILDING AND CONSTRUCTION AUTHORITY (GUIDEBOOK), ENGINEERING ESTIMATES (ESTIMATE), INDUSTRY-NORM RANGES (NORM), AND VALUES SET BY THIS STUDY (STUDY).

| Parameter | Value | Source |
|---|---|---|
| Modules (jobs) per instance | 10; 15–80 (transfer, VI-B2) | Study |
| Operations per module | 22 (220 per instance at 10 modules) | Guidebook |
| Stations | 25 across 9 types | Study |
| Stations per type (A–H, Q) | 3, 2, 4, 4, 3, 3, 1, 2, 3 | Study |
| Eligible stations per operation | 1–4 (mean 3.11) | Study |
| Routing classes | concr./steel × wet/dry | Guidebook |
| Processing time per operation | 1–12 h, ±20% across stations | Estimate |
| Curing lag (after concrete pour) | 24–48 h | Norm |
| Ponding lag (after waterproofing) | 24–48 h | Norm |
| Drying lag (two paint coats) | 12–24 h per coat | Norm |
| Total lag / total processing | 0.88 (range 0.75–1.06) | Estimate |
| Layout | 3 × 3 zone grid, one zone per type | Study |
| Travel time | Manhattan zone distance × speed | Study |
| Speed constant | calibrated to the target $\bar{\tau}/\bar{p}$ | Study |
| Travel intensity $\bar{\tau}/\bar{p}$ | 0.1, 0.3, 0.6; 1.0 zero-shot | Study |
| Fleet size | 1, 2, 3 (trained); 4 (zero-shot) | Study |
| Vehicle start | zone A (mold-line head) | Study |
| Training instances | 20/update, fresh (40,000 total) | Study |
| Validation / test instances | 100 per cell (16 cells) | Study |
| Scale-up test instances | 30 (20, 30 modules), 20 (50, 80) | Study |
| Size-mixture training sizes | 10/15/20, round-robin | Study |

TABLE S-XV
NOTATION FOR FJSP-TL-T (MAIN TEXT, SECTION III).

| Symbol | Meaning |
|---|---|
| $\mathcal{J}$, $n$ | set of jobs; number of jobs (index $i$) |
| $O_{ij}$, $n_i$ | $j$-th operation of job $i$; operations in job $i$ |
| $\mathcal{M}$, $m$ | set of machines (index $k$) |
| $\mathcal{M}_{ij}$ | machines eligible for operation $O_{ij}$ |
| $p_{ijk}$ | processing time of $O_{ij}$ on machine $k$ |
| $\ell_{ij}$ | machine-free post-operation time lag after $O_{ij}$ |
| $\mathrm{loc}(k)$ | layout location (station cell) of machine $k$ |
| $\tau(a,b)$ | loaded travel time from location $a$ to $b$ |
| $\tau_e(a,b)$ | empty travel time; $\tau_e = \tau$ in this study |
| $\mathcal{V}$ | homogeneous vehicle fleet (index $v$) |
| $\mu_{ij}$ | machine assigned to $O_{ij}$ |
| $\rho = (i,a,b)$ | transport task: job $i$ from cell $a$ to cell $b$ |
| $C_{ij}$ | completion time of operation $O_{ij}$ |
| $C_{\max}$ | makespan (objective) |
| $\mathrm{jr}(i)$, $b_m$, $f_v$ | job ready, machine busy-until, vehicle free times |
| $c(o)$, $\delta_{\min}(o)$ | chain floor of $o$; its minimal loaded-leg envelope |
| $W_g$, $W_{\mathrm{tr}}$ | pending work of group $g$; pending transport work |
| $C(s)$, $\mathcal{F}(s)$ | constructive class from $s$; class searched by CP-SAT |

*completion* $\hat{c} := \mathrm{rel}(o) + \tau_l + p(o,m)$. *Then* $\gamma_\pi(o) \geq \hat{c}$ *for every* $\pi \in C(s)$.

*Proof.* The move of $o$ is served by some vehicle $v$ with depart $= \max(f_v, \mathrm{rel}(o)) \geq \mathrm{rel}(o)$, empty leg $\tau_e \geq 0$, and loaded leg $\tau_l$ (fixed at commit), so arrival = depart $+ \tau_e + \tau_l \geq \mathrm{rel}(o) + \tau_l$. Machine $m$ is fixed, hence $\gamma_\pi(o)$ = arrival $+ p(o,m) \geq \mathrm{rel}(o) + \tau_l + p(o,m) = \hat{c}$. The envelope books zero empty move and zero queueing, both nonnegative increments in any completion. □

**Proposition S1** ($B_{\mathrm{chain}}$ admissibility)**.** *At every reachable* $s$ *and every* $\pi \in C(s)$, $c(o) \leq \gamma_\pi(o)$ *for all* $o$*; hence* $B_{\mathrm{chain}}(s) \leq C_{\max}(\pi)$*, with equality at realized ops.*

*Proof.* Two facts. *(Chain step.)* For any $\pi$ and $o > \mathrm{first}(i)$, precedence, lag, and the (possibly zero) loaded move give $\gamma_\pi(o) \geq \gamma_\pi(\mathrm{pred}\, o) + \ell(\mathrm{pred}\, o) + \delta_{\min}(o) + p_{\min}(o)$, using $\tau(\cdot,\cdot) \geq \delta_{\min}(o)$ (same-cell $\Rightarrow$ $\delta_{\min} = 0$) and $p(o,\cdot) \geq p_{\min}(o)$; and $\gamma_\pi(\mathrm{first}(i)) \geq p_{\min}(\mathrm{first}(i))$. *(Gap structure.)* Commits and realizations within a job occur in index order, since the per-job in-order gate blocks a job's next commit until its in-transit operation realizes; hence every suffix-raise anchors at the head of the still-uncommitted suffix and shifts it rigidly, preserving the consecutive gaps $\Delta_q$ for all uncommitted ops.

Induct on events. At $s_0$ the seed equals $p_{\min}(\mathrm{first}(i)) + \sum \Delta_q$, bounded by $\gamma_\pi(o)$ by telescoping the chain step. At a transition, an untouched op keeps its value and bound. The single raise anchors at $o^\star$: if $o^\star$ realizes, $c(o^\star) = \gamma_\pi(o^\star)$ exactly (a frozen decision every $\pi$ extends); if $o^\star$ commits in-transit, $c(o^\star) = \hat{c} \leq \gamma_\pi(o^\star)$ (Lemma S2). For the uncommitted suffix, the gap structure gives $c(o') = c(o^\star) + \sum_{q=o^\star+1}^{o'} \Delta_q \leq \gamma_\pi(o^\star) + \sum \Delta_q \leq \gamma_\pi(o')$ by the chain step. Realized equality is stable: later raises anchor strictly beyond a realized op. □

**Proposition S2** ($B_{\mathrm{mch}}$ admissibility for $C$)**.** *For every* $\pi \in C(s)$, $B_{\mathrm{mch}}(s) \leq C_{\max}(\pi)$.

*Proof.* Fix a group $g$, set $C := C_{\max}(\pi)$, and let $r_m$ be the true release of machine $m$. By Lemma S1(i), $b_m \leq r_m$ (idle or realized: $r_m = b_m$; reserved: $r_m = \gamma_\pi(o_{\mathrm{res}}) \geq \hat{c} = b_m$ by Lemma S2, the machine doing no unstarted work until arrival). Each $o \in \mathcal{U}(s)$ with $g(o) = g$ runs on one $m \in \mathrm{elig}(o) = g$ for $p(o,\cdot) \geq p_{\min}(o)$ in a no-overlap interval inside $[r_m, C]$; in-transit ops lie outside $\mathcal{U}(s)$, so there is no double count. Summing per machine, $W_g \leq \sum_{m\in g}(C - r_m)$, whence $\mathrm{busy}(g) + W_g \leq \sum_{m\in g} b_m + \sum_{m\in g}(C - r_m) \leq |g|C$ (using $b_m \leq r_m$). Maximize over $g$. Substituting the optimistic $b_m = \hat{c} \leq r_m$ only lowers the left side. □

**Proposition S3** ($B_{\mathrm{veh}}$ admissibility and domination)**.** *For every* $\pi \in C(s)$, $B_{\mathrm{veh}}^{\mathrm{head}}(s) \leq B_{\mathrm{veh}}^{\mathrm{agg}}(s) \leq C_{\max}(\pi)$.

*Proof.* Assume $W_{\mathrm{tr}} > 0$ and set $C := C_{\max}(\pi)$. Each active task (released but not yet assigned to a vehicle; assignment removes it from the pool, its loaded leg then living in the serving vehicle's $f_v$), from its pickup $a$ to its drop-off $b$, has an exact loaded duration $\tau(a,b)$ still to be served; each uncommitted mandatory op ($\delta_{\mathrm{tr}}(o) > 0$, the minimal loaded-leg envelope, which coincides with $\delta_{\min}(o)$ of the main text

and is a valid lower envelope of its loaded move) forces a genuine move of duration $\geq \delta_{\text{tr}}(o)$ in every completion. The two sets are disjoint (active tasks are committed, hence outside $\mathcal{U}(s)$, and there is one pending move per job), so $W_{\text{tr}}$ counts distinct moves at durations no greater than their actual loaded time. Each move seizes one vehicle for at least that duration; on vehicle $v$ the loaded intervals are disjoint and lie in $[f_v, C]$ (Lemma S1(ii), with arrival $\leq C$). Dropping empty legs only removes work. Hence $W_{\text{tr}} \leq \sum_v (C - f_v)$, giving $C \geq (\sum_v f_v + W_{\text{tr}})/|V| = B^{\text{agg}}_{\text{veh}}$, and $B^{\text{agg}}_{\text{veh}} \geq \min_v f_v + W_{\text{tr}}/|V| = B^{\text{head}}_{\text{veh}}$ since $\sum_v f_v \geq |V| \min_v f_v$. □

*Remark* S1 (Deployed form). The deployed bound uses $B^{\text{head}}_{\text{veh}}$, whose incremental update is $O(1)$ per commitment. The two forms coincide identically at $|V| = 1$ and at every root state ($f_v \equiv 0$). Mid-episode the aggregate form is tighter on most states (mean uplift 0.8% at $|V| = 2$ up to 10.3% at $|V| = 4$), but it changes the composed $B = \max(B_{\text{chain}}, B_{\text{mch}}, B_{\text{veh}})$ on at most 7.1% of states and by at most 0.02% on average, so the cheap form costs nothing measurable. The head form can decrease along a trajectory, because dispatching a vehicle other than the earliest-free one lowers $W_{\text{tr}}$ while $\min_v f_v$ stays fixed, and such a step earns a positive reward. By contrast, the aggregate form is monotone, and telescoping and admissibility hold for both forms.

*Remark* S2 (The two-component bound cannot certify fleet contention). Take $K$ jobs, each with one operation on its own machine in cell $A$ followed by one operation on its own machine in cell $B$, all processing times 1, travel time $\tau(A, B) = \tau$, and a single vehicle. Every per-job chain floor is $2 + \tau$, and every machine group carries load at most 2, so $\max(B_{\text{chain}}, B_{\text{mch}}) = 2 + \tau$ at the root regardless of $K$. Yet every schedule must serialize the $K$ loaded moves on the one vehicle, so $C_{\max} \geq K\tau$: the two-component bound stays admissible but its gap grows without bound in $K$. The aggregate fleet term closes it, $B_{\text{veh}} = \min_v f_v + W_{\text{tr}}/|V| = K\tau$ at the root, which is why the coupled class forces the fleet to be charged as one aggregate resource.

## XII. Proof of Theorem 1 and Extended Remarks

**Theorem 1 (restated).** *Let* $\chi_{\text{m}}, \chi_{\text{v}} \in \{0, 1\}$ *select the machine and vehicle components and set* $B(s) = \max(B_{\text{chain}}(s), \chi_{\text{m}} B_{\text{mch}}(s), \chi_{\text{v}} \mathbf{1}_{|V|>0} B_{\text{veh}}(s))$*; the main result takes* $\chi_{\text{m}} = \chi_{\text{v}} = 1$*. (a) For every reachable* $s$ *and every* $\pi \in C(s)$*,* $B(s) \leq C_{\max}(\pi)$*. (b) At termination* $s_T$*,* $B(s_T) = C_{\max}$*. (c) With* $r_t = (B(s_t) - B(s_{t+1}))\,\kappa$*,* $\kappa > 0$ *fixed, and* $\gamma = 1$*,* $\sum_{t=0}^{T-1} r_t = (B(s_0) - C_{\max})\,\kappa$*, so* $\arg\max$ *of the return equals* $\arg\min$ *of* $C_{\max}$*.*

*Proof.* (a) Each component is $\leq C_{\max}(\pi)$ by Propositions S1–S3; the max of admissible bounds is admissible. (b) With all ops realized, $c(o) = \gamma(o)$, the realized completion of $o$ (Prop. S1, equality clause), so $B_{\text{chain}}(s_T) = \max_o \gamma(o) = C_{\max}$, the realized makespan; $\mathcal{U}(s_T) = \varnothing \Rightarrow W_g = 0 \Rightarrow B_{\text{mch}}(s_T) = \max_g \text{busy}(g)/|g| \leq \max_m b_m \leq C_{\max}$; and $W_{\text{tr}} = 0 \Rightarrow B_{\text{veh}}(s_T) = 0$. No component exceeds $C_{\max}$ and $B_{\text{chain}}$ attains it, so $B(s_T) = C_{\max}$. (c) The sum telescopes; substitute $B(s_T) = C_{\max}$. The return is a positive-affine transform of $-C_{\max}$. □

*Remark* S3 (Scope of the $\mathcal{F}$ bound). The inequality $B(s) \leq \text{opt}_{\mathcal{F}}(s)$ holds at $s_0$ (Corollary 2 of the main text) but *not* at general mid-episode states with an in-transit op. The envelope $\hat{c} = \max(\text{jr}(i), b_m) + \tau_l + p$ charges the loaded travel on top of a binding machine clock $b_m$; $\mathcal{F}$ (no reservation) lets that move overlap the machine's busy window, so $B$ can exceed $\text{opt}_{\mathcal{F}}$, and the same reservation optimism enters $B_{\text{mch}}$ through $b_m = \hat{c}$. Example: two jobs, cells $A \neq B$, $\tau(A, B) = 10$, $|V| = 1$; a realized $[0, 100]$ op on $m_1 \in B$, a cross-cell op committed to $m_1$ gives $B_{\text{chain}} = \hat{c} = 115$, while $\mathcal{F}$ overlaps the move and finishes at 105. This does not affect Theorem 1 (stated over $C$) or any reported quantity (CP-SAT is built only at $s_0$); we state the CP-SAT lower bound at the root and the RL guarantees over $C$, and make no universal $\mathcal{F}$ claim.

*Remark* S4 (Reward carries $\kappa$; potential shaping). The reward $r_t = (B(s_t) - B(s_{t+1}))\kappa$ carries a positive constant factor $\kappa = 1/(\overline{p} - \underline{p} + \varepsilon)$; it is argmax-preserving and the telescoped total is $(B(s_0) - C_{\max})\kappa$. In shaping terms, $r_t$ equals the potential-based term $\gamma\Phi(s_{t+1}) - \Phi(s_t)$ at $\gamma = 1$ with $\Phi = -\kappa B$ and *no* base reward; the construction deliberately violates the episodic convention $\Phi(\text{terminal}) = 0$, since $\Phi(s_T) = -\kappa C_{\max}$ is what carries the objective, and we invoke no shaping-invariance theorem; Theorem 1(c) is exact for $\gamma = 1$.

*Remark* S5 (Degenerate reduction). If $\tau \equiv 0$ then $\delta_{\min} \equiv 0$ and $\delta_{\text{tr}} \equiv 0$, so the seed collapses to the lag-only recursion, $W_{\text{tr}} = 0 \Rightarrow B_{\text{veh}} = 0$, and every task serves instantly ($\hat{c} = \text{rel}(o) + p$, which is then the realized completion). Hence $B = \max(B^{\text{lag}}_{\text{chain}}, B_{\text{mch}})$ recovers the transport-free chain bound, and the counterexample above vanishes ($\tau_l = 0$). The construction degrades correctly.

## XIII. Analytic Credit: Construction, Lemma, and Results

The analytic credit of Section V-D of the main text, the bound-counterfactual baseline (BCB), relies on one fact. At step $t$ exactly one head samples $a_t \sim \pi_\theta(\cdot \mid s_t)$, so each decision's effect belongs to the head that made it. Passing the event, which commits nothing, gives the action-independent reference of a difference reward [12]. The bound at the virtual pass state $s^{\text{pass}}_t$ gives the pass reward $r^{\text{pass}}_t = (B(s_t) - B(s^{\text{pass}}_t))\,\kappa$, and BCB subtracts it from the generalized-advantage-estimation advantage, $A'_t = A^{\text{GAE}}_t - r^{\text{pass}}_t$. BCB thus replaces the learned counterfactual critic of COMA [13] with a closed form and needs no additional network. The lemma below shows that the subtraction is a state control variate.

**Lemma (BCB is a zero-mean state control variate).** $r^{\text{pass}}_t$ *is a deterministic function of* $s_t$ *(equivalently of* $(s_t, \eta_t)$*), independent of the sampled action* $a_t$*. Consequently, for the advantage-weighted score estimator,*

$$\mathbb{E}\big[\textstyle\sum_t \nabla_\theta \log \pi_\theta(a_t \mid s_t) A'_t\big] = \mathbb{E}\big[\textstyle\sum_t \nabla_\theta \log \pi_\theta(a_t \mid s_t) A^{\text{GAE}}_t\big].$$

*Proof. Action independence.* The pass floor $B(s^{\text{pass}}_t)$, together with the auxiliary next-event and vehicle-availability evaluations it invokes, reads only state and event fields (operation, vehicle, and task clocks, masks, candidates, and $\tau$): the machine

branch minimizes over all available candidates and eligible machines, the vehicle branch over all pool tasks and vehicles. The sampled action is never referenced, and $r_t^{\text{pass}}$ is evaluated *before* the environment transition, using $B(s_t)$ from the pre-transition state (it is refreshed only after the transition). Since $\eta_t = \eta(s_t)$, we have $r_t^{\text{pass}} = \varphi(s_t)$, an $\sigma(s_t)$-measurable function. *Cancellation.* Let $\mathcal{F}_t = \sigma(s_0, a_0, \ldots, s_t)$. Conditioning on $\mathcal{F}_t$, $r_t^{\text{pass}}$ and $s_t$ are fixed and $a_t \sim \pi_\theta(\cdot \mid s_t)$, so by the score identity $\mathbb{E}[\nabla_\theta \log \pi_\theta(a_t \mid s_t) \mid \mathcal{F}_t] = \nabla_\theta \sum_a \pi_\theta(a \mid s_t) = 0$; hence $\mathbb{E}[\nabla_\theta \log \pi_\theta(a_t \mid s_t)\, r_t^{\text{pass}}] = 0$. Summing over $t$ and subtracting from the base estimator gives the claim. □

*Remark* S6 (Scope of the lemma). The lemma is exact for the *pre-normalization* score estimator and uses only action-independence, not admissibility of $b_t$, nor that the pass floor is exactly the bound $B$ of the main text (it uses an admissible vehicle-availability *tightening* at $s_t^{\text{pass}}$, which changes only the control-variate variance). The deployed estimator also applies per-env advantage normalization *after* the subtraction and PPO clipping. These are standard, *arm-shared* approximations, identical in form for the base and BCB arms and present without BCB. BCB therefore introduces no bias beyond the base PPO estimator, while exact gradient equality is a statement about the un-normalized, unclipped score (verbatim at $\theta = \theta_{\text{old}}$, ratio 1).

*Remark* S7 (M2 shaped return, in full). Replacing the reward by $r_t - r_t^{\text{pass}}$, where $b_t$ denotes the admissible floor evaluated at the pass state $s_t^{\text{pass}}$ (the tightening named in the scope remark), yields shaped return $\kappa(B(s_0) - C_{\max}) + \kappa \sum_t (b_t - B(s_t))$. The correction is a *state-dependent, path-dependent* bonus $\rho(s_t) = \kappa(b_t - B(s_t)) \geq 0$ whose reference $s_t^{\text{pass}}$ is a virtual sibling of $s_t$, not the realized successor $s_{t+1}$. It is *not* a potential difference $\gamma\Phi(s_{t+1}) - \Phi(s_t)$, so it does not telescope to a boundary term and optimal-policy invariance does not follow. M2 is therefore reported as an empirical arm with no invariance claim, and it collapses (Table S-VI).

The pre-specified credit test asked whether BCB credit outperforms shared-reward training on the four scarce-fleet gate cells, pooled over 3 seeds. No cell is significant, and the ±2% TOST protocol shows all four cells equivalent (Holm-adjusted $p \leq 7.5 \times 10^{-14}$; widest 90% confidence interval $[+0.50\%, +0.98\%]$, positive favoring the shared reward). The lemma predicts this, because it guarantees only an unbiased subtraction. Writing the same credit into the return (arm M2) raises makespan by +19% to +50% on the gate cells at $\bar{\tau}/\bar{p} = 0.6$ (Table S-VI). With the price-channel test of the main text, these results separate attribution from decision guidance. Redistributing rewards after decisions does not shorten schedules, while distinguishing commitments before they are made does.

## XIV. Per-Cell Certificates

Fig. S-6 plots root-bound tightness across the three regime axes. Tightness declines as loaded utilization rises (Spearman $\rho = -0.97$ over the twelve primary-grid cells) and recovers as stations grow scarce and the machine component binds. Under processing-time dispersion, tightness declines only where the chain component is active, because the fleet component depends on travel times alone. In the figure $C$ is the per-instance best-known makespan over all evaluated arms. The station-scarcity ratios quoted with Table S-XXVI use the policy's own seed-mean makespan instead, so the two sets of values differ slightly by construction.

The certificate is valid on every instance, and its tightness depends on the active term. Over all 1,600 grid and transfer instances, $B(s_0)$ averages 87.6% of the best-known makespan, 90.3% on the instances where the chain term sets it and 76.0% on the 304 where the fleet term does. The best-known makespan is at least the optimum, so these ratios are lower limits on tightness. All 267 solver-closed instances are chain-bound, and on 204 of them the optimum equals $B(s_0)$, so the near-exact tightness of the main text is established only on chain-bound instances.

The fleet term admits a simple root refinement. Let $h$ be the smallest chain floor plus lag at which any job becomes ready for a mandatory move, and $t$ the shortest minimal processing time of an operation that follows such a move. Every loaded move starts at or after $h$, the vehicle with the most loaded work carries at least $W_{\text{tr}}/|V|$ of it, and its last delivery is followed by at least $t$ of processing. Hence $h + W_{\text{tr}}/|V| + t$ is an admissible root bound. BOLT does not use it. In the one-vehicle cells of Table S-X, it matches or exceeds the bound that the redundant constraint lets CP-SAT prove in 300 s on every instance.

TABLE S-XVI
Mean per-schedule certificate $(C - B(s_0))/B(s_0)$ on every evaluated cell (100 instances per cell; $C$ is the bound-guided policy's 3-seed-mean makespan), beside the reference solver's own proven gap, mean $(UB - LB)/LB$ of the warm-started 300-s CP-SAT run. Across all 1,600 instance-cell pairs with a CP-SAT upper bound, $B(s_0)$ never exceeded it, and the guide-price implementation check found no admissibility violation on 2,674 priced candidates. The solver closes 267 of 1,600 instances (nine cells contain no closed instance), so wide certificates track cells that are hard for exact search too; Section VI-D3 of the main text explains why the CP-SAT gap is the sharper quality measure where the solver's own gap is small.

| Cell ($\lvert V \rvert$, $\bar{\tau}/\bar{p}$) | Certificate | CP-SAT $(UB - LB)/LB$ |
|---|---|---|
| 1, 0.1 | 9.6% | 2.3% |
| 1, 0.3 | 30.9% | 17.9% |
| 1, 0.6 | 45.5% | 30.7% |
| 1, 1.0 | 39.1% | 29.3% |
| 2, 0.1 | 7.7% | 1.4% |
| 2, 0.3 | 14.1% | 3.7% |
| 2, 0.6 | 29.9% | 17.0% |
| 2, 1.0 | 56.1% | 42.8% |
| 3, 0.1 | 7.6% | 1.4% |
| 3, 0.3 | 11.8% | 1.6% |
| 3, 0.6 | 20.7% | 4.8% |
| 3, 1.0 | 36.0% | 17.8% |
| 4, 0.6 | 18.5% | 1.7% |
| 4, 1.0 | 30.2% | 7.9% |
| 15×25, 0.6 | 64.7% | 46.5% |
| 15×25, 1.0 | 48.8% | 38.9% |

## XV. Schedules of BOLT and the Transport-as-Penalty Anchor

Fig. S-7 compares the schedules of BOLT and the transport-as-penalty anchor on one test instance with one vehicle and

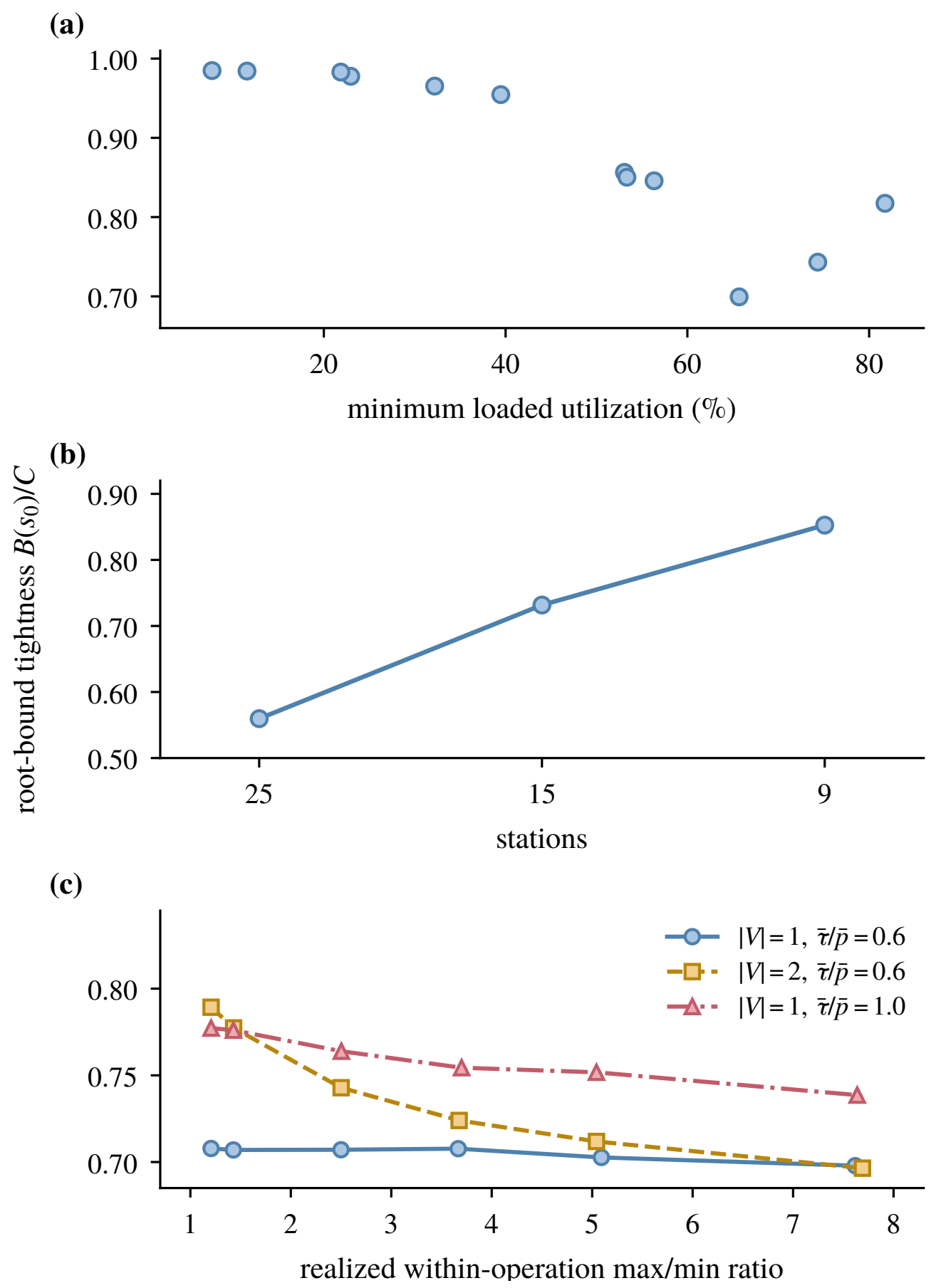


Fig. S-6. Root-bound tightness $B(s_0)/C$, with $C$ the per-instance best-known makespan over all evaluated arms, against (a) minimum loaded fleet utilization (one point per primary-grid cell), (b) station scarcity (50 modules, $|V| = 2$, $\bar{\tau}/\bar{p} = 0.1$), and (c) realized within-operation processing-time dispersion (the heterogeneity study).

$\bar{\tau}/\bar{p} = 0.6$. In both schedules, the vehicle moves almost without pause from its first move until about 390 h. After that point, BOLT keeps the vehicle moving until its makespan. Under the anchor, the vehicle makes only a few isolated moves, and the last operations spread out until the later makespan.

## XVI. Delay-Semantics Table

Table S-XVII summarizes the three delay classes of Definition 2 of the main text; the taxonomy figure and the coupling-regret map are in the main text (Figs. 1 and 3(a)).

TABLE S-XVII
Delay semantics (Definition 2 of the main text): what each class blocks and what sets its duration.

| Semantics | Blocks | Duration set by | Example |
|---|---|---|---|
| Machine-free | job | instance | curing lag [10] |
| Machine-side | machine | instance/choices | setup [14] |
| Resource-coupled | job & resource | schedule (fleet) | transport (ours) |

Machine-free and machine-side delays have instance-fixed durations; resource-coupled delays do not.

TABLE S-XVIII
Merged heads against per-class heads on the production batches, both arms trained on a 10/15/20-module mixture with the same fleet and travel grids, the same 2,000 updates, the same 3 seeds, and the same certified price channel. Per-instance seed means, $n = 20$ per cell, paired equivalence test at the pre-specified ±2% margin. Equivalent means the whole 90% interval lies inside the margin, inconclusive means the interval crosses it, and merged better means the interval lies above it.

| Cell | Per-class | Merged | Relative (90% CI) |
|---|---|---|---|
| 50, 2, 0.6 | 1061.1 | 1041.5 | +1.89 (+1.17, +2.67) |
| 50, 2, 1.0 | 1668.0 | 1627.1 | +2.54 (+1.59, +3.53) |
| 80, 3, 1.0 | 1890.8 | 1881.2 | +0.56 (−0.42, +1.50) |
| 50, 1, 0.3 | 1036.4 | 993.4 | +4.33 (+3.77, +4.91) |
| 80, 1, 0.3 | 1673.3 | 1602.2 | +4.44 (+3.93, +4.98) |
| 80, best-of-64 | 1803.9 | 1816.6 | −0.69 (−1.11, −0.29) |

Cell: modules, $|V|$, $\bar{\tau}/\bar{p}$ (machines: 25). Makespan in hours; relative difference in percent, positive favoring the merged head. Verdicts: rows 1 and 2 inconclusive, rows 3 and 6 equivalent, and rows 4 and 5, the one-vehicle cells of Table S-V, merged better (factorized W/T/L 0/0/20 in each). Under the rule fixed before the third seed finished training (Section XXV), a merged-head lead beyond the margin in any cell removes the parity claim. Both arms are evaluated on the same device, because a greedy rollout takes an argmax over logits and processor and accelerator kernels resolve near-ties differently, which moves a cell mean by up to about one percent.

## XVII. Production Batches: Merged Head, Penalty Anchor, and Vehicle-Rule Fallback

Table S-XVIII compares the two realizations on the production batches, Table S-XIX adds the penalty anchor, and Table S-XX replaces the learned vehicle decisions with a dispatching rule.

Both size-mixture policies are shorter than the penalty anchor of Fig. 3(a) of the main text on every instance of the three batches of Table III (Table S-XIX). All arms run on one processor thread with batch 10, so the comparison has no device effect.

TABLE S-XIX
Penalty anchor against the two size-mixture policies on the production batches of Table III in the main text, all evaluated on one CPU thread (makespan in hours, $n = 20$ per cell, policies: per-instance 3-seed means).

| | | Factorized | | Merged head | |
|---|---|---|---|---|---|
| Cell | Anchor | Makespan | Rel. (%) | Makespan | Rel. (%) |
| 50, 2, 0.6 | 1136.1 | 1052.2 | +8.0 | 1037.5 | +9.5 |
| 50, 2, 1.0 | 1803.1 | 1669.0 | +8.0 | 1630.8 | +10.6 |
| 80, 3, 1.0 | 2199.3 | 1883.3 | +16.8 | 1886.9 | +16.6 |

Cell: modules, $|V|$, $\bar{\tau}/\bar{p}$ (machines: 25). Anchor: one training seed, nearest-vehicle-first (NVF) rule for every vehicle decision. Rel.: (anchor − policy)/policy, positive when the policy is shorter. Both policies are shorter than the anchor on every instance (W/T/L 20/0/0, $p = 1.9 \times 10^{-6}$ in each comparison). Each training seed alone is shorter by 4.8–20.0% (factorized) and 7.5–19.4% (merged head), on at least 16 of the 20 instances. The CPU cell means differ from the GPU means of Table III and Table S-XVIII by at most 0.8% (single seeds: 1.3%).

Table S-XX replaces every vehicle-class decision of the two size-mixture policies with the nearest-vehicle-first (NVF) rule,

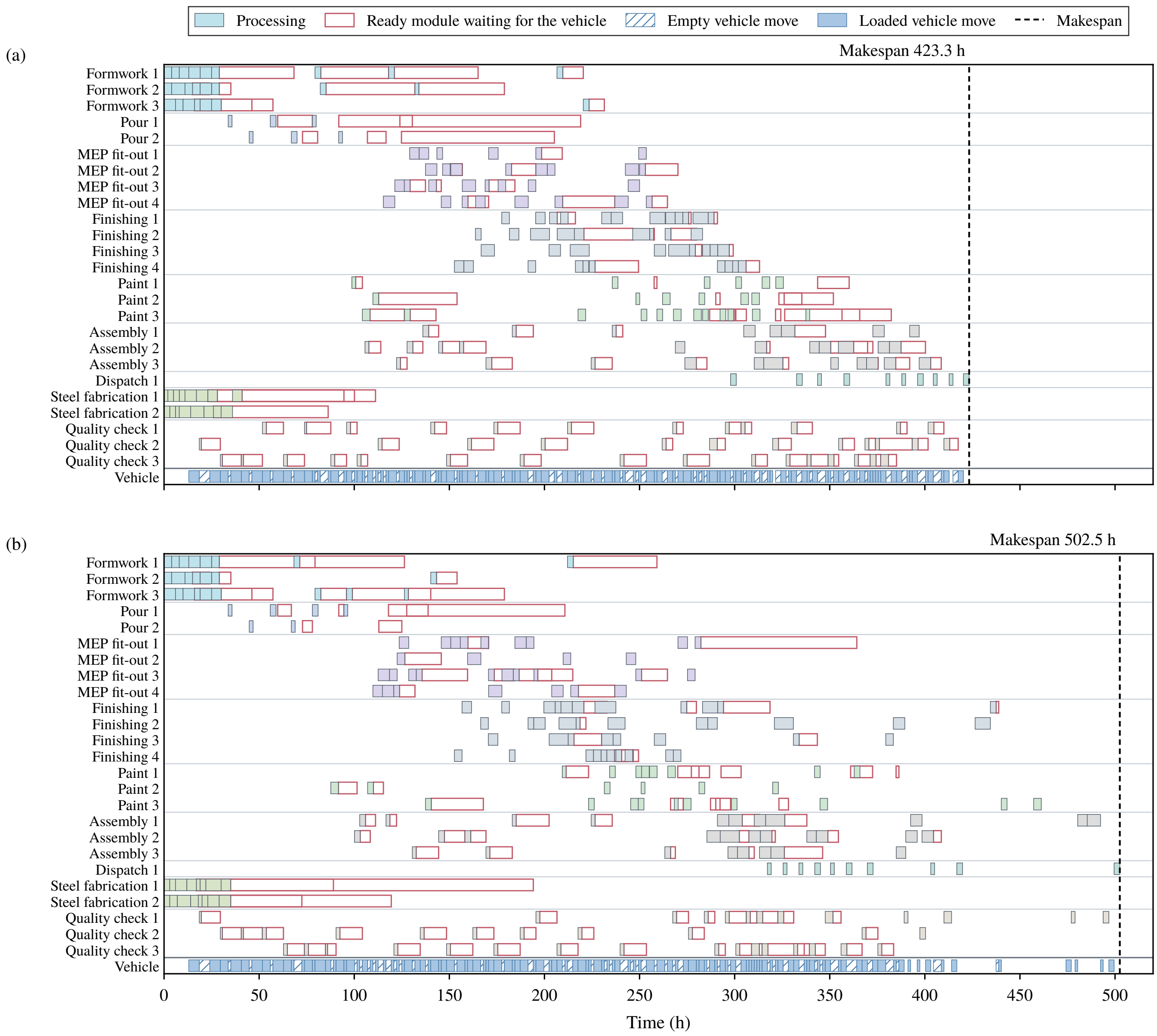


Fig. S-7. Schedules that (a) BOLT and (b) the transport-as-penalty anchor produce on test instance 3 of the $10 \times 25$ cell with one vehicle and $\bar{\tau}/\bar{p} = 0.6$. This instance has the largest makespan difference between the two policies among the first 20 test instances. Each row is one station or the vehicle. Filled bars mark processing. Outlined bars mark the time from when a module is ready to move, after its processing and time lag end, until the vehicle picks it up. On the vehicle row, hatched bars mark empty moves and solid bars mark loaded moves.

while the network keeps the machine-class decisions. Both fallbacks stay shorter than the best fixed pair in every cell, by 3.8–13.1%. Where the factorized policy degrades less than the merged head, its learned vehicle decisions were weaker to begin with. Under the same rule, the merged head's station decisions give equal or shorter schedules in every cell. The data therefore do not show that the factorized form tolerates a fallback better. With the shortest-travel-time rule as fallback, both realizations lose 11–20%, so the fallback rule must be chosen.

## XVIII. Recovery from a Machine Breakdown

After a breakdown, the policy's re-plan is the shortest schedule available until the GA has run for 33 s (Tables S-XXI and S-XXII). Warm-started CP-SAT returns the right-shifted plan unchanged at budgets of 1 and 10 s. At 60 s it shortens that plan by at most 0.8%, and only at 50 modules. The policy's re-plan therefore stays 11.0–19.0% shorter than the CP-SAT result at every budget. The GA outperforms the policy in every column, and the policy's re-plan is 4.8–9.5% longer than the GA's result. This gain costs time, because the GA needs 33–120 s of wall time while the policy needs 7 s at 50 modules and 14 s at 80 modules. In the two tested cells, the crossover from the policy's re-plan to the GA's shorter result therefore fell between 33 s and 120 s after the breakdown, depending on cell and budget.

TABLE S-XX
Size-mixture policies with learned vehicle decisions and with the nearest-vehicle-first (NVF) rule as fallback for every vehicle decision (makespan in hours, CPU, per-instance 3-seed means, $n = 20$ per cell).

| Cell | Policy | Learned | NVF rule | Change (%) | W/T/L | vs. pair (%) |
|---|---|---|---|---|---|---|
| 50, 2, 0.6 | Factorized | 1052.2 | 1052.1 | −0.01 | 9/0/11 | −5.5 |
| | Merged head | 1037.5 | 1051.1 | +1.31 | 3/0/17 | −5.6 |
| 50, 2, 1.0 | Factorized | 1669.0 | 1710.2 | +2.46 | 2/0/18 | −4.5 |
| | Merged head | 1630.8 | 1677.4 | +2.86 | 1/0/19 | −6.3 |
| 80, 3, 1.0 | Factorized | 1883.3 | 1907.6 | +1.29 | 5/0/15 | −12.2 |
| | Merged head | 1886.9 | 1889.8 | +0.15 | 10/0/10 | −13.1 |
| 50, 1, 0.3 | Factorized | 1035.8 | 1020.4 | −1.48 | 14/0/6 | −5.1 |
| | Merged head | 994.4 | 1018.2 | +2.39 | 2/0/18 | −5.3 |
| 80, 1, 0.3 | Factorized | 1671.7 | 1586.0 | −5.13 | 20/0/0 | −3.8 |
| | Merged head | 1605.1 | 1563.2 | −2.61 | 17/0/3 | −5.1 |

Cell: modules, $|V|$, $\bar{\tau}/\bar{p}$ (machines: 25). Learned: the network makes every decision. NVF rule: the network keeps the machine-class decisions and the nearest-vehicle-first rule makes every vehicle-class decision. Change: (NVF − learned)/learned. W/T/L: instances on which the fallback is shorter, equal or longer. vs. pair: (NVF − pair)/pair, with the best fixed pair SPT+NVF in every cell (means 1113.9, 1789.9, 2173.7, 1075.5 and 1647.9 h); the fallback is shorter on at least 16 of the 20 instances in every row. With STT as the fallback rule, the makespan rises by 11–20% in every row (W/T/L 0/0/20).

TABLE S-XXI
Recovery from a machine breakdown, arms that spend no search budget. The machine with the largest remaining workload fails at 30% of the baseline makespan and is down for 20% of it; the residual is re-planned. 3 seeds, $n = 20$ per setting, per-instance seed means. Both generators of the disrupted plan are reported: a dispatching schedule is the neutral case, and the policy's own schedule is what a plant running this policy faces. The search arms are in Table S-XXII.

| Cell | Disrupted plan | vs right shift | vs dispatching |
|---|---|---|---|
| 50, 2, 1.0 | policy | −11.7 (20/0/0) | −2.4 (20/0/0) |
| 50, 2, 1.0 | dispatching | −19.0 (20/0/0) | −3.2 (19/0/1) |
| 80, 3, 1.0 | policy | −11.7 (20/0/0) | −6.0 (20/0/0) |
| 80, 3, 1.0 | dispatching | −18.2 (20/0/0) | −3.0 (17/0/3) |

Cell: modules, $|V|$, $\bar{\tau}/\bar{p}$ (machines: 25). Entries are the policy's relative makespan against that arm, in percent, negative favoring the policy, with win/tie/loss over the 20 instances. Every comparison has $p \leq 2.7 \times 10^{-5}$. Right shift keeps every decision and delays what the breakdown invalidated; dispatching re-plans the residual with the nine rule pairs and takes the best.

## XIX. Execution Under Duration Noise

Under duration noise of up to ±30%, the policy stays ahead of rule re-dispatch and behind the search methods at every noise level, and it keeps this position without re-planning (Table S-XXIII). The study multiplies every processing time, machine-free lag, and loaded travel time by $1 + \varepsilon$, with $\varepsilon$ uniform on $[-\sigma, \sigma]$ and $\sigma \in \{0.1, 0.2, 0.3\}$. It uses the 20 instances of each of the two heavy production-batch cells of the breakdown study, with three noise replicates per instance and one set of draws shared by all arms. The policy is the size-mixture policy of training seed 301 only, so the study does not cover the spread across training seeds. Each arm executes its nominal schedule under the realized durations, either by right shift or with re-planning at fixed checkpoints.

Right-shift execution of a nominal schedule loses little makespan to this noise. Right shift keeps every assignment and sequence and starts each operation once its predecessors have finished under the realized durations. At $\sigma = 0.3$, it raises the policy's realized makespan by 1.7% at 50 modules and 2.5% at 80 modules over the noise-free execution. The 60-CPU-s GA schedule rises by 2.3% and 3.1%.

Re-planning adds no measurable gain to right shift. The re-planning arm re-plans the remaining work from the realized state at nine checkpoints, one after each further 10% of completed operations. Its realized makespan differs from right shift by at most 0.7% in either direction at every $\sigma > 0$, and no difference is significant (Holm-adjusted $p \geq 0.21$). Without noise, re-planning reproduces the nominal makespan on all 20 instances at 50 modules and shortens it by 1.3% at 80 modules (Holm-adjusted $p = 0.04$). A threshold arm re-plans only when a finished operation is late by more than 5% of the nominal makespan. It never triggered in 480 runs, because the largest lateness reached 83% of that limit.

The policy's place among the re-planning methods holds at every noise level. Its re-plan is 4.1–6.4% shorter than best-of-nine rule re-dispatch at $\sigma > 0$, and shorter on all 20 instances of both cells at $\sigma \geq 0.2$. The GA re-plan, with 10 CPU-s of search per re-plan, is shorter than the policy's re-plan on every instance, by 5.8–7.2% at $\sigma > 0$. The CP-SAT re-plan, run on one replicate at $\sigma = 0$ and 0.2, is shorter by 3.3–6.6%. Both search arms start from the 60-CPU-s GA schedule and keep the current plan when it is shorter than their result. They also spend more planning time. One re-plan takes 4.17 and 9.14 CPU-s for the policy at 50 and 80 modules, against 21.6 and 37.7 CPU-s for the GA and 18.0 and 28.7 CPU-s for CP-SAT. These times include building the residual problem and the warm start, so the GA's time exceeds its search budget. Duration noise of this size therefore requires no re-planning, and re-planning serves discrete disruptions such as the machine breakdown of Tables S-XXI and S-XXII.

## XX. Merged-Head Control on the Regime Grid

Table S-XXVII reports the factorized-versus-merged comparison of Section VI-D1 of the main text on the nine primary grid cells. The light-travel column and the transfer cells carry both arms in Table S-II; the per-cell verdicts there are EQUIVALENT except the 15×25, $\bar{\tau}/\bar{p} = 1.0$ cell, where the factorized policy is better. The largest absolute relative difference among all equivalent evaluation cells is 0.37%.

## XXI. Processing-Time Heterogeneity Study

Section VI-C4 of the main text raises within-operation processing-time dispersion while holding each operation's mean processing time fixed. For every operation, per-machine factors are drawn independently and log-uniformly on $[1/\sqrt{R}, \sqrt{R}]$. They are then renormalized to a mean of 1 over the operation's eligible machine set, so only the spread across eligible machines changes with the dispersion level $R$. The transform is applied to the same base test instances used throughout the paper, so every arm-vs-arm comparison below is paired at the instance level. Routing, lag durations, the travel-time matrix, the fleet, and the eligibility pattern

TABLE S-XXII

RECOVERY FROM A MACHINE BREAKDOWN WITH THE SEARCH ARMS ADDED, BY RE-PLANNING BUDGET. BOLD: SHORTEST MAKESPAN IN EACH COLUMN.

| | | Cell 50, 2, 1.0 | | | | Cell 80, 3, 1.0 | | | |
|---|---|---|---|---|---|---|---|---|---|
| | | Dispatching plan | | Policy's plan | | Dispatching plan | | Policy's plan | |
| Arm | Budget (s) | Makespan | Wall (s) | Makespan | Wall (s) | Makespan | Wall (s) | Makespan | Wall (s) |
| Right shift | – | 2104.5±49.9 | <0.1 | 1911.0±39.9 | <0.1 | 2475.9±58.1 | <0.1 | 2177.6±49.9 | <0.1 |
| Rule re-dispatch | – | 1761.7±25.3 | 6.0 | 1729.3±26.6 | 6.2 | 2088.8±51.6 | 15.5 | 2045.7±23.7 | 15.0 |
| Policy re-plan | – | 1705.5±18.6 | 7.0 | 1688.0±25.9 | 7.2 | 2026.2±28.4 | 14.1 | 1923.0±19.9 | 14.3 |
| CP-SAT | 1 | 2104.5±49.9 | 9.2 | 1911.0±39.9 | 9.1 | 2475.9±58.1 | 29.4 | 2177.6±49.9 | 31.3 |
| CP-SAT | 10 | 2104.5±49.9 | 18.3 | 1911.0±39.9 | 18.2 | 2475.9±58.1 | 39.5 | 2177.6±49.9 | 40.6 |
| CP-SAT | 60 | 2088.5±49.4 | 69.1 | 1897.3±40.6 | 68.9 | 2475.9±58.1 | 86.5 | 2177.6±49.9 | 91.9 |
| GA | 10 | 1626.8±21.3 | 35.7 | 1601.1±14.2 | 33.3 | 1859.4±28.1 | 70.0 | 1813.1±13.4 | 68.9 |
| GA | 60 | **1601.3±19.1** | 83.9 | **1569.7±14.2** | 82.7 | **1850.1±27.3** | 119.9 | **1802.3±12.4** | 117.9 |

Cell: modules, $|V|$, $\bar{\tau}/\bar{p}$ (machines: 25). Makespan in hours, mean ± standard deviation over the 20 instances. A per-instance value is the mean over three seeds for the policy and for every arm on the policy's plan. Wall: mean wall time per instance. CP-SAT is warm-started from the right-shifted plan and uses four workers. Both search arms exceed their budgets because work precedes the search clock. CP-SAT builds the residual model first, and the GA builds the nine rule-seeded chromosomes and evaluates its first population of 100, then finishes the generation in progress when the budget expires. The GA therefore exceeds its budget by 23–26 s at 50 modules and by 58–60 s at 80 modules. The rule and policy walls were measured while other jobs ran on the machine, so they are upper bounds. The policy outperforms CP-SAT at every budget in every column (win/tie/loss 20/0/0), and the GA outperforms the policy at every budget in every column (0/0/20); all these comparisons have $p \le 1.9 \times 10^{-6}$.

TABLE S-XXIII

REALIZED MAKESPAN UNDER DURATION NOISE, BY NOISE LEVEL $\sigma$. BOLD: SHORTEST MAKESPAN IN EACH COLUMN.

| | | Cell 50, 2, 1.0 | | | | Cell 80, 3, 1.0 | | | |
|---|---|---|---|---|---|---|---|---|---|
| Initial plan | Online response | $\sigma = 0$ | $\sigma = 0.1$ | $\sigma = 0.2$ | $\sigma = 0.3$ | $\sigma = 0$ | $\sigma = 0.1$ | $\sigma = 0.2$ | $\sigma = 0.3$ |
| Policy | Right shift | 1675.5±50.3 | 1680.9±50.0 | 1690.7±50.1 | 1704.8±51.1 | 1912.5±45.2 | 1920.7±44.8 | 1937.1±44.2 | 1960.4±43.3 |
| | Re-plan | 1675.5±50.3 | 1673.4±40.0 | 1679.5±28.4 | 1692.6±28.3 | 1887.5±37.4 | 1910.8±23.8 | 1951.0±28.7 | 1968.6±35.7 |
| | Re-plan when late | 1675.5±50.3 | 1680.9±50.0 | 1690.7±50.1 | 1704.8±51.1 | 1912.5±45.2 | 1920.7±44.8 | 1937.1±44.2 | 1960.4±43.3 |
| Best of nine rules | Rule re-dispatch | 1728.0±31.9 | 1744.2±25.0 | 1755.5±20.7 | 1772.6±21.9 | 2006.7±41.5 | 2040.8±27.4 | 2059.5±37.3 | 2072.0±23.3 |
| GA, 60 CPU-s | Right shift | 1587.1±15.5 | 1592.8±15.6 | 1605.4±16.0 | 1624.0±17.5 | 1827.4±18.4 | 1836.3±18.1 | 1856.3±18.2 | 1884.8±19.4 |
| | GA re-plan | **1553.5±16.1** | **1560.6±15.0** | **1570.9±16.4** | **1590.3±17.4** | **1790.6±18.1** | **1806.0±20.8** | **1822.2±18.0** | **1850.7±20.0** |
| | CP-SAT re-plan | 1571.1±15.1 | – | 1593.0±20.7 | – | 1827.4±18.4 | – | 1856.6±20.2 | – |
| Policy re-plan against each rival: win/tie/loss and Holm-adjusted $p$ | | | | | | | | | |
| vs policy right shift | W/T/L | 0/20/0 | 11/0/9 | 12/0/8 | 11/0/9 | 15/0/5 | 11/0/9 | 6/0/14 | 7/0/13 |
| | $p$ | 1.00 | 1.00 | 1.00 | 1.00 | 0.04 | 0.43 | 0.21 | 0.38 |
| vs rule re-dispatch | W/T/L | 15/0/5 | 18/0/2 | 20/0/0 | 20/0/0 | 20/0/0 | 20/0/0 | 20/0/0 | 20/0/0 |
| | $p$ | $2.0 \times 10^{-3}$ | $3.8 \times 10^{-5}$ | $7.6 \times 10^{-6}$ | $7.6 \times 10^{-6}$ | $7.6 \times 10^{-6}$ | $7.6 \times 10^{-6}$ | $7.6 \times 10^{-6}$ | $7.6 \times 10^{-6}$ |
| vs GA re-plan | W/T/L | 0/0/20 | 0/0/20 | 0/0/20 | 0/0/20 | 0/0/20 | 0/0/20 | 0/0/20 | 0/0/20 |
| | $p$ | $7.6 \times 10^{-6}$ | $7.6 \times 10^{-6}$ | $7.6 \times 10^{-6}$ | $7.6 \times 10^{-6}$ | $7.6 \times 10^{-6}$ | $7.6 \times 10^{-6}$ | $7.6 \times 10^{-6}$ | $7.6 \times 10^{-6}$ |

Cell: modules, $|V|$, $\bar{\tau}/\bar{p}$ (machines: 25). Makespan in hours, mean ± standard deviation over the 20 instances. An instance's value is the mean over three noise replicates. The policy is the size-mixture policy of training seed 301. Each processing time, machine-free lag and loaded travel time is multiplied by $1 + \varepsilon$, with $\varepsilon$ uniform on $[-\sigma, \sigma]$. The draw is made once per instance and replicate and is shared by every arm. A re-planning arm re-plans each time a further 10% of the operations has completed, which gives nine checkpoints. The threshold arm re-plans at a checkpoint only when the lateness, the largest delay of a finished operation against the plan in force, exceeds 5% of the nominal makespan, and it never re-planned in 480 runs. At each re-plan, an operation already committed to a machine keeps that machine. The GA re-plans with 10 s of CPU time, and CP-SAT with 10 s of wall time on three workers. Both keep the current plan when it is shorter. Every re-planning arm re-planned at all nine checkpoints of every run. The planning CPU time per re-plan at 50/80 modules is 4.17/9.14 s for the policy, 2.79/7.20 s for the rules, 21.57/37.73 s for the GA, and 18.02/28.71 s for CP-SAT (14.5/27.5 s wall). The CP-SAT arm ran at $\sigma = 0$ and 0.2 only, on one replicate. All arms ran on the same cores, cores 2–7 of a Core Ultra 9 285K at 5.5 GHz. The lower panel counts the instances on which the policy re-plan is shorter (W), equal within $10^{-6}$ h (T) or longer (L). Its $p$ is the two-sided paired Wilcoxon test over the 20 instances, Holm-adjusted across the four noise levels of each cell and rival.

$p(o, m) > 0$ are left untouched. The transformed instances are verified field by field against the base instances (operation types, station assignments, lags, machine types, routing classes, module weights, and the full transport layout: station cells, cell coordinates, the speed constant, the travel-time matrix, fleet size, and vehicle start cell).

Levels $R \in \{2, 5, 10, 20, 100\}$ are evaluated together with the untouched baseline, on the three cells carried in the main text: $|V| = 1$, $\bar{\tau}/\bar{p} = 0.6$ (scarce fleet, inside the training mixture); $|V| = 2$, $\bar{\tau}/\bar{p} = 0.6$ (moderate travel, inside the training mixture); and $|V| = 1$, $\bar{\tau}/\bar{p} = 1.0$ (held out from training). Every cell keeps $n = 100$ paired test instances. The arms are the bound-guided policy and the shared-reward-only arm, each the 3-seed, per-instance mean used throughout the paper, together with the nine dispatching-rule pairs crossing the machine rules shortest processing time (SPT), most work

TABLE S-XXIV

BEST-OF-64 DECODING AGAINST THE GENETIC ALGORITHM AT A LONGER SEARCH BUDGET. THE SAMPLED DECODE NEEDS ABOUT 180 S PER INSTANCE AT 50 MODULES, SO THE 60-S COMPARISON OF TABLE S-IV GRANTS THE SEARCH FAR LESS TIME THAN THE POLICY SPENDS. GIVEN 600 S THE GENETIC ALGORITHM IS AHEAD ON EVERY INSTANCE. THE SAMPLED ARM'S ADVANTAGE THEREFORE APPLIES TO THE ONE-MINUTE BUDGET AND DOES NOT EXTEND TO THE SEARCH METHOD ITSELF. $n$ = 20, THREE SEEDS FOR THE POLICY, SAME INSTANCES THROUGHOUT. BOLD: SHORTEST IN THE ROW.

| Cell | Best-of-64 | GA, 600 s | Relative |
|---|---|---|---|
| 50, 2, 0.6 | 995.1±9.1 | **959.2±9.9** | +3.7 (0/0/20) |
| 50, 2, 1.0 | 1568.4±8.4 | **1530.6±13.1** | +2.5 (0/0/20) |

Cell: modules, $|V|$, $\bar{\tau}/\bar{p}$ (machines: 25). Makespan in hours. Relative: the policy's difference from the genetic algorithm in percent, positive favoring the genetic algorithm, with win/tie/loss over the instances. The 600-s runs consume about 640 CPU-s per instance. No 600-s run was made at 80 modules.

TABLE S-XXV

WHICH COMPONENT OF $B = \max(B_{\text{chain}}, B_{\text{mch}}, B_{\text{veh}})$ IS ACTIVE AT THE ROOT STATE, READ OFF THE INSTANCES WITHOUT SOLVING. THE REPORTED GRID IS VEHICLE-BOUND WHEREVER THE FLEET IS SCARCE AND CHAIN-BOUND WHERE IT IS AMPLE; THE MACHINE COMPONENT NEVER WINS THERE, WHICH IS WHY THE LAST THREE ROWS EXIST. THEY CUT THE FACTORY FROM 25 STATIONS TO 15 AND 9 WHILE HOLDING THE FLEET AND THE TRAVEL INTENSITY INSIDE THE TRAINING MIXTURE, SO THE MACHINE COMPONENT BINDS ON EVERY INSTANCE AND THE ONLY QUANTITY OUTSIDE TRAINING IS THE STATION COUNT. BOLD: THE COMPONENT ATTAINING THE MAXIMUM.

| Cell | $M$ | $B_{\text{chain}}$ | $B_{\text{mch}}$ | $B_{\text{veh}}$ | Active |
|---|---|---|---|---|---|
| 10, 1, 0.6 | 25 | 234.1 | 51.2 | **290.8** | vehicle |
| 10, 3, 0.6 | 25 | **234.1** | 51.2 | 96.9 | chain |
| 50, 2, 1.0 | 25 | 269.0 | 245.7 | **1211.3** | vehicle |
| 80, 3, 1.0 | 25 | 273.5 | 392.3 | **1287.9** | vehicle |
| 50, 2, 0.1 | 25 | 219.4 | **243.9** | 120.4 | machine |
| 50, 2, 0.1 | 15 | 221.4 | **522.3** | 120.7 | machine |
| 50, 2, 0.1 | 9 | 228.0 | **1118.8** | 120.6 | machine |

Cell: modules, $|V|$, $\bar{\tau}/\bar{p}$. $M$: stations. Values are means over the instances of the cell ($n$ = 20 in the two production-batch rows, $n$ = 30 elsewhere), in hours; the bold component attains the maximum on every instance of its row. Over all 14 evaluated cells (390 instances, the seven rows shown and seven further cells) the chain, machine, and fleet components attain the maximum on 60, 90, and 240 instances.

TABLE S-XXVI

SCHEDULING QUALITY WHERE THE MACHINE COMPONENT BINDS. THE FACTORY IS CUT FROM 25 STATIONS TO 15 AND 9 AT 50 MODULES, WITH THE FLEET AT TWO VEHICLES AND $\bar{\tau}/\bar{p}$ = 0.1, BOTH INSIDE THE TRAINING MIXTURE, SO THE STATION COUNT IS THE ONLY QUANTITY THE POLICY NEVER TRAINED ON. THREE SEEDS, PER-INSTANCE MEANS, $n$ = 30 ON AN INSTANCE BLOCK SEPARATE FROM THE PRODUCTION BATCHES OF TABLE III. BOLD: SHORTEST IN THE ROW.

| $M$ | Best PDR | GA 60 s | BOLT | vs. GA |
|---|---|---|---|---|
| 25 | 473.7±10.8 | 467.2±8.6 | **452.5±6.8** | −3.1 (29/0/1) |
| 15 | 742.5±13.6 | 740.3±13.4 | **728.7±12.0** | −1.6 (27/0/3) |
| 9 | 1352.6±26.7 | **1329.3±23.8** | 1330.3±24.5 | +0.1 (12/0/18) |

Makespan in hours. The policy outperforms the best dispatching pair at every station count, by −1.6 to −4.5%, and outperforms the 60-CPU-s genetic algorithm at 25 and 15 stations, which one-shot decoding does not do in the vehicle-bound production-batch cells (Table S-IV). At 9 stations the two are not distinguishable ($p$ = 0.40). Margins narrow as stations grow scarcer because every method is pressed toward the machine component: the ratio of makespan to bound falls from 1.86 to 1.19.

TABLE S-XXVII

MERGED-HEAD CONTROL AGAINST FACTORIZED BOLT ON THE PRIMARY REGIME GRID (MAKESPAN IN HOURS, MEAN±STD, $n$ = 100 PER CELL, PER-INSTANCE 3-SEED MEANS). RELATIVE DIFFERENCE IS (FACTORIZED − MERGED)/MERGED, NEGATIVE FAVORING THE FACTORIZED POLICY, WITH ITS 90% BOOTSTRAP CONFIDENCE INTERVAL; THE TOST VERDICT USES THE PRE-SPECIFIED ±2% MARGIN. BOLD: SHORTEST MAKESPAN IN THE ROW.

| $\|V\|$ | $\bar{\tau}/\bar{p}$ | Merged | Factorized | Rel. diff. | 90% CI | TOST |
|---|---|---|---|---|---|---|
| 1 | 0.3 | 285.1±7.5 | **284.6±7.6** | −0.16% | [−0.42, +0.10] | equiv. |
| 1 | 0.6 | 423.7±10.0 | **422.1±9.4** | −0.37% | [−0.72, −0.03] | equiv. |
| 1 | 1.0 | 673.7±23.5 | **672.1±17.7** | −0.13% | [−0.77, +0.50] | equiv. |
| 2 | 0.3 | **247.8±8.2** | 248.1±8.5 | +0.14% | [−0.04, +0.34] | equiv. |
| 2 | 0.6 | 304.7±8.3 | **304.4±7.9** | −0.08% | [−0.37, +0.23] | equiv. |
| 2 | 1.0 | **400.3±10.4** | 401.4±8.5 | +0.32% | [−0.02, +0.66] | equiv. |
| 3 | 0.3 | 243.3±8.2 | **243.0±8.1** | −0.09% | [−0.27, +0.09] | equiv. |
| 3 | 0.6 | 282.9±8.5 | **282.7±8.3** | −0.04% | [−0.28, +0.21] | equiv. |
| 3 | 1.0 | 350.0±9.8 | **349.2±9.9** | −0.24% | [−0.49, +0.01] | equiv. |

remaining (MWKR), and first in first out (FIFO) with the vehicle rules shortest travel time (STT), nearest vehicle first (NVF), and FIFO. The study reports makespans, bound values, and price statistics only. No timing claim is made for it.

The realized within-operation max/min ratio (mean over the eligible set) rises from 1.21 at the untouched baseline to 1.43, 2.50, 3.67, and 5.09 at $R = 2, 5, 10, 20$, and to a mean of 7.6 (median 6.0, 95th percentile 18) at the $R$ = 100 stress level. The three cells agree to within 0.07 of each other at every level (per-cell values are in Tables S-XXVIII–S-XXX). At that stress level the one-hour integer floor on processing times raises total mean work by about 2.7%, against under 0.7% at $R \leq 20$. Because the comparison is paired on identical instances at every level, this drift does not affect any arm-vs-arm result. It only qualifies the claim that the operation mean is held exactly fixed at the stress level.

Table S-XXXI decomposes the root bound $B(s_0) = \max(B_{\text{chain}}, B_{\text{mch}}, B_{\text{veh}})$ at the baseline and the stress level. The fleet term $B_{\text{veh}}$ is exactly invariant under the transform, because it depends only on travel times, at 290.00, 145.00, and 483.10 h on the three cells. The chain and machine terms fall as the eligible-minimum optimism in $B_{\text{chain}}$ becomes a looser estimate of the realized processing time. The active term is unchanged by dispersion: the vehicle term attains the bound on all 100 instances of both $|V| = 1$ cells, at every level, and the chain term attains it on all 100 instances of the $|V| = 2$ cell.

On the held-out heavy-travel cell ($|V| = 1$, $\bar{\tau}/\bar{p} = 1.0$), the three-seed baseline in Table S-XXX is the only level, on any cell, at which the bound-guided arm is significantly behind the shared-reward arm (−2.08%, $p < 10^{-4}$). $R$ = 2 on the same cell also has a negative point estimate but is not distinguishable from zero ($p$ = 0.213). The baseline loss is driven by a single checkpoint: bound-guided seed 302 evaluates to a mean of 737.7 h on this cell, against 649.9 and 628.2 h for seeds 301 and 303, while the shared-reward arm's three seeds sit close together (669.9, 676.2, 628.5 h; results/hetero/summary.json per-seed means). Averaging only seeds 301 and 303 for both arms, the bound-guided arm is ahead at every level, from +1.57% at the baseline to +1.54% at $R$ = 20, essentially flat with no trend over dispersion. The advantage over shared-

TABLE S-XXVIII

PROCESSING-TIME HETEROGENEITY, CELL $|V| = 1$, $\bar{\tau}/\bar{p} = 0.6$ (SCARCE FLEET, TRAINING MIXTURE; MAKESPAN IN HOURS, $n = 100$ PAIRED INSTANCES, 3-SEED MEANS FOR GUIDE AND SHARED; PROVENANCE RESULTS/HETERO/ IN THE RELEASED REPOSITORY). RATIO: REALIZED WITHIN-OPERATION MAX/MIN, MEAN OVER THE ELIGIBLE SET. GUIDE, SHARED: BOUND-GUIDED AND SHARED-REWARD-ONLY ARMS, MEAN±STD. GUIDE VS. SHARED: RELATIVE IMPROVEMENT OF THE BOUND-GUIDED ARM AND WIN/TIE/LOSS. BEST PAIR: DISPATCHING PAIR WITH THE LOWEST CELL MEAN, AND THAT MEAN. GUIDE VS. PAIR: RELATIVE IMPROVEMENT OVER THE BEST PAIR. TIGHT.: $B(s_0)/C_{\text{best}}$. †: NOT SIGNIFICANT AT $\alpha = 0.05$ (TWO-SIDED WILCOXON); EXACT $p$ IN THE TABLE NOTE.

| $R$ | Ratio | Guide (h) | Shared (h) | Guide vs. shared | Best pair (h) | Guide vs. pair | $B(s_0)$ | Tight. | Price>0 | +Price (h) |
|---|---|---|---|---|---|---|---|---|---|---|
| base | 1.21 | 422.22±9.57 | 430.65±8.52 | +1.96%<br>90/0/10 | MWKR+NVF<br>449.75 | +6.12% | 290.00 | 0.7078 | 0.277 | 2.92 |
| 2 | 1.43 | 424.38±8.85 | 433.17±9.83 | +2.03%<br>89/0/11 | SPT+NVF<br>451.13 | +5.93% | 290.00 | 0.7070 | 0.289 | 2.80 |
| 5 | 2.50 | 427.15±8.70 | 440.01±9.63 | +2.92%<br>89/0/11 | SPT+NVF<br>441.96 | +3.35% | 290.00 | 0.7071 | 0.305 | 3.10 |
| 10 | 3.67 | 430.69±10.15 | 445.52±9.99 | +3.33%<br>92/0/8 | SPT+NVF<br>437.26 | +1.50% | 290.00 | 0.7077 | 0.309 | 3.00 |
| 20 | 5.09 | 436.67±9.54 | 452.68±11.47 | +3.54%<br>94/0/6 | MWKR+NVF<br>436.99 | +0.07%$^{\dagger}$ | 290.00 | 0.7027 | 0.315 | 3.33 |
| 100 (stress) | 7.62 | 448.66±9.91 | 461.09±9.09 | +2.70%<br>89/0/11 | MWKR+NVF<br>433.13 | −3.59% | 290.00 | 0.6979 | 0.312 | 3.10 |

$^{\dagger}$ Guide vs. pair at $R = 20$: $p = 0.613$ (tie).

TABLE S-XXIX

PROCESSING-TIME HETEROGENEITY, CELL $|V| = 2$, $\bar{\tau}/\bar{p} = 0.6$ (MODERATE TRAVEL, TRAINING MIXTURE; SAME PROTOCOL AND COLUMN DEFINITIONS AS TABLE S-XXVIII).

| $R$ | Ratio | Guide (h) | Shared (h) | Guide vs. shared | Best pair (h) | Guide vs. pair | $B(s_0)$ | Tight. | Price>0 | +Price (h) |
|---|---|---|---|---|---|---|---|---|---|---|
| base | 1.21 | 304.34±8.11 | 306.01±8.23 | +0.55%<br>63/0/37 | SPT+NVF<br>309.25 | +1.59% | 234.58 | 0.7893 | 0.067 | 2.91 |
| 2 | 1.43 | 301.63±7.83 | 303.49±8.25 | +0.61%<br>62/0/38 | SPT+NVF<br>304.53 | +0.95% | 228.53 | 0.7773 | 0.073 | 2.71 |
| 5 | 2.50 | 292.33±7.87 | 295.54±8.45 | +1.09%<br>75/0/25 | SPT+NVF<br>290.42 | −0.66%$^{\dagger}$ | 210.26 | 0.7429 | 0.088 | 3.44 |
| 10 | 3.68 | 286.52±7.81 | 291.47±8.97 | +1.70%<br>79/0/21 | SPT+NVF<br>280.91 | −1.99% | 199.39 | 0.7240 | 0.095 | 4.02 |
| 20 | 5.04 | 282.42±7.80 | 288.63±8.22 | +2.15%<br>80/0/20 | SPT+NVF<br>274.58 | −2.85% | 192.44 | 0.7119 | 0.095 | 4.11 |
| 100 (stress) | 7.69 | 277.46±7.43 | 284.29±8.89 | +2.40%<br>84/0/16 | SPT+NVF<br>266.90 | −3.96% | 182.96 | 0.6965 | 0.101 | 4.49 |

$^{\dagger}$ Guide vs. pair at $R = 5$: $p = 0.0512$ (tie).

TABLE S-XXX

PROCESSING-TIME HETEROGENEITY, CELL $|V| = 1$, $\bar{\tau}/\bar{p} = 1.0$ (HELD-OUT HEAVY TRAVEL, ZERO-SHOT; SAME PROTOCOL AND COLUMN DEFINITIONS AS TABLE S-XXVIII).

| $R$ | Ratio | Guide (h) | Shared (h) | Guide vs. shared | Best pair (h) | Guide vs. pair | $B(s_0)$ | Tight. | Price>0 | +Price (h) |
|---|---|---|---|---|---|---|---|---|---|---|
| base | 1.21 | 671.93±17.10 | 658.22±13.17 | −2.08%<br>20/0/80 | SPT+NVF<br>726.06 | +7.46% | 483.10 | 0.7772 | 0.339 | 5.75 |
| 2 | 1.43 | 667.13±15.81 | 664.90±14.01 | −0.34%$^{\dagger}$<br>47/0/53 | SPT+NVF<br>721.25 | +7.50% | 483.10 | 0.7760 | 0.348 | 5.06 |
| 5 | 2.50 | 670.90±18.36 | 678.23±13.84 | +1.08%<br>62/0/38 | SPT+NVF<br>714.82 | +6.15% | 483.10 | 0.7639 | 0.356 | 5.20 |
| 10 | 3.70 | 676.01±19.89 | 686.01±14.37 | +1.46%<br>66/0/34 | SPT+NVF<br>707.61 | +4.47% | 483.10 | 0.7544 | 0.359 | 4.93 |
| 20 | 5.04 | 684.11±18.79 | 691.64±12.45 | +1.09%<br>60/0/40 | SPT+NVF<br>700.86 | +2.39% | 483.10 | 0.7518 | 0.359 | 5.01 |
| 100 (stress) | 7.64 | 699.49±18.50 | 705.39±15.05 | +0.84%<br>56/0/44 | MWKR+NVF<br>702.18 | +0.38%$^{\dagger}$ | 483.10 | 0.7386 | 0.352 | 5.16 |

$^{\dagger}$ Guide vs. shared at $R = 2$: $p = 0.213$ (tie). Guide vs. pair at $R = 100$: $p = 0.492$ (tie).

reward training therefore survives heterogeneity. The results do not show that heterogeneity creates the advantage. The three-seed baseline loss is concentrated in one stored checkpoint, which indicates seed sensitivity. It does not indicate a systematic dispersion effect. The same evaluation code reproduces the paper's stored baseline evaluation for that checkpoint to within 0.05% (largest discrepancy over all 18 seed-cell pairs checked: 0.11%; results/hetero/baseline_crosscheck.json).

On the two training cells, excluding the same seed leaves the trend intact. At $|V| = 1$, $\bar{\tau}/\bar{p} = 0.6$, the two-seed advantage rises from +1.99% at the baseline to +3.11% at $R = 20$ (three-seed: +1.96% to +3.54%). At $|V| = 2$, $\bar{\tau}/\bar{p} = 0.6$, it rises from +0.21% to +1.60% over the same range (three-seed: +0.55% to +2.15%). Direction and approximate size of the trend match the full three-seed result on both cells.

Against fixed dispatching rules the picture depends on the cell (Tables S-XXVIII–S-XXX). On $|V| = 1$, $\bar{\tau}/\bar{p} = 0.6$, the margin over the best pair shrinks from +6.12% at the baseline

TABLE S-XXXI
ROOT-BOUND COMPONENTS AT THE BASELINE AND THE $R = 100$ STRESS LEVEL, THE THREE HETEROGENEITY CELLS (MEANS IN HOURS, $n = 100$; RESULTS/HETERO/ IN THE RELEASED REPOSITORY). ACTIVE: THE TERM ATTAINING $B(s_0)$, ON ALL 100 INSTANCES AT EVERY LEVEL AND CELL.

| Cell | $R$ | $B_{chain}$ | $B_{mch}$ | $B_{veh}$ | Active |
|---|---|---|---|---|---|
| $\|V\|=1$, $\bar{\tau}/\bar{p}=0.6$ | base | 234.58 | 51.39 | 290.00 | veh |
| $\|V\|=1$, $\bar{\tau}/\bar{p}=0.6$ | 100 | 182.96 | 26.34 | 290.00 | veh |
| $\|V\|=2$, $\bar{\tau}/\bar{p}=0.6$ | base | 234.58 | 51.39 | 145.00 | chain |
| $\|V\|=2$, $\bar{\tau}/\bar{p}=0.6$ | 100 | 182.96 | 26.72 | 145.00 | chain |
| $\|V\|=1$, $\bar{\tau}/\bar{p}=1.0$ | base | 257.12 | 51.39 | 483.10 | veh |
| $\|V\|=1$, $\bar{\tau}/\bar{p}=1.0$ | 100 | 204.88 | 26.64 | 483.10 | veh |

TABLE S-XXXII
WHAT EACH BOUND COMPONENT CAPTURES.

| | Captures | Relaxes | Binds when |
|---|---|---|---|
| $B_{chain}$ | precedence, minimum processing, lags, optimistic travel envelope | aggregate resource contention | chains and waits dominate |
| $B_{mch}$ | committed busy time and remaining minimal work per machine group | order of jobs on each machine | stations are scarce |
| $B_{veh}$ | mandatory loaded transport workload and earliest fleet availability | vehicle assignment and sequencing | fleet is scarce or highly utilized |

to a tie at $R = 20$ (+0.07%, $p = 0.61$) and reverses to −3.59% ($p < 10^{-4}$) at the stress level. On $|V| = 2$, $\bar{\tau}/\bar{p} = 0.6$, the margin falls faster: still positive at $R = 2$ (+0.95%), it turns negative by $R = 5$ (−0.66%, a tie at $p = 0.05$) and widens to −3.96% ($p < 10^{-4}$) at the stress level. On the held-out heavy-travel cell, $|V| = 1$, $\bar{\tau}/\bar{p} = 1.0$, the bound-guided arm keeps winning through $R = 20$ (from +7.46% at the baseline to +2.39%), narrowing to a tie at the stress level (+0.38%, $p = 0.49$). The mechanical reading is that a shortest-processing-time rule has more to exploit as within-operation spread grows, so the fixed-rule comparison narrows or reverses where processing-time variation favors it. The paired comparison against the shared-reward arm does not show this reversal because that arm is trained on the same objective and has no fixed rule for exploiting the realized spread. Retraining for this processing-time variation is untested here. Every result in this section is zero-shot.

## XXII. Bound-Component Summary

Table S-XXXII condenses what each component of $B = \max(B_{chain}, B_{mch}, B_{veh})$ captures, what it relaxes, and when it binds (Section IV of the main text).

## XXIII. External Benchmark: Port Validation and Timing Semantics

The released repository of Link et al. [8] does not preserve the seed mapping of its test set, so we validate the port by comparing distributions. We regenerate 100 instances from their generator at its released head and schedule them with their own solver and heuristic classes: all 24 deterministic released combinations (four heuristic pairs at each of six fleet sizes) are compared against the released per-instance makespans by two-sample Kolmogorov–Smirnov tests, and none rejects at $\alpha = 0.01$. Accordingly, all external comparisons in the main text are distribution-level (unpaired, $n = 100$ per side) against their released model rows.

The two environments time the same decisions differently. Their simulator dispatches anticipatorily: a vehicle may travel empty to the pickup before the job is ready and wait there, departing loaded at the later of its arrival and the job's readiness. Our environment releases a task only when its job is ready, so the empty and loaded legs run back to back. The anticipatory rule never lengthens a schedule. Our policy's schedules are computed under the stricter conservative rule and then scored by replaying the decision sequences in their simulator, so a win carries over. Their executor times our decision sequences and does not hold machines, so our destination-machine reservation does not apply there. Timing semantics explain part of any loss. Re-timing identical schedules under their rule changes the makespan by −1.0% at $|V| = 3$, −11.0% at $|V| = 6$, and up to −15.3% at $|V| \geq 9$ (3 seeds, $n = 100$ per fleet size). The difference is negligible where the fleet is scarce and largest where it is not.

The scarce-fleet result also holds for a single credit-only run (seed 301, no price channel), which scores 4213±286 at $|V| = 3$, within 0.6% of the bound-guided arm, and outperforms the same three released anchors (-15.3%, -16.9%, and -8.9% against their joint model, best modular pair, and best heuristic).

## XXIV. Deployment Cost by Instance Size

Table S-XXXIII gives the cost of running the deployed size-mixture policy at the training size and at the production batch sizes. On one processor core, the time per event grows from 2.3 ms at 10 modules to 5.9 ms at 80 modules. The forward pass and action selection take 1.5 ms and 4.2 ms of this time, and the environment step takes the rest. A GPU keeps the forward pass at 2.5–2.6 ms at every size, so it shortens the time per event only at 50 and 80 modules. Memory changes little with size, because the PyTorch runtime takes most of the 506–519 MiB peak resident memory on the processor. Training the size-mixture policy took 7.3–11.6 GPU-hours per seed as wall-clock time on a GPU shared with other jobs, so these hours are an upper bound on the exclusive-device cost, against 3.4 GPU-hours per arm at the training size (Table S-XIII).

## XXV. Pre-Specified Decision Record

Each decision rule below was written down in the project's version-controlled log before the result it governs was examined, and dates are given where the log records them. We list the rules so that readers can see what was fixed in advance, and we include the gate that failed.

- *Premise gate (2026-07-07).* The transport-as-penalty representative was fixed before any gate data existed: the same architecture trained with full travel-time knowledge but an ample fleet and a fixed nearest-vehicle rule, so

TABLE S-XXXIII
DEPLOYMENT COST OF THE SIZE-MIXTURE POLICY BY INSTANCE SIZE.

| | | CPU per event (ms) | | GPU per event (ms) | | End to end (s) | | Peak RSS |
|---|---|---|---|---|---|---|---|---|
| Modules | Events | Mean | p95 | Mean | p95 | CPU | GPU | (MiB) |
| 10 | 334 | 2.3 | 2.4 | 3.5 | 3.6 | 0.8 | 1.2 | 506 |
| 20 | 668 | 2.9 | 2.9 | 3.7 | 3.9 | 1.9 | 2.5 | 508 |
| 50 | 1,677 | 4.4 | 4.5 | 4.1 | 4.2 | 7.4 | 6.8 | 513 |
| 80 | 2,675 | 5.9 | 6.3 | 4.3 | 4.5 | 15.9 | 11.7 | 519 |

Cells: 10 and 20 modules with $|V| = 1$ and $\bar{\tau}/\bar{p} = 0.6$, 50 modules with $|V| = 2$ and $\bar{\tau}/\bar{p} = 1.0$, and 80 modules with $|V| = 3$ and $\bar{\tau}/\bar{p} = 1.0$ (machines: 25). Events: decisions per schedule, mean over the 3 instances of the row. Each entry pools these 3 instances and seeds 301–303 at batch size 1. Every run uses one torch thread on one pinned performance core of the evaluation processor of Table S-XIII, and that core was kept free of other jobs. The GPU is the training device of the same table. Per-event time runs from receiving an event to holding the next state, so it covers the policy forward pass, greedy action selection, and the environment step. The forward pass and action selection alone take 1.5–4.2 ms on the processor core and 2.5–2.6 ms on the GPU. End-to-end time covers reading the instance, building the environment and the bound, and the full rollout. Peak RSS is the peak resident memory of a CPU-pass process that ran one row. The PyTorch runtime holds 489 MiB of it before the model loads. A GPU-pass process peaks at 1,224–1,243 MiB of host memory and allocates at most 13.2 MiB on the device. The checkpoint stores 40,205 parameters in 190,413 bytes (186 KiB), and 34,892 of them run at deployment because the critic is not used. The training cost of a 10-module arm is in Table S-XIII.

that it knows travel times but never experiences vehicle contention. The rule required the explicit-coupling policy to outperform it on *both* scarce-fleet cells under paired one-sided Wilcoxon tests, Holm-corrected. The rationale recorded at the same time was that this is the strongest fair representative, since beating a weaker one would not support the premise.

- *Premise gate, first attempt (2026-07-08).* Recorded as a failure at $\bar{\tau}/\bar{p} = 0.6$ and kept. The retry regime, $\bar{\tau}/\bar{p} = 1.0$, was named in advance as the one where the premise should hold if anywhere, fixed after the first result and before any retry data; both retry arms train on $\bar{\tau}/\bar{p} \in \{0.6, 1.0\}$.
- *Credit gate (2026-07-30).* The analytic credit had to outperform shared-reward training on four scarce-fleet cells pooled over three seeds. It did not, and the pre-specified response was to report the result, reframe, and avoid searching for a cell where it passes.
- *Action-price adoption.* The price channel was adopted only if it outperformed the credit-only arm on all four gate cells, at seed 301 and again on three seeds.
- *Equivalence margin.* All equivalence tests use two one-sided tests at a ±2% margin fixed before any parity comparison was run.
- *Solver cell fill rule.* A solver cell of a production batch is reported only if all 20 instances return full-protocol incumbents; partial cells are left empty.
- *Architecture control (2026-08-17).* The decision rule for the merged-head control was fixed before its third seed finished training, including the decision to remove the parity claim without narrowing it to the cells that pass. The merged head led beyond the margin in the two one-vehicle production cells (Table S-XVIII), so the main text makes no parity claim.


\end{document}